\documentclass[twoside,11pt]{article}

\usepackage{jmlr2e}
\usepackage{amsmath} 
\usepackage{algorithm, algpseudocode}
\usepackage{bm, enumerate, rotating}
\usepackage{verbatim} 
\usepackage{subfigure}
\usepackage{epstopdf}
\usepackage{multirow}

\algnewcommand\algorithmicinput{\textbf{Input:}}
\algnewcommand\Input{\item[\algorithmicinput]}
\algnewcommand\algorithmicoutput{\textbf{Output:}}
\algnewcommand\Output{\item[\algorithmicoutput]}
\newcommand{\ie}{\textit{i.e.}}

\newcommand{\fou}{\ensuremath{\vec{u}}}
\DeclareMathOperator*{\argmin}{argmin}
\newcommand{\adjoint}{\ensuremath{{\intercal}}}
\newcommand{\ma}[1]{\ensuremath{\mathsf{#1}}}
\renewcommand{\vec}[1]{\ensuremath{\bm{#1}}}
\newcommand{\norm}[1]{\ensuremath{\left\| #1\right\|}}

\usepackage{lastpage}
\jmlrheading{20}{2019}{1-\pageref{LastPage}}{3/18; Revised
	10/19}{11/19}{18-167}{Nicolas Tremblay, Simon Barthelm\'e, Pierre-Olivier Amblard}
\ShortHeadings{Determinantal Point Processes for Coresets}{Tremblay, Barthelm\'e and Amblard}
\firstpageno{1}

\begin{document}
	
	\title{Determinantal Point Processes for Coresets}
	
	\author{\name Author One \email author-one-email\\
		\name Author Two \email author-two-email\\
		\addr Authors' address line one\\
		Authors' address line two\\
		Authors' address line three...}

	\author{\name Nicolas Tremblay \email nicolas.tremblay@gipsa-lab.fr \\
		\name Simon Barthelm\'e \email simon.barthelme@gipsa-lab.fr \\
		\name Pierre-Olivier Amblard \email pierre-olivier.amblard@gipsa-lab.fr \\
		\addr CNRS, Univ. Grenoble Alpes, Grenoble INP, GIPSA-lab, Grenoble, France}
	
	\editor{Michael Mahoney}
	
	\maketitle
	
	\begin{abstract}%   <- trailing '%' for backward compatibility of .sty file
		When faced with a data set too large to be processed all at once, an obvious solution is to retain only part of it. In practice this takes a wide variety of different forms, and among them ``coresets'' are especially appealing. A coreset is a (small) weighted sample of the original data that comes with the following guarantee: a cost function can be evaluated on the smaller set instead of the larger one, with low relative error.
		For some classes of problems, and via a careful choice of sampling
		distribution (based on the so-called ``sensitivity'' metric), iid random sampling has turned to be one of the most successful methods for building coresets efficiently. However, independent samples are sometimes overly redundant, and one could hope that enforcing diversity would lead to better performance.
		The difficulty lies in proving coreset properties in non-iid samples. We show
		that the coreset property holds for samples formed with determinantal point
		processes (DPP). DPPs are interesting because they are a rare example of
		repulsive point processes with tractable theoretical properties, enabling us
		to prove general coreset theorems.  We apply our results to both the $k$-means
		and the linear regression problems, and give extensive empirical evidence that
		the small additional computational cost of DPP sampling comes with superior
		performance over its iid counterpart. Of independent interest, we also
			provide analytical formulas for the  sensitivity in the linear regression and $1$-means cases.
	\end{abstract}
	

\section{Introduction}
\label{sec:intro}
Given a learning task, if an algorithm is too slow on large data sets, one can either speed up the algorithm or reduce the amount of data. The theory of ``coresets'' gives theoretical guarantees on the latter option. A coreset is a weighted sub-sample of the original data, with the guarantee that for any learning parameter, the task's cost function estimated on the coreset is equal to the cost computed on the entire data set up to a controlled relative error.   

An elegant consequence of such a property is that one may run learning algorithms solely on the coreset, allowing for a significant decrease in the computational cost while guaranteeing almost-equal performance. There are many algorithms that produce coresets, with some tailored for a specific task (such as $k$-means, $k$-medians, logistic regression, etc.), and others more generic. Also, there exists coreset sampling strategies both for the streaming setting and the offline setting: we choose here to focus on the offline setting. 
%\subsection{Related work}
%\label{subsec:related_work}
%Various coreset construction techniques have been proposed in the past. 
We follow the review of~\citet{munteanu_coresets-methods_2017} and classify coreset construction techniques in four categories:
	\begin{enumerate}
		\item Geometric decompositions~\citep[\emph{e.g.,}][]{har-peled_coresets_2004, har-peled_smaller_2005, agarwal_geometric_2005, har-peled_geometric_2011}. These methods propose to first discretize the ambient space of the data into a set of cells, snap each data point to its nearest cell in the discretization, and then use these weighted cells to approximate the target tasks. In all these constructions, the minimum number of samples required to guarantee the coreset property depends exponentially in the dimensionality of the ambient space, making them less useful in high-dimensional problems.
		\item Gradient descent~\citep[\emph{e.g.,}][]{badoiu_optimal_2008,
			de_la_vega_approximation_2003, kumar_linear-time_2010,
			clarkson_coresets_2010}. These methods have been originally designed for
		the smallest enclosing ball problem (\ie, finding the ball of minimum radius
		enclosing all datapoints), and have been later generalized to other
		problems. One of the main drawback of these algorithms in the $k$-means
		setting for instance is that their running time grows exponentially in the number of classes $k$~\citep{kumar_linear-time_2010}. Also, these algorithms provide only so-called \emph{weak} coresets. 
		\item Random sampling~\citep[\emph{e.g.,}][]{chen_coresets_2009, langberg_universal_2010, feldman_unified_2011, braverman_new_2016, bachem_practical_2017}. The state of the art for many different tasks such as $k$-means or $k$-median is currently via iid random non-uniform sampling. For optimal performance, the probability to sample an element should be set proportional to a quantity known as its \textit{sensitivity} (introduced by~\citet{langberg_universal_2010}). See Definition~\ref{def:sensi} for the formal definition of sensitivity. In practice, it is unpractical to compute sensitivities: state of the art algorithms rely on bi-criteria approximations to find upper bounds, and set the probability distribution proportional to this upper bound. More details on these results are provided in Section~\ref{subsec:stateoftheart}. 
		\item Sketching and projections~\citep[\emph{e.g.,}][]{phillips_coresets_2016, woodruff_sketching_2014, mahoney_randomized_2011, boutsidis_randomized_2015, boutsidis_improved_2013, cohen_dimensionality_2015, keriven_compressive_2017, clarkson_low-rank_2017, becchetti_oblivious_2019}. Another direction of research regarding data reduction that provably keeps the relevant information for a given learning task is via sketches~\citep{woodruff_sketching_2014}: compressed mappings (obtained via projections) of the original data set that are in general easy to update with new or modified data. Sketches are not strictly speaking coresets, and the difference resides in the fact that coresets are subsets of the data, whereas sketches are projections of the data. Note finally that the frontier between the two is permeable and some data summaries may combine both. 
	\end{enumerate}
	Our work falls into the random sampling category, in which the state of the art consists in tailoring a sampling distribution for the data set at hand, and then sampling iid from that distribution~\citep{chen_coresets_2009, langberg_universal_2010, feldman_unified_2011, braverman_new_2016}. Independent processes being ignorant of the past, and thus liable to sample similar points repeatedly, an avenue for improvement is to produce samples that are less redundant than what iid sampling produces. A natural idea is to consider negatively correlated point
	processes, \ie, point processes for which sampling jointly two similar
	datapoints is less probable than sampling two very different datapoints. Methods based on negatively correlated sampling have been studied in the past for specific tasks.  

For instance, for the column subset selection problem (CSSP), a method called volume sampling has been investigated in the literature~\citep[see][]{deshpande_matrix_2006, deshpande_efficient_2010}. A determinant-based sampling strategy has also been studied by~\citet{belabbas_spectral_2009}. Also, a recent work~\citep{belhadji_determinantal_2018} discusses with some details the different existing sampling-based methods (iid or with negative correlations) for the CSSP, and compares them versus a determinantal sampling strategy. 
	Another specific task for which volume sampling strategies have been used is linear regression~\citep[see for instance][]{derezinski_leveraged_2018, derezinski_reverse_2018}. 

	We propose: i/~to concentrate on a specific type of negatively-correlated sampling:
	determinantal point processes (DPPs), known to provide samples that are representative
	of the ``diversity'' in the data set~\citep{kulesza_determinantal_2012}; ii/~to study their coreset performance on generic tasks. To the
	best of our knowledge, we provide the first general coreset guarantee using non-iid
	random sampling. 
	
	DPPs are parametrized by a positive semi-definite matrix called
	$L$-ensemble and denoted by $\ma{L}$. %, whose diagonal elements encode for the
	This matrix encodes for the inclusion probabilities of each sample as well as higher order inclusion probabilities defining the correlation between samples. Note that DPPs have in general a random number of samples which in many practical situations is not adapted. This lead~\citet{kulesza_determinantal_2012} to define $m$-DPPs: DPPs conditioned to output $m$ samples (for precise definitions related to DPPs and $m$-DPPs, see Section~\ref{subsec:DPPdef}). It so happens that DPPs are more tractable than $m$-DPPs, making some proofs easier to show in the DPP context; however, $m$-DPPs are more useful in practice, especially when one needs to compare with fixed-size sampling methods as we do in this work. The reader should thus be mentally prepared to juggle from one concept to the other throughout the remainder of this paper.

	\subsection{Contributions}
	\label{subsec:contribution}
	
	%We show here that \MODIF{samples from a DPP verify the coreset property}. 
	\noindent\textbf{Theoretical contributions.} Our theorems are quite generic, and assume mostly that the cost functions under study are Lipschitz. We have two main lines of argument: the first is that DPP samples do indeed verify the coreset property, the second is that DPPs produce \emph{better} coresets than their iid counterparts if one uses the right $L$-ensemble to define the DPP. %, and the third is an asymptotic rebalancing property of DPPs. 
	%Inclusion probabilities of each sample, and non-diagonal elements encode the correlation between samples.
	More specifically, we show:
	
	\begin{itemize}
		\item Theorem~\ref{thm:main} and~\ref{thm:main_DPP}. Whatever the higher-order inclusion probabilities, if the inclusion
		probability of each sample of a DPP (or $m$-DPP) is set proportional to the sensitivity, then the
		results are at least as good as in the iid case. Technical limitations in
		controlling the concentration properties of correlated samples currently
		keep us from deriving exactly the minimum coreset size one may hope for when using DPPs. 
		\item Theorem~\ref{thm:variance_vs_poisson}. A DPP sample necessarily has a lower variance than its (independent) Poisson counterpart with same inclusion probabilities.
		\item Theorem~\ref{thm:variance_iid} and its Corollary~\ref{coro:variance_iid}. In the fixed-size context: a sample from an $m$-DPP with a rank $m$ projective $L$-ensemble (also called projective DPP) 
		necessarily leads to a lower variance than its iid counterpart with same inclusion probabilities.
	\end{itemize}
	
	We also show Theorem~\ref{thm:balanced_thm}, stating that samples from a particular polynomial $L$-ensemble based on the
	Vandermonde matrix of the data asymptotically have a rebalancing property, made precise in Section~\ref{sec:rebalancing_theorem}. 
	For instance in the $k$-means setting, this rebalancing property means that, asymptotically, such a DPP produces samples in each cluster, even if some are much smaller than others (see Figure~\ref{fig:rebalancing-clusters} for an illustration). 
	
	Finally, of independent interest, we provide for the first time analytical formulas for the sensitivity, in two specific settings: the $1$-means and the linear regression cases (Lemmas~\ref{lemma:sensi_1means} and~\ref{lemma:sensi_lr}). \\
	
	\noindent\textbf{Empirical contributions.} In the iid setting, for optimal performance, the probability of sampling an element should be set proportional to its sensitivity. In general, the sensitivity is not computable in polynomial time, thus out of reach in practice. For the specific 1-means and linear regression tasks, now that we have provided analytical formulas, these quantities become computable in polynomial time but turn out to be  heavier to compute than solving the task on the whole data set --thus useless in practice. The usual workaround in the iid setting is to set the sampling probability proportional to an upper bound (efficiently computed via, \textit{e.g.},  bi-criteria approximations) of the sensitivity. Thankfully, one still controls the performance of the obtained coreset (as a function of the upper bound's tightness).
	
	Sensitivity playing a central role in the DPP-based coreset theorems we provide, these theorems also suffer from the same impracticality. Unfortunately, due to the dependencies introduced by DPPs, mere upper bounds of the sensitivity are not sufficient to propose a controlled workaround. The theorems enable to discuss in some detail what is the ideal task-specific choice of $L$-ensemble, but in practice we for now need to resort to heuristics. 
	
	We apply our results to both the $k$-means and the linear regression problems where the initial data consists in $n$ points in $\mathbb{R}^d$. As explained, the ideal choice of $L$-ensemble $\ma{L}$ for DPP sampling is untractable in practice, we thus provide two efficient heuristics: one based on random Fourier features of the Gaussian kernel, the other on polynomial features. We pay particular attention to the computation cost of these two heuristics, and provide implementation details. These heuristics output a coreset sample in respectively $\mathcal{O}(nm^2 + nmd)$ and $\mathcal{O}(nm^2)$ time where $m$ is the number of samples of the coreset. In the $k$-means context, this is to compare to $\mathcal{O}(nkd)$ the cost of the current state of the art iid sampling algorithm via bi-criteria approximation. $m$ being necessarily larger than $d$ and $k$ to obtain the coreset guarantee in this context, our proposition is computationally heavier, especially as $m$ increases. We provide nonetheless extensive empirical evidence showing that this additional cost stays reasonable, given the enhanced performance it provides. In particular, given that we provide analytical formulas for the sensitivities in the 1-means and linear regression contexts, we are able, in these two settings, to compare the DPP-based heuristics to the ideal iid coresets (\textit{i.e.}, the coresets sampled iid from the distribution \textit{exactly} proportional to the sensitivity): results clearly show the superior performance of our heuristics. 

Finally, a Julia toolbox called DPP4Coresets is available on the authors' website.\footnote{The DPP4Coresets toolbox is also available at \url{https://gricad-gitlab.univ-grenoble-alpes.fr/tremblan/dpp4coresets.jl}\,.}

\subsection{Organization of the paper}
The paper is organized as follows. Section~\ref{sec:background} recalls the
background: the types of learning problems under consideration, the formal
definition of coresets, sensitivities and DPPs. The theoretical
Section~\ref{sec:thms} presents our main theorems on the performance of DPPs for
coreset sampling: while Section~\ref{sec:mDPP} details coreset performance in the usual formulation of coreset theorems, Section~\ref{subsec:variance} shows general variance arguments in favor of DPPs, and finally Section~\ref{sec:rebalancing_theorem} provides an original asymptotic rebalancing property of DPPs. Section~\ref{sec:DPP_for_kmeans} shows how these theorems are applicable to both the $k$-means and the linear regression problems. We provide in Section~\ref{sec:implementation} a discussion on the choice of $L$-ensemble adapted to these  problems, and detail our sampling algorithms. Finally, the empirical Section~\ref{sec:experiments} presents experiments on artificial as well as real-world data sets comparing the performance of DPP sampling to iid sampling. Section~\ref{sec:conclusion} concludes the paper. Note that for the sake of readability, many proofs and some implementation details are pushed to the Appendix.

\section{Background}
\label{sec:background}

Let $\mathcal{X}=\{\vec{x}_1,\ldots,\vec{x}_n\}$ be a set of $n$ datapoints. Let $(\Theta, d_\Theta)$ be a metric space of parameters, and $\theta$ an element of $\Theta$. 
We consider cost functions of the form:
\begin{align}
L(\mathcal{X},\theta) = \sum_{\vec{x}\in\mathcal{X}} f(\vec{x},\theta),
\end{align}
where $f$ is a non-negative $\gamma$-Lipschitz function ($\gamma>0$) with respect to $\theta$, \ie, $\forall \vec{x}\in\mathcal{X}$:
\begin{align*}
&\forall \theta\in\Theta\qquad\quad\quad\text{ } f(\vec{x}, \theta) \geq 0 ,\\
&\forall (\theta,\theta')\in\Theta^2 \qquad |f(\vec{x},\theta)-f(\vec{x},\theta')|\leq \gamma~ d_\Theta(\theta,\theta').
\end{align*}
Many classical machine learning cost functions fall under this model: $k$-means, % (as will be shown in Section~\ref{sec:DPP_for_kmeans}), 
$k$-median, logistic or linear regression, support-vector machines, low-rank approximations of matrices, etc.

\subsection{Problem considered}
\label{subsec:considered_pblem}
A standard learning task is to minimize the cost $L$ over all $\theta\in\Theta$. We write:
\begin{align}
\theta^{\text{opt}}= \argmin_{\theta\in\Theta} L(\mathcal{X},\theta) \text{, } \qquad L^\text{opt} = L(\mathcal{X}, \theta^\text{opt}) \qquad \text{ and }\qquad \langle f\rangle_\text{opt}=\frac{L^\text{opt}}{n}.
\end{align}

In some instances of this problem, e.g., if $n$ is very large and/or if $f$ is expensive to evaluate and should be computed as rarely as possible, one may rely on sampling strategies to efficiently perform this optimization task.

\subsection{Coresets}
\label{sec:coreset_def}
Let $\mathcal{S}=\{\vec{x}_{s_1},\ldots,\vec{x}_{s_m}\}$ be a subset of $\mathcal{X}$ (possibly with repetitions). To each element $\vec{x}_s\in\mathcal{S}$, associate a weight $\omega(\vec{x}_s)\in\mathbb{R}^+$. %$\mathcal{S}$ is thus a weighted subset. 
Define the estimated cost associated to the weighted subset $\mathcal{S}$ as:
\begin{align}
\label{eq:est_cost}
\hat{L}(\mathcal{S},\theta) = \sum_{\vec{x}_s\in\mathcal{S}} \omega(\vec{x}_s)f(\vec{x}_s,\theta).
\end{align}
\begin{definition}[Coreset]
	Let $\epsilon\in (0,1)$. The weighted subset $\mathcal{S}$ is a $\epsilon$-coreset for $L$ if, \emph{for any  parameter $\theta$}, the estimated cost is equal to the exact cost up to a relative error:
	\begin{align} \label{eq:coresets}
	\forall\theta\in\Theta\qquad\left|\frac{\hat{L}}{L}-1\right|\leq\epsilon. 
	\end{align}
\end{definition}
This is the so-called ``strong'' coreset definition, as the $\epsilon$-approximation is required for all $\theta\in\Theta$. A weaker version of this definition exists in the literature where the $\epsilon$-approximation is only required for $\theta^{\text{opt}}$. In the following, we focus on theorems guaranteeing the strong coreset property. 

Let us write $\hat{\theta}^{\text{opt}}$ the optimal solution computed on the weighted subset $\mathcal{S}$: $\hat{\theta}^{\text{opt}}= \argmin_{\theta\in\Theta} \hat{L}(\mathcal{S},\theta)$. 
An important consequence of the coreset property is the following:
\begin{align*}
(1-\epsilon) L(\mathcal{X}, \theta^{\text{opt}}) \leq  (1-\epsilon) L(\mathcal{X}, \hat{\theta}^{\text{opt}})  \leq \hat{L}(\mathcal{S}, \hat{\theta}^{\text{opt}})\leq \hat{L}(\mathcal{S}, \theta^{\text{opt}})\leq (1+\epsilon) L(\mathcal{X}, \theta^{\text{opt}}),
\end{align*}
\ie, running an optimization algorithm on the weighted sample $\mathcal{S}$ will result in a minimal learning cost that is a controlled $\epsilon$-approximation of the learning cost one would have obtained by running the same algorithm on the entire data set $\mathcal{X}$. Note that the guarantee is over costs only: the estimated optimal parameters $\hat{\theta}^{\text{opt}}$ and $\theta^{\text{opt}}$ may be different. Nevertheless, if the cost function is well suited to the problem: either there is one clear global minimum and the estimated parameters will almost coincide; or there are multiple solutions for which the learning cost is similar and selecting one over the other is not an issue. 

In terms of computation cost, if the sampling scheme is efficient, $n$ is very large and/or $f$ is expensive to compute for each datapoint, coresets thus enable a significant gain in computing time.

\subsection{Sensitivity}
To define appropriate sampling schemes for coresets,~\citet{langberg_universal_2010} introduce the notion of sensitivity:
\begin{definition}[Sensitivity] The sensitivity of a datapoint $\vec{x}_i \in \mathcal{X}$ with respect to a fuction $f:\mathcal{X},\Theta \rightarrow \mathbb{R}^+$ is:
	\begin{align}
	\label{eq:sensi}
	\sigma_i = \max_{\theta\in\Theta} \frac{f(\vec{x}_i,\theta)}{L(\mathcal{X},\theta)} \qquad\in[0,1].
	\end{align}
	Also, the total sensitivity is defined as :
	\begin{align*}
	\mathfrak{S} = \sum_{i=1}^n \sigma_i. 
	\end{align*}
	\label{def:sensi}
\end{definition}
Note that the fraction defining the sensitivity is not defined for $L(\mathcal{X},\theta)=0$ (that may happen for instance in the 1-means problem, in the degenerate case where all $\vec{x}_i$ are superimposed and equal to $\theta$). For simplicity, we suppose that $\forall\theta\in\Theta,~ L(\mathcal{X},\theta)>0$. 

The sensitivity is related to the concept of statistical leverage
	score~\citep[\emph{e.g.},][]{drineas_lectures_2018, drineas_fast_2012}, which plays a crucial role in iid random
	sampling theorems in the randomized numerical linear algebra
	literature~\citep{mahoney_randomized_2011}. Both notions are similar, but not equivalent. For instance, we show in Lemma~\ref{lemma:sensi_lr} that sensitivities for the linear regression task are different from the usual definition of leverage score in this context. Thus, in general, leverage scores used in the randomized linear algebra literature 
	are not sensitivities, \textit{i.e.}, they do not necessarily verify Eq.~\eqref{eq:sensi}.

In words, the sensitivity $\sigma_i$ is the worse case contribution of datapoint $x_i$ in the total cost. Informally, the larger it is, the larger its ``outlierness''~\citep{lucic_linear-time_2016}.

\subsection{iid importance sampling and state of the art results}
\label{subsec:stateoftheart}
In the iid sampling paradigm, the importance sampling estimator of $L$ is the following. Say the sample set $\mathcal{S}$ consists in $m$ samples drawn iid with replacement from a (discrete) probability distribution $\vec{p}\in\mathbb{R}^n$ (with $p_i$ the probability of sampling $\vec{x}_i$ at each draw, and $\sum_i p_i=1$). Denote by $\epsilon_i$ the random variable counting the number of occurences of $\vec{x}_i$ in $\mathcal{S}$. One may define $\hat{L}_{\text{iid}}$, the so-called  importance sampling estimator of $L$, as :
\begin{align}
\label{eq:imp_sampling_est_iid}
\hat{L}_{\text{iid}}(\mathcal{S},\theta) = \sum_i \frac{f(\vec{x}_i,\theta)\epsilon_i}{mp_i}.
\end{align}
One can show that $\mathbb{E}(\epsilon_i) = mp_i$, such that $\hat{L}_{\text{iid}}$ is an unbiased estimator of $L$:
\begin{align*}
\mathbb{E}(\hat{L}_{\text{iid}}(\mathcal{S},\theta)) = L(\mathcal{X},\theta).
\end{align*}

The concentration of $\hat{L}_{\text{iid}}$ around its expected value is controlled by the following state of the art theorem:
\begin{theorem}[Coresets with iid random sampling]
	\label{thm:iid}
	Let $\vec{p}\in[0,1]^n$ be a probability distribution over all datapoints $\mathcal{X}$ with $p_i$ the probability of sampling $\vec{x}_i$ and $\sum_i p_i=1$. Draw $m$ iid samples with replacement according to $\vec{p}$. Associate to each sample $\vec{x}_s$ a weight $\omega(\vec{x}_s)=1/mp_s$. The weighted subset obtained is a $\epsilon$-coreset with probability $1-\delta$ provided that:
	\begin{align*}
	m\geq m^*
	\end{align*}
	with
	\begin{align*}
	m^*=\mathcal{O}\left(\frac{1}{\epsilon^2}\left(\max_i\frac{\sigma_i}{p_i}\right)^2(d'+\log{(1/\delta)})\right),
	\end{align*}
	where $d'$ is the pseudo-dimension of $\Theta$ (a generalization of the Vapnik-Chervonenkis dimension). The optimal probability distribution minimizing $m^*$ is $p_i=\sigma_i/\mathfrak{S}$. In this case, the weighted subset is a $\epsilon$-coreset with probability $1-\delta$ provided that:
	\begin{align*}
	m\geq \mathcal{O}\left(\frac{\mathfrak{S}^2}{\epsilon^2}(d'+\log{(1/\delta)})\right).
	\end{align*}
	For instance, in the $k$-means setting\footnote{\label{fn:2}In the literature~\citep{feldman_unified_2011, balcan_distributed_2013}, $d'$ is often taken to be equal to $dk$ in the $k$-means setting. We nevertheless agree with~\citet{bachem_practical_2017} and their discussion in Section 2.6 regarding $k$-means' pseudo-dimension and thus write $d'=dk\log{k}$}, $d'=dk\log{k}$ and $\mathfrak{S}=\mathcal{O}(k)$ such that the coreset property is guaranteed with probability $1-\delta$ provided that:
	\begin{align*}
	m\geq \mathcal{O}\left(\frac{k^2}{\epsilon^2}(dk\log{k}+\log{(1/\delta)})\right).
	\end{align*}
\end{theorem}

This theorem is taken from the paper by~\citet{bachem_practical_2017}. Its original form goes back to~\citet{langberg_universal_2010}. Note that sensitivities cannot be computed rapidly, such that, as it is,  this theorem is unpractical. Thankfully, bi-criteria approximation schemes (such as Algorithm~2 of~\citealt{bachem_practical_2017}, or other propositions such as in~\citealt{feldman_unified_2011, makarychev_bi-criteria_2016}) may be used to efficiently find an upper bound of the sensitivity for all $i$: $b_i\geq \sigma_i$. Noting $B=\sum b_i$, and setting $p_i = b_i/B$, one shows that the coreset property may be guaranteed in the iid framework provided that $m\geq \mathcal{O}\left(\frac{B^2}{\epsilon^2}(d'+\log{(1/\delta)})\right)$. %This idea of using bi-criteria approximations to upper bound the sensitivity also goes back to~\citet{langberg_universal_2010} and has been used in many works on coresets~\citep{feldman_unified_2011,makarychev_bi-criteria_2015, braverman_new_2016, bachem_one-shot_2017}. 

Note that if one authorizes coresets with negative weights (that is, authorizes negative weights in the estimated cost of  Eq.~\eqref{eq:est_cost}), then the above theorem may be further improved~\citep{feldman_unified_2011}. Nevertheless, we prefer to restrict ourselves to positive weights as optimization algorithms such as Lloyd's $k$-means heuristics~\citep{lloyd_least_1982} are in practice more straightforward to implement on positively weighted sets rather than on sets with possibly negative weights. 

Finally, \citet[Theorem 5.5]{braverman_new_2016} improve the previous theorem by showing that under the same non-uniform iid framework, the coreset property is guaranteed provided that  $m\geq \mathcal{O}\left(\frac{\mathfrak{S}}{\epsilon^2}(d'\log{\mathfrak{S}}+\log{(1/\delta)})\right)$, thus reducing the term in $\mathfrak{S}^2$ to $\mathfrak{S}\log\mathfrak{S}$. 

\subsection{Correlated importance sampling}
Eq.~\eqref{eq:imp_sampling_est_iid} is not suited to correlated sampling and, in the following, we will use a slightly different importance sampling estimator, more adapted to this case. 
Consider a point process defined on $\mathcal{X}$ that outputs a random sample $\mathcal{S}\subset\mathcal{X}$. 
For each data point $\vec{x}_i$, denote by $\pi_i$ its inclusion (or marginal) probability:
\begin{align}
\pi_i = \mathbb{P}\left(\vec{x}_i\in\mathcal{S}\right). 
\end{align}
Moreover, denote by $\epsilon_i$ the random Boolean variable such that  
$\epsilon_i=1$ if $\vec{x}_i\in\mathcal{S}$, and $0$ otherwise. In this paper, we focus on the following definition\footnote{Note that in fact $\hat{L}_{\text{iid}}$ and $\hat{L}$ are the same objects if one defines $\epsilon_i$ to be the number of times $i$ is sampled (which will be in practice Boolean in the DPP case as the same sample can never be sampled twice in this context) and write $\hat{L}(\mathcal{S},\theta) = \sum_i \frac{f(\vec{x}_i,\theta)\epsilon_i}{\mathbb{E}(\epsilon_i)}$. We prefer to introduce both notations to avoid confusions.} of the  importance sampling cost estimator $\hat{L}$:
\begin{align}
\label{eq:imp_sampling_est_corr}
\hat{L}(\mathcal{S},\theta) = \sum_i \frac{f(\vec{x}_i,\theta)\epsilon_i}{\pi_i}.
\end{align}
By construction, $\mathbb{E}(\epsilon_i) = \pi_i$, such that $\hat{L}$ is an unbiased estimator of $L$:
\begin{align*}
\mathbb{E}(\hat{L}(\mathcal{S},\theta)) = L(\mathcal{X},\theta).
\end{align*}
Studying the coreset property in this setting boils down to studying the concentration properties of $\hat{L}$ around its expected value.

\subsection{Determinantal Point Processes}
\label{subsec:DPPdef}
In order to induce negative correlations within the samples, we choose to focus on Determinantal Point Processes (DPP), point processes that have recently gained attention due to their ability to output ``diverse'' subsets within a tractable probabilistic framework (for instance with explicit formulas for marginal probabilities). In the following, $2^{[n]}$ denotes the set of all possible subsets of the $n$ first integers. 

The central object is called the $L$-ensemble, and is nothing else than a positive semi-definite matrix $\ma{L}\in\mathbb{R}^{n\times n}$. We will write its eigenvalues $0\leq\lambda_1\leq\lambda_2\leq\ldots\leq\lambda_n$. 

%\begin{definition}[DPP~\cite{kulesza_determinantal_2012}] 
%	\label{def:DPP} Consider a point process, \ie, a process that randomly draws an element $\mathcal{S}\in[n]$. It is determinantal if there exists a positive semi-definite matrix  $\ma{K}\in\mathbb{R}^{n\times n}$ verifying $0\preceq\ma{K}\preceq 1$ such that, for every $\mathcal{A}\subseteq\mathcal{S}$, 
%	$$\mathbb{P}(\mathcal{A}\subseteq\mathcal{S}) = \det(\ma{K}_{\mathcal{A}}),$$
%	where $\ma{K}_\mathcal{A}$ is the restriction of $\ma{K}$ to the rows and columns indexed by the elements of $\mathcal{A}$. $\ma{K}$ is called the marginal kernel of the DPP.
%\end{definition}
\begin{definition}[DPP,~\citealt{kulesza_determinantal_2012}] 
	\label{def:DPP} Consider a point process, \ie, a process that randomly draws an element $\mathcal{S}\in2^{[n]}$. It is determinantal with $L$-ensemble $\ma{L}$ if
	$$\mathbb{P}(\mathcal{S}) = \frac{\det(\ma{L}_{\mathcal{S}})}{\det(\ma{I}+\ma{L})},$$
	where $\ma{L}_\mathcal{S}$ is the restriction of $\ma{L}$ to the rows and columns indexed by the elements of $\mathcal{S}$.
\end{definition}
The following well-known properties are verified (see~\citet{kulesza_determinantal_2012} for details):
\begin{itemize}
	\item one can indeed show that the normalization is proper: $\sum_{\mathcal{S}}\det(\ma{L}_\mathcal{S})=\det(\ma{I}+\ma{L})$.
	\item all inclusion probabilities, at any order, are explicit: 
	\begin{align*}
	\forall\mathcal{A}\in 2^{[n]}\qquad \mathbb{P}(\mathcal{A}\subseteq\mathcal{S}) = \det(\ma{K}_{\mathcal{A}})
	\end{align*}
	where $\ma{K}=\ma{L}(\ma{I}+\ma{L})^{-1}\in\mathbb{R}^{n\times n}$ is called the marginal kernel. %$\ma{K}$ is also a positive semi-definite matrix, but with eigenvalues constrained between $0$ and $1$. 
	In particular, the probability of inclusion of $i$, $\pi_i$, is equal to $\ma{K}_{ii}$. Also, to gain insight in the repulsive nature of DPPs, one may readily see that the joint marginal probability of sampling $i$ and $j$ reads: $\det(\ma{K}_{\{i,j\}})=\pi_i\pi_j - \ma{K}_{ij}^2$ and is necessarily smaller than $\pi_i\pi_j$, the joint probability in the case of Poisson uncorrelated sampling. The stronger the ``interaction'' between $i$ and $j$ (encoded by the absolute value of element $\ma{K}_{ij}$), the smaller the probability of sampling both jointly: this determinantal nature thus favors diverse sets of samples.
	\item $\ma{K}$ is also positive semi-definite. The eigenvalues of $\ma{K}$ are $\{\frac{\lambda_i}{1+\lambda_i}\}_i$ and are necessarily between $0$ and $1$. 
	%	\item every $L$-ensemble has an associated marginal kernel $\ma{K}=\ma{L}(\ma{I}+\ma{L})^{-1}$ that is well-defined. 
	\item it can be shown  that the number of samples of a DPP is itself random and distributed as a sum of Bernoulli parametrized by the eigenvalues of $\ma{K}$. In particular, the expected number of samples is $\mu =  \text{Tr}(\ma{K})=\sum_i\frac{\lambda_i}{1+\lambda_i}$. 
\end{itemize}

In many cases, one prefers to specify deterministically the number of samples, instead of having a random number of them. % (with a given mean). 
This leads to $m$-DPPs: DPPs conditioned to output $m$ samples. 

\begin{definition}[$m$-DPP,~\citealt{kulesza_determinantal_2012}] 
	\label{def:mDPP} Consider a point process that randomly draws an element $\mathcal{S}\in2^{[n]}$. This process is an $m$-DPP with $L$-ensemble $\ma{L}$ if:
	\begin{enumerate}
		\item[i)]  $\forall\mathcal{S} \text{ s.t. } |\mathcal{S}|\neq m, ~~ \mathbb{P}(\mathcal{S}) = 0$
		\item[ii)] $\forall\mathcal{S} \text{ s.t. } |\mathcal{S}|= m, ~~ \mathbb{P}(\mathcal{S}) = \frac{1}{Z} \det(\ma{L}_{\mathcal{S}})$ with $Z$ %=\sum_{\mathcal{S}' \text{ s.t. } |\mathcal{S}'|=m}\det(\ma{L}_{\mathcal{S'}})$ 
		the normalization constant.
	\end{enumerate}
\end{definition}
The following properties hold:
\begin{itemize}
	\item the normalization constant $Z$ is in fact the $m$-th order elementary symmetric polynomial of the eigenvalues of $\ma{L}$:
	\begin{align*}
	Z= \sum_{\mathcal{S}' \text{ s.t. } |\mathcal{S}'|=m}\det(\ma{L}_{\mathcal{S'}}) 
	~=~e_m(\lambda_1,\ldots,\lambda_n) = \sum_{1\leq  j_1 < j_2 < \cdots < j_m \leq n} \lambda_{j_1} \dotsm \lambda_{j_m}.
	\end{align*}
	\item  in general, $m$-DPPs are not DPPs: for instance the probability of including element $i$, $\pi_i$,  is no longer $\ma{K}_{ii}$ in general. In fact, one has $\pi_i = \frac{1}{Z}\sum_{\mathcal{S'} \text{ s.t } |\mathcal{S'}|=m \text{ and } i\in\mathcal{S'}} \det(\ma{L}_{\mathcal{S'}})$. 
	\item by construction, $\sum_i \pi_i =m$.
\end{itemize}

Let us define the specific but important case of projective DPPs. 
\begin{definition}[projective-DPP] 
	\label{def:mDPP} A projective DPP is a $m$-DPP whose $L$-ensemble is a projection of rank $m$:
	\begin{align*}
	\ma{L} = \ma{UU}^\top
	\end{align*}
	where $\ma{U}\in\mathbb{R}^{n\times m}$ has orthonormal columns (\ie, $\ma{U}^\top\ma{U}=\ma{I}_m$).
\end{definition}

\begin{lemma}[Lemma 1.3 of~\citealt{barthelme_asymptotic_2019}]
	\label{lemma:projDPP}
	A projective DPP with $L$-ensemble $\ma{L}$ is also a DPP, with \emph{marginal kernel} $\ma{L}$.
\end{lemma}
In fact, the set of projective DPPs is precisely the intersection between the set of DPPs and the set of $m$-DPPs. Projective DPPs are very practical objects: they have both the practical convenience of $m$-DPPs (a fixed number of samples) and the theoretical convenience of DPPs (for instance, $\pi_i$ is simply  $\ma{L}_{ii}$, \ie, the sum of squares of the $i$-th line of $\ma{U}$). 
%The number of samples of a DPP being distributed as a sum of Bernoullis parametrized by the eigenvalues of $\ma{K}$, projective DPPs output a fixed number of samples equal to the rank of $\ma{K}$. Stated otherwise: projective DPPs with kernel $\ma{K}$ are $m$-DPPs with the same kernel and $m=\text{rank}(\ma{K})$.  

%We will state the coreset theorem in the case of $m$-DPPs, as they are more useful in practice. The result thus encompasses the case of projective DPPs. The statement for DPPs is very similar and is pushed to the Appendix.

%Our goal will be to design the best possible $\ma{K}$ such that sampling a DPP with marginal kernel $\ma{K}$ guarantees the coreset property with high probability. 

\section{Coreset theorems}
\label{sec:thms}
We now detail our main theoretical contributions. In Section~\ref{sec:mDPP}, we present a coreset theorem for $m$-DPPs providing sufficient conditions on the marginal probabilities $\{\pi_i\}_i$ to guarantee the coreset property. We will see that, similar to the iid case (Theorem~\ref{thm:iid}), the optimal marginal probability should be set proportional to the sensitivity. A similar result is derived for DPPs in Appendix~\ref{app:DPPs}. These theorems are valid for any choice of higher order inclusion probabilities (the conditions are only on the first-order inclusion probabilities $\{\pi_i\}$). We further discuss in Section~\ref{subsec:variance} how one may take advantage of these additional degrees of freedom encoding the negative correlations of DPPs to improve the coreset performance over iid sampling. Finally, in Section~\ref{sec:rebalancing_theorem}, we show that a particular polynomial projective DPP asymptotically verifes a rebalancing property, thus making them natural candidates for the coreset problem.

\subsection{$m$-Determinantal Point Processes for coresets}
\label{sec:mDPP}

%We deliver the following result assuming that $n\sigma_\text{min}\geq 1$.
\begin{theorem}[$m$-DPP for coresets] 
	\label{thm:main}Let $\mathcal{S}$ be a sample from an $m$-DPP with $L$-ensemble $\ma{L}$, $\epsilon\in(0,1)$, and $\eta$ the minimal number of balls of radius $\frac{\epsilon \langle f\rangle_\text{opt}}{6 \gamma}$ necessary to cover $\Theta$, with $\gamma$ the Lipschitz parameter of $f$.   $\mathcal{S}$ is a $\epsilon$-coreset with probability larger than $1-\delta$ provided that:
	\begin{align*}
	m\geq m^*=\max(m^*_1, m^*_2)
	\end{align*}
	with:
	\begin{align*}
	m^*_1 &= \frac{32}{\epsilon^2} \left(\max_{i} \frac{\sigma_i}{\bar{\pi}_i}\right)^2 \log{\frac{4\eta}{\delta}},\\
	m^*_2 &= \frac{32}{\epsilon^2} \left(\frac{1}{n\bar{\pi}_\text{min}}\right)^2 \log{\frac{4\eta}{\delta}},
	\end{align*}
	and $\forall i, \bar{\pi}_i=\pi_i/m$. 
\end{theorem}

The proof is provided in Appendix~\ref{app:proof_thm_main}. 
Note that $m^*_1$ and $m^*_2$ are not independent of $m$: they are in fact dependent via $\bar{\pi}_i = \pi_i / m$. While this formulation may be surprising at first, this is due to the fact that in non-iid settings, separating $m$ from $\pi_i$ is not as straightforward as in the iid case (in Theorem~\ref{thm:iid}, $m$ and $p_i$ are independent) . Also, we give this particular formulation of the theorem to mimic classical concentration results obtained with iid sampling. 

In order to simplify further analysis, we suppose from now on that $n\sigma_\text{min}\geq1$. As shown in the second lemma of Appendix~\ref{app:proof_lemmas}, this is in fact verified in the $k$-means case for instance. Nevertheless, the following results may be generalized to cases with unconstrained $\sigma_{\text{min}}$ if needed, with little effects on the main results. 

\begin{corollary}
	If $n\sigma_\text{min}\geq1$, then $m_1^* \geq m_2^*$ and the coreset property of Theorem~\ref{thm:main} is verified if:
	\begin{align}
	\label{eq:coreset_condition}
	m\geq m^*=\frac{32}{\epsilon^2} \left(\max_{i} \frac{\sigma_i}{\bar{\pi}_i}\right)^2 \log{\frac{4\eta}{\delta}}
	\end{align}
	with $\forall i,~\bar{\pi}_i = \pi_i / m$.
\end{corollary}

\begin{proof}
	Denote by $j$ the index for which $\bar{\pi}_i$ is minimal and, provided that $n\sigma_\text{min}\geq1$, one has:
	\begin{align*}
	\max_i \frac{\sigma_i}{\bar{\pi}_i} n\bar{\pi}_\text{min}\geq n\sigma_j \geq n\sigma_\text{min} \geq 1,
	\end{align*}
	%	This implies that:
	%	\begin{align}
	%	\label{condition:technical}
	%	(\max_i \frac{\sigma_i}{\bar{\pi}_i} n\bar{\pi}_\text{min})^2\geq \log{\frac{4}{\delta}}~/~(\log{\frac{4}{\delta}}+\log{\eta}),
	%	\end{align}
	%	as $\eta$ is necessarily larger than 1. One can show that Eq.~\eqref{condition:technical} is equivalent to $m_1^*\geq m_2^*$, such that 
	which implies $m^* = \max(m_1^*,m_2^*)= m_1^*$.
\end{proof}

One would like to have the coreset guarantee for a minimal number of samples,
that is: to find the  marginal probabilities $\pi_i$ minimizing $m^*$. A quick
glance at Eq.~\eqref{eq:coreset_condition} tells us to set
$\pi_i=m\sigma_i/\mathfrak{S}$ in order to minimize the bound $m^*$ while
satisfying the constraint $\sum_i\pi_i=m$. In practice, however, computing the
sensitivities is often untractable. We thus propose to set the marginal
probabilities according to the following looser  condition. 

\begin{corollary}
	\label{cor:opt_DPP}
	If one sets the $\pi_i$'s such that there exists $\alpha> 0$ and $\beta \geq 1$ verifying: 
	\begin{align}
	\label{eq:condition_beta_gamma1}
	\forall i ~~~~~~\qquad &\alpha\sigma_i\leq\pi_i\leq\alpha\beta\sigma_i,\\
	\text{and}~\qquad \label{eq:condition_beta_gamma2} 
	&\frac{\alpha}{\beta}\geq\frac{32}{\epsilon^2} \mathfrak{S} \log{\frac{4\eta}{\delta}},
	\end{align}
	then $\mathcal{S}$ is a $\epsilon$-coreset with probability at least $1-\delta$. In this case, the number of samples verifies:
	\begin{align*}
	m\geq\frac{32}{\epsilon^2} \beta\mathfrak{S}^2\log{\frac{4\eta}{\delta}}.
	\end{align*}
\end{corollary}

\begin{proof}
	Let us suppose that the marginal probabilities $\pi_i$ are set such that there exists $\alpha> 0$ and $\beta \geq 1$ verifying: 
	\begin{align*}
	\forall i,\qquad \alpha\sigma_i\leq\pi_i\leq\alpha\beta\sigma_i.
	\end{align*}
	Note that:
	\begin{align*}
	\left(\max_{i} \frac{\sigma_i}{\pi_i}\right)^2 m \leq \frac{m}{\alpha^2} = \frac{1}{\alpha^2}\sum_i\pi_i\leq \frac{\beta}{\alpha}\sum_i\sigma_i = \frac{\beta}{\alpha}\mathfrak{S}. 
	\end{align*}
	Thus, the inequality 
	\begin{align*}
	\frac{\alpha}{\beta}\geq \frac{32}{\epsilon^2} \mathfrak{S} \log{\frac{4\eta}{\delta}}
	\end{align*}
	implies:
	\begin{align*}
	1\geq \frac{32}{\epsilon^2} \left(\max_{i} \frac{\sigma_i}{\pi_i} \right)^2 m \log{\frac{4\eta}{\delta}},
	\end{align*}
	that we recognize as the coreset condition~\eqref{eq:coreset_condition} by multiplying on both sides by $m$: $\mathcal{S}$ is indeed a $\epsilon$-coreset with probability larger than $1-\delta$. Moreover, in this case:
	\begin{align*}
	m = \sum_i \pi_i \geq \alpha\sum_i \sigma_i = \alpha\mathfrak{S} \geq \frac{32}{\epsilon^2} \beta\mathfrak{S}^2 \log{\frac{4\eta}{\delta}}.
	\end{align*}
\end{proof}

Corollary~\ref{cor:opt_DPP} is applicable to cases where $\sigma_\text{max}$ is not too large. In fact, in order for $\alpha\sigma_i$ to be smaller than $\pi_i$, and thus smaller than $1$ as $\pi_i$ is a probability, $\alpha$ should always be set inferior to $\frac{1}{\sigma_\text{max}}$. Now, if $\sigma_\text{max}$ is so large that $\frac{1}{\sigma_\text{max}}\leq \frac{32}{\epsilon^2}\mathfrak{S}\log{\frac{4\eta}{\delta}}$, then, even by  setting $\beta$ to its minimum value $1$, there is no admissible $\alpha$ verifying both conditions~\eqref{eq:condition_beta_gamma1} and~\eqref{eq:condition_beta_gamma2}. We refer to Appendix~\ref{remark:outliers} for a simple workaround if this issue arises. We will further see in the experimental section (Section~\ref{sec:experiments}) that elements with large sensitivities~\citep[\ie, outliers,][]{lucic_linear-time_2016} are not an issue in practice. 

Similar results are obtained for DPPs (instead of $m$-DPPs) in Appendix~\ref{app:DPPs}. 

\subsection{Links with the iid case and variance arguments}
\label{subsec:variance}
Let us first compare these results with Theorem~\ref{thm:iid} obtained in the  iid setting. A few remarks are in order:
\begin{enumerate}
	\item setting $\beta$ and $\alpha$ to $1$ in Corollary~\ref{cor:opt_DPP}, that is, setting each $\pi_i$ exactly to $\sigma_i$, the minimum number of required samples is $\frac{32\mathfrak{S}^2}{\epsilon^2} (\log{\eta} + \log{\frac{4}{\delta}})$, to compare to $\mathcal{O}(\frac{\mathfrak{S}^2}{\epsilon^2}(d'+\log{(1/\delta)}))$ of Theorem~\ref{thm:iid}, where $d'$ is the pseudo-dimension of $\Theta$. $\eta$ being the number of balls of radius $\frac{\epsilon \langle f\rangle_\text{opt}}{6 \gamma}$ necessary to cover $\Theta$, it will typically be $\frac{\epsilon \langle f\rangle_\text{opt}}{6 \gamma}$ to the power of the ambient dimension of $\Theta$ (analogous to $d'$). For instance, in the $k$-means case, $d'=dk\log{k}$ (see footnote~\ref{fn:2}), whereas, as shown later in Section~\ref{sec:DPP_for_kmeans}, $\log{\eta}=dk\log{\left(\frac{12\rho\gamma}{\epsilon\langle f\rangle_\text{opt}}+1\right)}$ where $\rho$ is the diameter of the minimum enclosing ball of the data $\mathcal{X}$. Up to the $\log$ term, $d'$ and $\log\eta$ are the same. 
		The difference observed in the $\log$ term is due to the fact that coreset theorems in the iid case~\citep[see for instance][]{bachem_practical_2017} take advantage of powerful results from the  Vapnik-Chervonenkis (VC) theory, as detailed in~\cite{li_improved_2001}. Unfortunately, these fundamental results are  valid in the iid case only, and are not easily generalized to the correlated case. Possible improvements to reduce this small gap could take  advantage of chaining arguments in correlated contexts such as in~\cite{baraud_bernstein-type_2010}, in order to improve over the repeated loose union bounds we have used in the proof.
	\item Outliers are not naturally dealt with using our proof techniques, mainly due to our multiple use of the union bound that necessarily englobes the worse-case scenario. In fact, in the importance sampling estimator used in the  iid case (Eq.~\ref{eq:imp_sampling_est_iid}), outliers are not problematic as they can be sampled several times. In our setting, outliers are constrained to be sampled only once, which in itself makes sense, but complicates the analysis. Empirically, we will see in Section~\ref{sec:experiments} that outliers are not an issue. 
	\item The DPP coreset theorems obtained are in a sense disappointing: they do not show that the concentration is tighter in the DPP case than in the iid case. They are in fact limited by the current state-of-the-art in concentration of strongly Rayleigh measures~\citep{pemantle_concentration_2014}. On the bright side, our results take \emph{only} into account first-order  inclusion probabilities: the $\{\pi_i\}$'s; meaning that these DPP sampling theorems are valid for any choice of higher-order inclusion probabilities (encoding the correlation between samples). We will now see how these extra degrees of freedom enable to provably decrease the variance of the cost estimator, compared to the iid case. 
\end{enumerate}	
\subsubsection{A first variance argument: improvement over the Poisson point process}
Consider a DPP with marginal kernel $\ma{K}$.  Build the diagonal kernel $\ma{K}_d$ with $\ma{K}_d(i,i)=\ma{K}(i,i)$. Note that a DPP from $\ma{K}_d$ reduces to a Poisson point process. 
Note also that marginal probabilities $\pi_i$ of both processes (and consequently their expected number of samples) are the same. 	
We compare the variance of the estimator $\hat{L}$ obtained with a DPP with marginal kernel $\ma{K}$ versus its variance obtained with its Poisson uncorrelated counterpart: a DPP with marginal kernel $\ma{K}_d$. %Note that a DPP with a diagonal kernel reduces to a Poisson point process and is uncorrelated up to the fact that elements cannot be sampled more than once. 

%Let us write $\text{Var}_{p}$ the variance of the estimator $\hat{L}$ in the case of a diagonal kernel $\ma{K}_d$ with $\ma{K}_d(i,i)=\pi_i$. 
\begin{theorem}
	\label{thm:variance_vs_poisson}
	For any admissible marginal kernel $\ma{K}$ (\ie, positive semi-definite with eigenvalues between $0$ and $1$), we have:
	\begin{align*}
	\forall\theta\in\Theta\qquad\text{\emph{Var}}(\hat{L}) = \text{\emph{Var}}_{d} - \sum_{i\neq j}\frac{\ma{K}_{ij}^2}{\pi_i\pi_j}f(\vec{x}_i,\theta)f(\vec{x}_j,\theta)
	\end{align*}
	where $\text{\emph{Var}}_{d}$ is the variance of the estimator based on the diagonal DPP. 
	As the function $f$ is positive, the variance of $\hat{L}$ via DPP sampling with kernel $\ma{K}$ is thus necessarily smaller than its Poisson counterpart with same inclusion probabilities.
\end{theorem}	

\begin{proof}
	We have:
	\begin{align*}
	\text{Var}(\hat{L})&=\mathbb{E}(\hat{L}^2)-\mathbb{E}(\hat{L})^2\\
	&=\sum_{i,j}\frac{\mathbb{E}(\epsilon_i \epsilon_j)}{\pi_i\pi_j}f(\vec{x}_i,\theta)f(\vec{x}_j,\theta)-L^2.
	\end{align*}
	As $\mathcal{S}$ is sampled from a DPP, the following is verified. If $i\neq j$, $\mathbb{E}(\epsilon_i \epsilon_j)=\det(\ma{K}_{\{i,j\}})=\pi_i\pi_j-\ma{K}_{ij}^2$. If $i=j$, $\mathbb{E}(\epsilon_i \epsilon_j)=\mathbb{E}(\epsilon_i)=\pi_i$. One obtains:
	\begin{align}
	\label{eq:var}
	\text{Var}(\hat{L})&=\sum_i\left(\frac{1}{\pi_i}-1\right)f(x_i,\theta)^2-\sum_{i\neq j}\frac{\ma{K}_{ij}^2}{\pi_i\pi_j}f(\vec{x}_i,\theta)f(\vec{x}_j,\theta).
	\end{align}
	The first term of the right-hand side is in fact the variance in the case of a diagonal kernel: $\sum_i\left(\frac{1}{\pi_i}-1\right)f(\vec{x}_i,\theta)^2=\text{Var}_{d}$, finishing the proof. 
\end{proof}

The important message here is that this variance reduction occurs \emph{regardless} of the choice of $\ma{K}$'s off-diagonal elements: any choice --provided that $0\preceq\ma{K}\preceq 1$ stays true-- will reduce the variance.

Proving such a variance reduction when comparing a $m$-DPP with $L$-ensemble $\ma{L}$ versus its conditional Poisson equivalent (a Poisson point process conditioned to $m$ samples, with same $\{\pi_i\}$) is much more involved, and remains open.
%\textcolor{blue}{TO MODIFY
%One can show that the variance is also improved in a similar fashion if one considers $m$-DPPs rather than DPPs, and compare them to their Poisson fixed-size equivalent (refered to as conditional Poisson sampling, or rejective sampling in the literature, see, e.g.~\cite{bertail_sharp_2016}). Consider a $m$-DPP with $L$-ensemble $\ma{L}$, its associated singleton marginal probabilities $\{\pi_i\}$ and the diagonal $L$-ensemble $L_d=\text{diag}(\pi_i)$.  	}
%The take-home message is that these improvements exist \emph{regardless} of the choice of $\ma{K}$'s off-diagonal elements (the only constraint is that $0\preceq\ma{K}\preceq 1$ stays true)!

\subsubsection{A second variance argument: improvement over the iid estimator with replacement}
We now compare the variance of the iid estimator with replacement $\hat{L}_{\text{iid}}$ of Eq.~\eqref{eq:imp_sampling_est_iid} and the variance of the DPP estimator $\hat{L}$ of Eq.~\eqref{eq:imp_sampling_est_corr}. % as the estimators have different forms\footnote{for instance, $\epsilon_i$ counts the number of times $\vec{x}_i$ is sampled in $\hat{L}_{\text{iid}}$, whereas $\epsilon_i$ is Boolean in $\hat{L}$}. %In the case of a projective DPP however, one can explicitly write the variance reduction.  
Consider a DPP with marginal kernel $\ma{K}$, with $\forall i\quad\pi_i=\ma{K}_{ii}$ the marginal probability of sampling element $i$ such that the expected number of samples $\mu=\sum_i  \pi_i$ is an integer.
We compare the variance of $\hat{L}$ with such a DPP and the variance of $\hat{L}_{\text{iid}}$ with $\mu$ independent draws with replacement with $p_i=\pi_i/\mu$ (in order to have a fair comparison). 

Before we state the result, suppose that $\ma{K}$ is of rank $r$ (with, necessarily, $\mu\leq r\leq n$). $\ma{K}$ being positive-semi definite and of rank $r$, there exists  $\ma{V}=(\vec{v}_1|\vec{v}_2|\ldots|\vec{v}_n)\in\mathbb{R}^{r\times n}$ a set of $n$ vectors in dimension $r$ such that $\ma{K}=\ma{V}^\top\ma{V}$. By construction, $\forall i\quad \norm{\vec{v}_i}^2=\ma{K}_{ii}=\pi_i$. For each vector $\vec{v}$, consider its diagram vector~\citep[Definition 2.3]{copenhaver_diagram_2014}, denoted $\tilde{\vec{v}}$, defined as:
\begin{align}
\tilde{\vec{v}}=\frac{1}{\sqrt{r-1}}\begin{bmatrix}
v(1)^2-v(2)^2\\
\vdots\\
v(r-1)^2-v(r)^2\\
\sqrt{2r}\,v(1)v(2)\\
\vdots\\
\sqrt{2r}\,v(r-1)v(r)
\end{bmatrix}\in\mathbb{R}^{r(r-1)},
\end{align}
where the difference of squares $v(i)^2-v(j)^2$ and the product $v(i)v(j)$ occur exactly once for $i<j, i = 1, 2, \cdots, r-1$.

\begin{theorem}
	\label{thm:variance_iid}
	One has:
	\begin{align*}
	\text{\emph{Var}}(\hat{L}) = \text{\emph{Var}}(\hat{L}_{\text{\emph{iid}}}) + \left(\frac{1}{\mu}-\frac{1}{r}\right)L^2-\frac{r-1}{r}\norm{\sum_i\frac{f(\vec{x}_i,\theta)}{\pi_i}\tilde{\vec{v}}_i}^2.
	\end{align*}
\end{theorem}	
\begin{proof}
	In the iid case,
	\begin{align*}
	\mathbb{E}(\hat{L}_{\text{iid}}^2) = \sum_{i=1}^n\sum_{j=1}^n \frac{f(\vec{x}_i,\theta)f(\vec{x}_j,\theta)\mathbb{E}(\epsilon_i \epsilon_j)}{\mu^2p_ip_j}
	\end{align*}
	where $\epsilon_i$ is not Boolean but counts the number of times $i$ is sampled. 
	One can show that if $i\neq j$, $\mathbb{E}(\epsilon_i \epsilon_j)=p_i p_j(\mu^2-\mu)$, and if $i=j$, $\mathbb{E}(\epsilon_i \epsilon_j)=p_i\mu + p_i^2\mu^2-\mu p_i^2$.
	Thus:
	\begin{align*}
	\text{Var}(\hat{L}_{\text{iid}})&= \mathbb{E}(\hat{L}_{\text{iid}}^2) - L^2\\
	& = \sum_{i=1}^n\sum_{j\neq i} f(x_i,\theta)f(x_j,\theta)(1-1/\mu)  + \sum_{i=1}^n f(x_i,\theta)^2 \frac{1 + p_i\mu-p_i}{\mu p_i} - L^2\\
	&= \frac{1}{\mu}\sum_{i=1}^n \frac{f(\vec{x}_i,\theta)^2}{p_i} - \frac{1}{\mu} L^2
	\end{align*}
	Moreover:
	\begin{align*}
	\text{Var}(\hat{L}) &= \sum_i \frac{f(x_i,\theta)^2}{\pi_i} - \sum_i \sum_j \frac{f(\vec{x}_i,\theta) f(\vec{x}_j,\theta)}{\pi_i \pi_j }\ma{K}_{ij}^2.
	\end{align*}
	Thus:
	\begin{align}
	\label{eq:preliminary_comparison}
	\text{Var}(\hat{L}) = \text{Var}(\hat{L}_{\text{iid}}) + \frac{1}{\mu} L^2 - \sum_i \sum_j \frac{f(\vec{x}_i,\theta) f(\vec{x}_j,\theta)}{\pi_i \pi_j }\ma{K}_{ij}^2
	\end{align}
	Proposition 2.5 of~\cite{copenhaver_diagram_2014} states: 
	\begin{align*}
	\forall (i,j)\qquad 
	\ma{K}_{ij}^2&=\left(\vec{v}_i^\top\vec{v}_j\right)^2=\frac{1}{r}\norm{\vec{v}_i}^2\norm{\vec{v}_j}^2+\frac{r-1}{r}\tilde{\vec{v}}_i^\top\tilde{\vec{v}}_j\\
	&= \frac{1}{r}\pi_i \pi_j+\frac{r-1}{r}\tilde{\vec{v}}_i^\top\tilde{\vec{v}}_j.
	\end{align*}
	Replacing this in Eq.~\eqref{eq:preliminary_comparison} yields the desired result. 
\end{proof}

\begin{remark}
	The variance of the DPP estimator is partly due to the fact that the number of samples is random, which is not the case with the iid scheme we compare it to. The following corollary compares variances when the number of samples is fixed, \ie, in the case where the DPP is projective.
\end{remark}

\begin{corollary}
	\label{coro:variance_iid}
	The marginal kernel of a projective DPP with a (fixed) number of samples $\mu$ is, by definition, of rank $r=\mu$. In this case:
	\begin{align}
	\label{eq:var_TF}
	\forall\theta\in\Theta\qquad	\text{\emph{Var}}(\hat{L}) = \text{\emph{Var}}(\hat{L}_{\text{\emph{iid}}}) -\frac{\mu-1}{\mu}\norm{\sum_i\frac{f(\vec{x}_i,\theta)}{\pi_i}\tilde{\vec{v}}_i}^2.
	\end{align}
\end{corollary}

The variance is thus necessarily improved when using a projective DPP compared to its iid counterpart. This result is remarkable: the variance reduction is independent of the sign of $f$ (supposed positive in the coreset context). This opens interesting generalizing perspectives to a more general class of cost functions $L$. 

\subsubsection{A link with tight frames}
In order to design the ideal marginal kernel $\ma{K}$, and according to the previous discussion, one wants $\ma{K}$ to verify:
\begin{itemize}
	\item The previous corollary suggests to design a projective DPP, that is: $\ma{K}=\ma{V}^\top\ma{V}$ with $\ma{VV}^\top=\ma{I}_m$.
	\item Theorem~\ref{thm:main} suggests to set $\pi_i=\ma{K}_{ii}=\frac{m\sigma_i}{\mathfrak{S}}$. 
\end{itemize}
Finding such a marginal kernel boils down to finding $\ma{V}=(\vec{v}_1|\ldots|\vec{v}_n)$ a set of $n$ vectors $\vec{v}_i$ in dimension $m$ with specified norms $\norm{\vec{v}_i}^2=\pi_i$, such that $\sum_i \pi_i=m$ and $\ma{VV}^\top=\ma{I}_m$. This is exactly the problem of finding a tight frame of $n$ vectors in dimension $m$, with specified norms~\citep{casazza_finite_2012}.
\begin{lemma}
	Such a tight frame exists.
\end{lemma}
\begin{proof}
	Let us denote by $\pi_{(i)}$ the non-decreasing ordered sequence of $\pi_i$: $\pi_{(1)}\leq \pi_{(2)} \leq \ldots\leq \pi_{(n)}$. The Schur-Horn theorem states that a hermitian matrix $K$ of size $n\times n$ with diagonal entries $\pi_i$ and eigenvalues $(0,\ldots, 0, 1, \ldots, 1)$ with $n-m$ zeros and $m$ ones, exists if $\pi_{(i)}$ majorizes $(0,\ldots, 0, 1, \ldots, 1)$, that is, if all the following inequalities are simultaneously verified:
	\begin{align*}
	&\pi_{(1)}\geq 0,~~~~\quad \pi_{(1)} + \pi_{(2)}\geq 0,~~~~\quad \cdots,~~~~\quad\sum_{i=1}^{n-m} \pi_{(i)}\geq 0\\
	&\sum_{i=1}^{n-m+1} \pi_{(i)}\geq 1,~~~~\quad\cdots,~~~~\quad\sum_{i=1}^{n-1} \pi_{(i)}\geq m-1,~~~~\quad\sum_{i=1}^{n} \pi_{(i)} \geq m.
	\end{align*}
	The first $n-m$ inequalities are trivially verified as all $\pi_i$ are supposed positive. Now, $\sum_{i=1}^{n-m+1} \pi_{(i)}\geq 1$ is also verified. Indeed, if it was not case, \ie, if $\sum_{i=1}^{n-m+1} \pi_{(i)} < 1$, then $\sum_{i=1}^{n} \pi_{(i)} < m$ as the largest $m-1$ values of $\pi_i$ are by hypothesis upper bounded by 1. This would contradict $\sum_i\pi_i=m$. A similar argument can be applied to the remaining inequalities. 
\end{proof}
Also, a tight frame not only exists, but several solutions exist in general, and efficient algorithms have been designed to build one~\citep[see for instance][]{tropp_finite-step_2004}. Out of all these possibilities, the ideal would be to find the tight frame that minimizes the variance of Eq.\eqref{eq:var_TF}. 
Up to our knowledge, this is an open and difficult question, rooted in frame theory.\\

\begin{comment}
Now, two lesser results [TODO: should we keep these?]
\begin{theorem}
For all $\theta\in\Theta$ for which
\begin{align}
\label{eq:hypo}
\sum_i \frac{f(\vec{x}_i,\theta)^2}{L^2}\geq \frac{1}{\mu}
\end{align}
is verified, 
the variance of the DPP estimator is smaller than the iid's equivalent, \emph{regardless} of the choice of $\ma{K}$'s off-diagonal elements:
\begin{align}
\text{\emph{Var}}(\hat{L}) \leq  \text{\emph{Var}}(\hat{L}_{\text{\emph{iid}}})-\sum_i \sum_{j\neq i} \frac{f(\vec{x}_i,\theta) f(\vec{x}_j,\theta)}{\pi_i \pi_j }\ma{K}_{ij}^2
\end{align}
\end{theorem}
\begin{proof}
One has:
\begin{align}
\label{eq:preliminary_comparison}
\text{Var}(\hat{L}) &= \text{Var}(\hat{L}_{\text{iid}}) + \frac{1}{\mu} L^2 - \sum_i \sum_j \frac{f(\vec{x}_i,\theta) f(\vec{x}_j,\theta)}{\pi_i \pi_j }\ma{K}_{ij}^2\\
&= \text{Var}(\hat{L}_{\text{iid}}) + \frac{1}{\mu} L^2 - \sum_i f(\vec{x}_i,\theta)^2-\sum_i \sum_{j\neq i} \frac{f(\vec{x}_i,\theta) f(\vec{x}_j,\theta)}{\pi_i \pi_j }\ma{K}_{ij}^2
\end{align}
Eq.~\eqref{eq:hypo} does imply the result.
\end{proof}
\begin{corollary}
If
\begin{align}
\label{eq:hypo}
\min_{\theta\in\Theta}\;\sum_i \frac{f(\vec{x}_i,\theta)^2}{L^2}\geq \frac{1}{\mu}
\end{align}
then
\begin{align}
\forall\theta\in\Theta\qquad\text{\emph{Var}}(\hat{L}) \leq  \text{\emph{Var}}(\hat{L}_{\text{\emph{iid}}})-\sum_i \sum_{j\neq i} \frac{f(\vec{x}_i,\theta) f(\vec{x}_j,\theta)}{\pi_i \pi_j }\ma{K}_{ij}^2
\end{align}
\end{corollary}
\end{comment}

Let us recap the above variance results. We showed that a DPP sampling scheme has necessarily a lower variance than its Poisson counterpart, \emph{regardless of the choice of off-diagonal elements of $\ma{K}$}, provided that $\ma{K}$ stays PSD with eigenvalues between $0$ and $1$. We also showed that a projective DPP sampling scheme has necessarily a lower variance than its iid counterpart \emph{regardless of the choice of off-diagonal elements of $\ma{K}$}, provided that $\ma{K}$ stays projective. We finally showed that finding the projective DPP that minimizes the variance is equivalent to a difficult problem in frame theory. 
In other words: finding the optimal DPP for a given problem and data set may be very hard, but on the other hand \emph{any} DPP is guaranteed to do at least as well as iid sampling, in the sense discussed above. Further, we can easily design DPPs which are not optimal, but still have valuable properties, as the next section shows.

\subsection{DPPs provide balanced sampling: a new type of guarantee}
\label{sec:rebalancing_theorem}

An important insight of coreset theory is that the datapoints which are different from the rest should be kept in the sample. We show in this section that one can construct a DPP which  asymptotically guarantees a rebalancing of the datapoints $\mathcal{X}$, meaning that points which are relatively isolated have a high chance of being retained.    For instance, in
the $k$-means setting, this property implies that, asymptotically, one can
construct a DPP that provably produces a balanced sample across clusters, even
in data sets where some clusters are much smaller than others. The result is
illustrated in Figs.~\ref{fig:rebalancing-clusters} and~\ref{fig:rebalancing-discs}.

\begin{figure}
	\centering
	\includegraphics[width=.7\columnwidth]{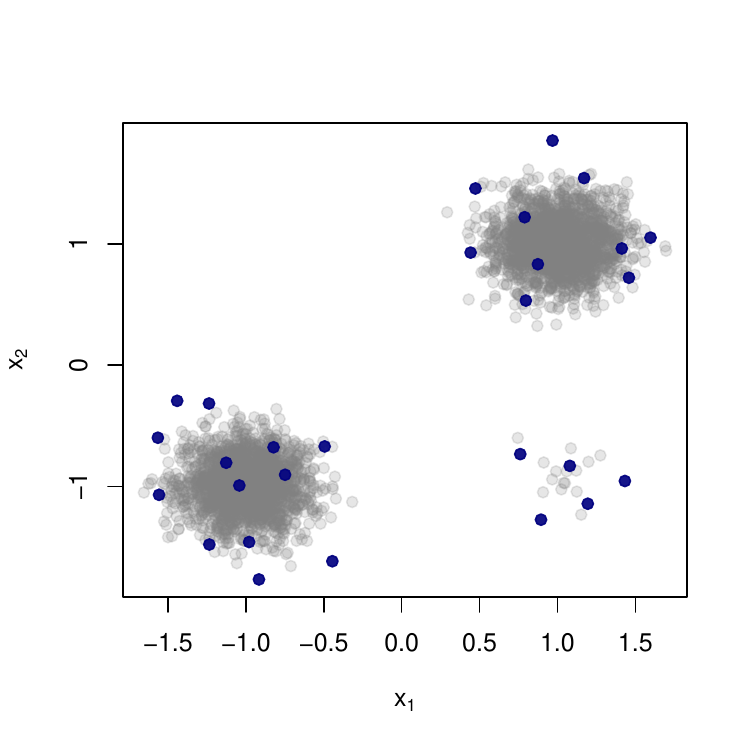}
	\caption{It is possible to construct a DPP with a (asymptotic) rebalancing	property, meaning that it will sample several points from each cluster even
	when clusters are severely imbalanced. Here, we show three imbalanced clusters: two have size 2,000 and one has size 20. In blue, a sample from a polynomial DPP (see text for definition): it samples from each cluster despite their very different sizes. The formal rebalancing property is illustrated in Figure~\ref{fig:rebalancing-discs}.}
	\label{fig:rebalancing-clusters}
\end{figure}

\begin{figure}
	\centering
	(a)\includegraphics[width=.7\columnwidth]{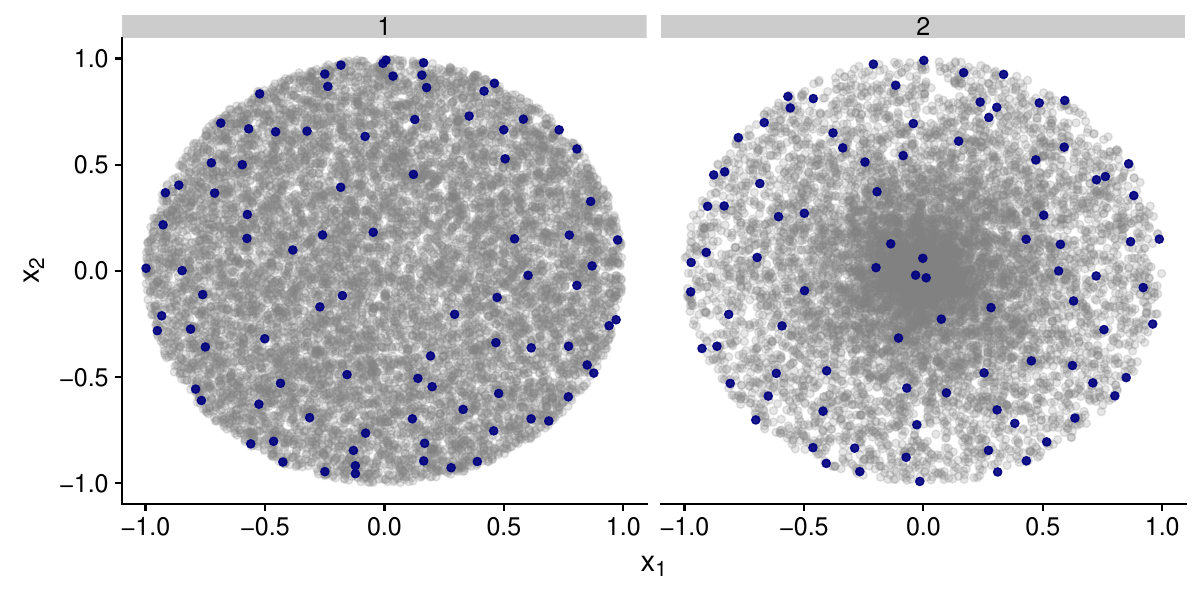}\hfill
	(b)\includegraphics[width=.7\columnwidth]{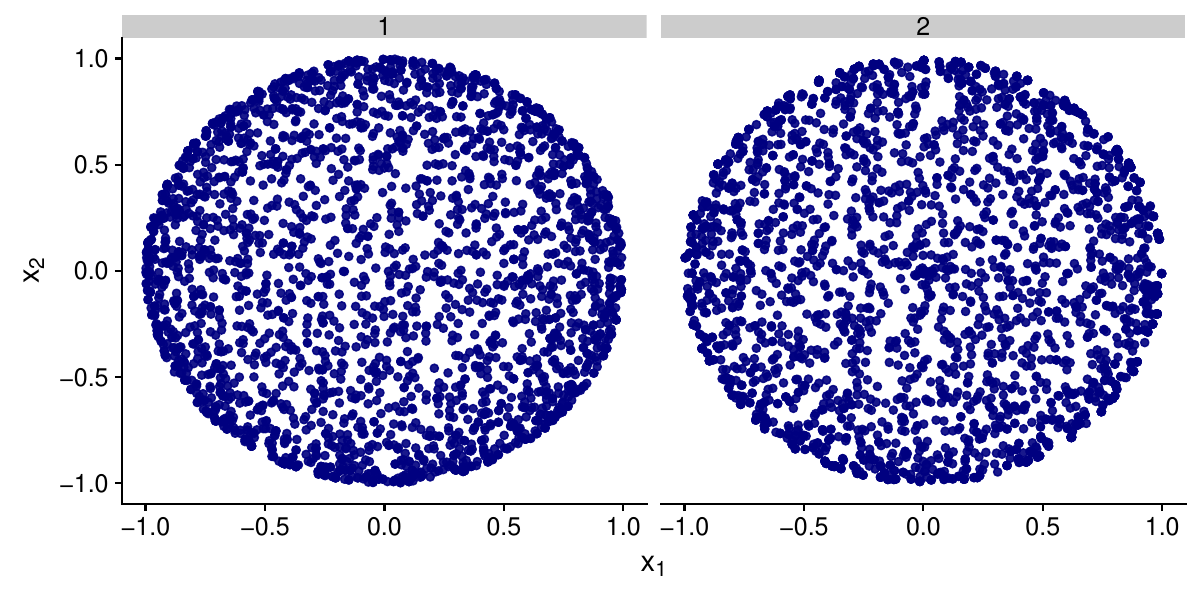}\hfill
	
	\caption{An illustration of the rebalancing result. (a)~We sample ground sets (grey points) from two different distributions on the disc. In blue, two
		realisations of a polynomial DPP constructed from the ground set. Note how
		similar the two realisations are (density-wise), despite the very different
		ground sets they are drawn from. (b)~We overlay 30 realisations of each DPP:
		the two resulting densities are very similar, again despite the different
		ground sets. Our result states that they should indeed converge in large $n$
		and $m$, and that the limiting density depends only on the shape of the
		domain. Here points close to the boundary are oversampled relative to points
		in the center, as predicted. }
	\label{fig:rebalancing-discs}
\end{figure}

In a nutshell, the result is as follows. Suppose that the data $\mathcal{X}$ is a set of $n$ elements drawn iid from a continuous distribution $\mu$ defined on $\Omega\subset\mathbb{R}^d$. Build a projective DPP $\mathcal{S}$ of size $m$ based on the monomials of the $x_i$'s (see Section~\ref{sec:polynomial-dpps} for a precise definition). Under mild regularity assumptions on $\mu$, we show that the intensity measure of $\mathcal{S}$, marginalized over $\mathcal{X}$ is \emph{independent} of $\mu$. Our proof is based on a powerful theorem from \citet{kroo_christoffel_2013}.% errThis property has no iid equivalent and is original -as far as we know- in the coreset literature.\\

Note that this rebalancing property also occurs for iid sampling with sensitivities (that provide a sort of density estimation: the lower the density of points around $x_i$, the larger $\sigma_i$, the higher the chance of sampling it). What is noteworthy here is that the rebalancing property occurs ``naturally'': without any sort of prior density-like estimation. We will emphasize this important point at the end of this Section.

In Section~\ref{sec:polynomial-dpps}, we present the specific type of polynomial DPPs for which our proof holds, that are similar to those used by~\citet{bardenet_monte_2016}. Our result is then formally stated in Section~\ref{section:result}.

\subsubsection{Projective polynomial DPPs}
\label{sec:polynomial-dpps}
The $L$-ensemble we shall build is based on the first $m$ monomials. In dimension one this is easy to
define, so we start there and generalize later to dimension $d \geq 2$. For $d=1$, we denote by  $\mathcal{X} = \{x_1, \ldots, x_n\}$ the original set (supposed to be drawn iid from $\mu$ defined on $\Omega$), and form the $n \times m$ Vandermonde matrix
\begin{equation}
\label{eq:vdm-dim1}
\ma{V}(\mathcal{X}) = [x^{j-1}_i]_{i,j=1}^{n,m} ~\in\mathbb{R}^{n\times m}.
\end{equation}
Note that this matrix has rank $m$ a.s. (as $\mu$ is supposed regular enough) and contains all monomials up to degree $m-1$.
The $L$-ensemble we consider equals:
\begin{equation}
\label{eq:poly-DPP-L-ens}
\ma{L} = \ma{V} \ma{V}^\top\in\mathbb{R}^{n\times n}.
\end{equation}
The orthogonal polynomials (defined on $\Omega$) under the empirical measure $d\mu_n = (1/n) \sum \delta_{x_i}$ associated to $\mathcal{X}$, are defined in the usual manner, i.e. $q_0(x)$ of degree $0$, $q_1(x)$ of degree $1$, $\ldots$ such that: $\int q_i(x)q_j(x)d\mu_n = \delta_{ij}$ and $\int x^{i}q_j(x) d\mu_n = 0$ if $i < j$. In other words, the sequence is constructed from Gram-Schmidt orthogonalisation under $d\mu_n$. Let us write $\vec{q}_j(\mathcal{X})=(q_j(x_1),\ldots, q_j(x_n))^\top\in\mathbb{R}^n$ the vector consisting of the  polynomial $q_j(x)$ taken at values in $\mathcal{X}$. It is well-known (and easily verified) that the QR decomposition of $\ma{V}$ verifies:
\begin{align}
\label{eq:QR-vdm}
\ma{V} =\ma{QR}
\end{align}
with $\ma{Q}=(\vec{q}_0(\mathcal{X})|\ldots|\vec{q}_{m-1}(\mathcal{X}))\in\mathbb{R}^{n\times m}$ and $\ma{R}\in\mathbb{R}^{m\times m}$ an upper triangular matrix. \\
%The orthogonal polynomials under the empirical measure $d\nu_n = (1/n) \sum \delta_{x_i}$ are defined in the usual manner, i.e. the sequence $q_0(x),q_1(x),\ldots$ such that: $\int q_i(x)q_j(x)d\nu = \delta_{ij}$ and $\int x^{i}q_j(x) d\nu = 0$ if $i < j$. In other words, the sequence is constructed from Gram-Schmidt orthogonalisation under $d\nu_n$. It is well-known (and easily verified) that in this case the coefficients of the orthogonal polynomials are given by the QR decomposition of $\ma{V}$, i.e.:
%\begin{equation}
%\label{eq:QR-vdm}
%\ma{V} = \ma{Q} \ma{R}
%\end{equation}
%where $\ma{Q}$ is a $n\times m$ matrix with the value of the $q_i$'s along the columns, and $\ma{R}$ is an upper triangular matrix of size $m \times m$.

Now, consider the $m$-DPP $\mathcal{S}$ with $L$-ensemble $\ma{L}=\ma{VV}^\top$. Using the fact that $\det(AB)=\det(A)\det(B)$ if $A$ and $B$ are square, we have:
\begin{align*}
\mathbb{P}(\mathcal{S}) &= Z^{-1} \det( \ma{L}_{\mathcal{S}} ) \\
&= Z^{-1} \det( (\ma{Q} \ma{R} \ma{R}^\top \ma{Q}^\top  )_{\mathcal{S}} ) \\
&= Z^{-1} \det((\ma{QQ}^\top)_\mathcal{S}) \det(\ma{R})^2 \\
&= Z'^{-1} \det((\ma{QQ}^\top)_\mathcal{S})
\end{align*}
such that $\mathcal{S}$ is also a $m$-DPP with $L$-ensemble $\ma{Q} \ma{Q}^\top$. As $\ma{Q}^\top\ma{Q}=\ma{I}_m$, \ie, $\ma{Q} \ma{Q}^\top$ is projective, $\mathcal{S}$ is in fact a projective DPP. As a result, its associated marginal kernel is also $\ma{QQ}^\top$ (see Lemma.~\ref{lemma:projDPP}) and, for instance:
\begin{equation*}
\mathbb{P}(x_i \in \mathcal{S}) = \sum_{j=0}^{m-1} q_j^2(x_i).
\end{equation*}
The extension to $d>1$ is mostly straightforward, but there are a few
differences to keep in mind when defining the Vandermonde matrix of monomials. Monomials $\vec{x}^{\vec{\alpha}}$ are now defined as:
\begin{equation*}
\vec{x}^{\vec{\alpha}} = \prod_{j=1}^d x(j)^{\alpha_j}
\end{equation*}
The total degree of a monomial equals the sum of the degrees in each variable,
i.e. $\sum \alpha_i = ||\vec{\alpha}||_1$. The most significant difference between the
one-dimensional case and the general case is that there is more than one
monomial of total degree $\phi$. For example, in dimension 2, $\vec{x}=(x(1),x(2))^\top$ and the monomials of
degree 2 are given by the powers $(2,0)$, $(0,2)$ and $(1,1)$: $x(1)^2$, $x(2)^2$, $x(1)x(2)$. 
A good way of thinking about the construction of a polynomial DPP in the
multidimensional case is to pick first a maximum order (e.g. $\phi=3$), meaning
that all monomials with total degree up to 3 are included. Then the natural
sample size $m$ for the DPP equals the total number of features, giving $m = {d+\phi
	\choose \phi}$. Again, for $d=2$ and $\phi=3$, this gives $m = 1+2+3+4 = 10$. In fact, there is one monomial of order $0$: $1$, two monomials of order $1$: $x(1)$ and $x(2)$, three monomials of order $2$ (the ones stated above), and four monomials of order $3$: $x(1)^3$, $x(2)^3$, $x(1)^2x(2)$ and $x(1)x(2)^2$. This implies that in dimension $d$, the $m$-DPP detailed earlier is well-defined only for specific values of $m$: $m={d+1\choose 1}=d+1$, or $m={d+2\choose 2}=\frac{1}{2}(d+1)(d+2)$, or $m={d+3\choose 3}=\frac{1}{6}(d+1)(d+2)(d+3)$, etc. 

A slight technical difficulty arises in defining the orthogonal polynomials of a
multivariate measure: in dimension 1, the fact that there is a single monomial
of a given degree leads to a natural order in which to perform the Gram-Schmidt
procedure. In higher dimensions the order is only a partial order, so that we
can introduce the monomials by blocks of equal degree, but within a block the
ordering is arbitrary. So we may pick any arbitrary order (e.g. lexicographic)
and run Gram-Schmidt in that order~\citep[for more, see][]{dunkl_orthogonal_2014}.
Given this choice the QR decomposition remains well-defined and all properties given above in the 1D case carry over to the general case. In particular, the link with the orthogonal polynomials on the discrete measure $\mu_n$ stays valid.

\subsubsection{The rebalancing theorem}
\label{section:result}
%The high-level overview of the proof is as follows. We first show that the inclusion probabilities of the discrete DPP $\mathcal{S}$ tend to that of a continuous DPP. We then use results on the asymptotics of Christoffel functions to show that ``large enough'' coresets are automatically rebalanced.

Formally, the result is as follows. The intensity function $\iota(\vec{x})$ of a
point process quantifies the expected number of points to be found around
$\vec{x}$. We characterize the asymptotics of the intensity function of a DPP $\mathcal{S}$
when both $\mathcal{S}$ and the ground set $\mathcal{X}$ are large, and show that, in that limit,
the intensity is \emph{independent} of the measure $\mu$ from which $\mathcal{X}$ is sampled from.\\
The result is stated formally as a double limit, letting first $n$ go to
infinity (an easy discrete-to-continuous limit), and then increasing the order $\phi$ 
of the polynomial DPP, which implies $m$ going to infinity too.  %(letting us use the asymptotics of Christoffel functions). 
We emphasize that, empirically
speaking, rebalancing occurs for reasonable values of $n$ and $m$ but the rate
of convergence is hard to quantify. \\
Certain regularity assumptions are inherited
from the work of~\citet{kroo_christoffel_2013}, to which we refer for more thorough details.  The formal assumptions are as follows: 
\begin{enumerate}
	\item The initial data set $\mathcal{X}=\{x_1,\ldots,x_n\}$ is drawn i.i.d. from a measure $\mu$ over a
	compact, convex\footnote{The convexity
		assumption can probably be relaxed.} domain $\Omega \subset \mathbb{R}^d$.
	\item $\mu$ and the Lebesgue measure $\nu$ are mutually absolutely continuous
	on $\Omega$, so that $\mu'$, the density, is well-defined everywhere on the
	domain (we use the Lebesgue measure for simplicity, another measure may be
	substituted)
	\item We are interested in convergence ``in the bulk'', ie. inside the domain.
	Formally, the results hold for $D
	\subset D_1 \subset \Omega$, where $D$ is compact and $D_1$ is open
	\item $\mu'$ is bounded above and below on $D_1$
	\item We form a $m$-DPP $\mathcal{S}$ on the set $\mathcal{X}$, with a polynomial kernel of degree $\phi$
	(defined in the previous Section), such that $m={\phi+d\choose \phi}$.
	\item (technical) $\mu$ is regular in the sense of Stahl, Totik, and Ullman, and
	the Christoffel function with respect to $\mu$ verifies condition (1.7) in
	\cite{kroo_christoffel_2013}. 
\end{enumerate}
The intensity measure of $\mathcal{S}$, marginalizing over $\mathcal{X}$, which we denote by $I_{n,\phi}(\mathcal{A})$ equals the expected number of points of $\mathcal{S}$ in set $\mathcal{A}$, i.e.:
\begin{equation}
\label{eq:intensity-func}
I_{n,\phi}(\mathcal{A}) =  \mathbb{E}_{\mathcal{X},\mathcal{S}} \left(|\mathcal{S} \cap \mathcal{A} | \right)= \mathbb{E}_{\mathcal{X},\mathcal{S}} \left\{ \sum_{s \in \mathcal{S}} \mathbb{I} \left(s \in \mathcal{A} \right)\right\}
\end{equation}
Note that the expectation is over \emph{both} $\mathcal{X}$ and $\mathcal{S}$. Furthermore, 
$I_{n,\phi}(\Omega)$ equals $m$, the total number of points in $\mathcal{S}$.\\
Our result may be stated as follows.
\begin{theorem}
	\label{thm:balanced_thm}
	Under the assumptions above, for all $\mathcal{A} \subset D_1$,
	\begin{equation*}
	\lim_{\phi \rightarrow \infty} \frac{1}{{\phi+d \choose \phi}} \lim_{n \rightarrow \infty} I_{n,\phi}(\mathcal{A}) = \int_{\mathcal{A}} \kappa(\vec{y}) d\vec{y}
	\end{equation*}
	where $\kappa$ is a density \emph{independent} of $\mu$
\end{theorem}
The proof is in Appendix~\ref{app:proof_asymptotic_thm}.
\begin{lemma}
	$\kappa$ is mostly dependent on the distance to the boundaries of $\Omega$. For
	example, if $\Omega$ is the unit ball in $\mathbb{R}^d$, $\kappa(\vec{y}) = \left( 1- \left| \left| \vec{y} \right| \right|^2 \right)^{-1/2}$
\end{lemma}
See~\cite{kroo_christoffel_2013} for a proof. 

Several important remarks are in order:
\begin{itemize}
	\item unlike iid sampling with sensitivities or other density-related measure for which such rebalancing property will also occur, there is here no prior density estimation: the $L$-ensemble is defined via the Vandermonde matrix that is trivial to compute. Thus, this rebalancing is a property that ``naturally'' arises from the DPP. 
	\item this is only an asymptotic result as $n$ and $m$ go to infinity. Finding minimal values of $m$ for which rebalancing is highly probable, or even rates of convergence is likely a difficult endeavour. We emphasize nevertheless that, empirically speaking,  rebalancing occurs for reasonable values of $n$ and $m$, %. However, it is partial rather than perfect, 
	as visible in Figs.~\ref{fig:rebalancing-clusters} and~\ref{fig:rebalancing-discs}. 
\end{itemize} 
This ends the theoretical results of this paper. We now move on to applications. In the next Section, we apply the results to two problems: $k$-means and linear regression. In Section~\ref{sec:implementation}, implementation details are provided. Finally, experimental validation on artificial and real-world data sets is provided in Section~\ref{sec:experiments}.

\section{Application to two problems: $k$-means and linear regression}
We focus on two problems: $k$-means and linear regression. Admittedly, these are not the best problems to exhibit the usefulness of coresets: there already exists very efficient algorithms to solve them and the need for a small controlled summary is in fact rare. We nevertheless focus on these two problems as they have been well studied in the iid setting, which it is our goal to improve on. Moreover, we derived analytical formulas for the sensitivity in the $1$-means and the linear regression settings: we will thus be able to compare, in those two cases, DPP sampling vs the ideal iid setting (later in the experimental Section~\ref{sec:experiments}).

\subsection{Application to $k$-means}
\label{sec:DPP_for_kmeans}
The theoretical results of Section~\ref{sec:thms} are valid for any learning problem of the form detailed in Section~\ref{subsec:considered_pblem}. We now  specifically consider the $k$-means problem on a set $\mathcal{X}$ comprised of $n$ datapoints in $\mathbb{R}^d$.  This problem boils down to finding $k$ centroids $\theta=(c_1, \ldots, c_k)$, all in $\mathbb{R}^d$, such that the following cost is minimized:
\begin{align}
L(\mathcal{X},\theta) = \sum_{x\in\mathcal{X}} f(x,\theta) ~~\text{ with }~~ f(x,\theta) = \min_{c\in\theta}\norm{x-c}^2.\nonumber
\end{align}
Let $\rho$ be the diameter of the minimum enclosing ball of $\mathcal{X}$ (the smallest ball enclosing all points in $\mathcal{X}$). 
Theorem~\ref{thm:main} and its corollaries are applicable to the $k$-means problem, such that:

\begin{corollary}[$m$-DPP for $k$-means]
	\label{corollary:mDPP_kmeans}
	Let $\mathcal{S}$ be a sample from an $m$-DPP with $L$-ensemble $\ma{L}$. Let $\epsilon, \delta\in(0,1)^2$. With probability at least $1-\delta$, $\mathcal{S}$ is a $\epsilon$-coreset provided that:
	\begin{align*}
	m\geq m^*=\frac{32}{\epsilon^2} \left(\max_{i} \frac{\sigma_i}{\bar{\pi}_i}\right)^2 \left(kd \log\left(\frac{24\rho^2}{\epsilon\langle f\rangle_\text{opt}}+1\right) + \log{\frac{4}{\delta}}\right),
	\end{align*}
	with $\forall i, \bar{\pi}_i=\pi_i/m$. 
	
	Setting the marginal probabilities to their optimal values $\pi_i=m\sigma_i/\mathfrak{S}$, $\mathcal{S}$ is a $\epsilon$-coreset with probability larger than $1-\delta$ provided that:
	\begin{align*}
	m\geq\frac{32}{\epsilon^2} \mathfrak{S}^2 \left(kd \log\left(\frac{24\rho^2}{\epsilon\langle f\rangle_\text{opt}}+1\right) + \log{\frac{4}{\delta}}\right).
	\end{align*}
\end{corollary}

\begin{proof}
	Let us write $\mathcal{B}$ the minimum enclosing ball of $\mathcal{X}$, of diameter $\rho$. The potentially interesting centroids are necessarily included in $\mathcal{B}$ such that the space of parameters $\Theta$ in the $k$-means setting is the set of all possible $k$ centroids in $\mathcal{B}$: $\Theta=\mathcal{B}^k$. The metric $d_\Theta$ we consider is the Hausdorff metric associated with the Euclidean distance: 
	\begin{align}
	\forall \theta, \theta',\qquad	d_\Theta(\theta,\theta') = \max&\left\{\max_{c\in\theta}\min_{c'\in\theta'}\norm{c-c'}_2, ~~
	\max_{c'\in\theta'}\min_{c\in\theta}\norm{c-c'}_2 \right\}.\nonumber
	\end{align}
	
	\noindent\textbf{An $\epsilon'$-net of $\Theta$.} 
	Consider $\Gamma_\mathcal{B}$ an $\epsilon'$-net of $\mathcal{B}$ consisting of $(\frac{2\rho}{\epsilon'}+1)^d$ small balls of radius $\epsilon'$. Such a covering indeed exists: see, \emph{e.g.}, Lemma $2.5$ in \citet{geer_empirical_2000}. Consider $\Gamma=\Gamma_\mathcal{B}^k$ of cardinality $|\Gamma| =  (\frac{2\rho}{\epsilon'}+1)^{kd}$. Let us show that $\Gamma$ is an $\epsilon'$-net of $\Theta$, that is:
	\begin{align}\nonumber
	\forall\theta\in\mathcal{B}^k,\quad \exists\theta^*\in\Gamma \quad\text{ s.t. }\quad d_\Theta(\theta,\theta^*)\leq \epsilon'.
	\end{align}
	In fact, consider $\theta = (c_1, \ldots, c_k) \in\mathcal{B}^k$. By construction, as $\Gamma_\mathcal{B}$ is an $\epsilon'$-net of $\mathcal{B}$, we have:
	\begin{align}\nonumber
	\forall i=1,\ldots,k\qquad \exists c_i^*\in\Gamma_\mathcal{B} \quad\text{ s.t. }\quad \norm{c_i-c_i^*}\leq\epsilon'.
	\end{align}
	Writing $\theta^*=(c_1^*,\ldots, c_k^*)\in \Gamma$, one has:
	\begin{align}\nonumber
	d_\Theta(\theta,\theta^*) \leq \epsilon',
	\end{align}
	which proves that $\Gamma$ is an $\epsilon'$-net of $\Theta$. The number of balls of radius $\epsilon'=\epsilon \langle f\rangle_\text{opt} / 6 \gamma$ sufficient to cover $\Theta$ is thus $\eta=(\frac{12\rho\gamma}{\epsilon\langle f\rangle_\text{opt}}+1)^{kd}$.\\
	
	\noindent\textbf{$f(x,\theta)$ is $\gamma$-Lipschitz with $\gamma=2\rho$.} 
	Consider any $\theta$, $\theta'$ and $x\in\mathcal{X}$. We want to show that:
	\begin{align}
	-\gamma~ d_\Theta(\theta,\theta') \leq f(x,&\theta)-f(x,\theta')\leq \gamma~ d_\Theta(\theta,\theta').\nonumber
	\end{align}
	Let us write $c=\argmin_{t\in\theta}\norm{x-t}^2$ the centroid in $\theta$ closest to $x$ and  $c'=\argmin_{t'\in\theta'}\norm{x-t'}^2$ the centroid in $\theta'$ closest to $x$. Moreover, let us write  $\tilde{c}'=\argmin_{t'\in\theta'}\norm{c-t'}^2$ the centroid in $\theta'$ closest to $c$. Note that $c'$ and $\tilde{c}'$ are not necessarily equal.  
	By definition of $c'$, one has:
	\begin{align}\nonumber
	\norm{x-c'}\leq \norm{x-\tilde{c}'},
	\end{align}
	such that:
	\begin{align}
	\norm{x-c'}_2 - \norm{x-c}_2 \leq \norm{x-\tilde{c}'}_2 - \norm{x-c}_2 
	\leq \norm{\tilde{c}'-c}_2 \leq d_\Theta(\theta,\theta').\nonumber
	\end{align}
	Thus:
	\begin{align}\nonumber
	f(x, \theta')-f(x,\theta)=\norm{x-c'(x)}^2 - \norm{x-c}^2 
	&= (\norm{x-c'} - \norm{x-c}) (\norm{x-c'} + \norm{x-c})\\
	&\leq (\norm{x-c'} + \norm{x-c}) ~d_\Theta(\theta,\theta') \leq 2\rho~d_\Theta(\theta,\theta') .\nonumber
	\end{align}
	
	\noindent\textbf{Finally, $n\sigma_\text{min}\geq 1$}, as shown by the second lemma of Appendix~\ref{app:proof_lemmas}.
	
	Given all these elements, Theorem~\ref{thm:main} and its subsequent corollaries are thus applicable to the $k$-means setting and one obtains the desired result.
\end{proof}	

Note that, in the case of DPPs, one could apply Theorem~\ref{thm:main_DPP} to the $k$-means problem, and obtain similar results.

\subsection{Application to linear regression}
\label{sec:DPP_for_LR}
We now consider the linear regression problem: find $\theta\in\mathbb{R}^d$ such that a measured vector $y\in\mathbb{R}^n$ is closest to $\ma{X}\theta$ where $\ma{X}^\top=(x_1|\ldots|x_n)\in\mathbb{R}^{d\times n}$ are $n$ data points in $\mathbb{R}^d$. Let us write $X_i=(y_i,x_i)$ and $\mathcal{X}=\{X_1, \ldots, X_n\}$. 
The least squares estimator minimizes:
$$L(\mathcal{X},\theta)=\norm{y-\ma{X}\theta}^2_2.$$
By denoting 
\begin{align*}
f(X_i,\theta)=(y_i-x_i^\top\theta)^2,
\end{align*}
one can thus write the least squares solution to the linear regression problem in the form of the problems investigated in this paper: the objective is to minimize the cost $L$ with $f$ a positive function:
\begin{align*}
L(\mathcal{X},\theta)=\norm{y-\ma{X}\theta}^2_2=\sum_{i=1}^n f(X_i, \theta).
\end{align*} 

We suppose that all $\vec{x}_i$ are enclosed in the unit ball in dimension $d$ and that $y_i\in[0,1]$. Moreover, we suppose that the space $\Theta$ is bounded and enclosed in a $d$-dimensional ball $\mathcal{B}$ centered in $0$ of diameter $\rho$.

Even though we derived the analytical formulation of the sensitivity for linear regression (Lemma~\ref{lemma:sensi_lr}), we were not able to show that $n\sigma_{\text{min}}\geq 1$ in general.  We thus have the following slightly more complicated result:

\begin{corollary}[$m$-DPP for linear regression]
	\label{corollary:mDPP_LR}
	Let $\mathcal{S}$ be a sample from an $m$-DPP with $L$-ensemble $\ma{L}$. Let $\epsilon, \delta\in(0,1)^2$. With probability at least $1-\delta$, $\mathcal{S}$ is a $\epsilon$-coreset provided that:
	\begin{align*}
	m\geq \max(m_1^*, m_2^*)
	\end{align*}
	with
	\begin{align*}
	m_1^*&=\frac{32}{\epsilon^2} \left(\max_{i} \frac{\sigma_i}{\bar{\pi}_i}\right)^2 \left(d \log\left(\frac{12\rho(4\rho +2)}{\epsilon\langle f\rangle_\text{opt}}+1\right) + \log{\frac{4}{\delta}}\right),\\
	m_2^*&=\frac{32}{\epsilon^2} \left(\frac{1}{n\bar{\pi}_{\text{min}}}\right)^2 \left(d \log\left(\frac{12\rho(4\rho +2)}{\epsilon\langle f\rangle_\text{opt}}+1\right) + \log{\frac{4}{\delta}}\right)
	\end{align*}
	with $\forall i, \bar{\pi}_i=\pi_i/m$. 
	
	Setting the marginal probabilities to their optimal values $\pi_i=m\sigma_i/\mathfrak{S}$, $\mathcal{S}$ is a $\epsilon$-coreset with probability larger than $1-\delta$ provided that:
	\begin{align*}
	m\geq\frac{32}{\epsilon^2} \max\left(\mathfrak{S}^2, \frac{\mathfrak{S^2}}{n^2\sigma_{\text{min}}^2}\right) \left(d \log\left(\frac{12\rho(4\rho +2)}{\epsilon\langle f\rangle_\text{opt}}+1\right) + \log{\frac{4}{\delta}}\right).
	\end{align*}
\end{corollary}

\begin{proof}
	The metric $d_\Theta$ we consider is the Euclidean distance in dimension $d$.
	
	\noindent\textbf{- An $\epsilon'$-net of $\Theta$.} 
	Consider $\Gamma_\mathcal{B}$ an $\epsilon'$-net of $\mathcal{B}$ consisting of $(\frac{2\rho}{\epsilon'}+1)^d$ small balls of radius $\epsilon'$. Such a covering indeed exists: see, \emph{e.g.}, Lemma $2.5$ in~\citet{geer_empirical_2000}. The number of balls of radius $\epsilon'=\epsilon \langle f\rangle_\text{opt} / 6 \gamma$ sufficient to cover $\Theta$ is thus $\eta=(\frac{12\rho\gamma}{\epsilon\langle f\rangle_\text{opt}}+1)^{d}$. \\
	
	\noindent\textbf{- $f(X,\theta)$ is $\gamma$-Lipschitz with $\gamma=4\rho +2$.} 
	Consider any $\theta$, $\theta'$, $\vec{x}_i$ and $y_i$. We want to show that:
	\begin{align*}
	\left((y_i-\vec{x}_i^\top\theta)^2-(y_i-\vec{x}_i^\top\theta')^2\right)^2\leq \gamma^2~ \norm{\theta-\theta'}^2.
	\end{align*}
	In fact:
	\begin{align*}
	\left((y_i-\vec{x}_i^\top\theta)^2-(y_i-\vec{x}_i^\top\theta')^2\right)^2&=\left(\theta^\top \vec{x}_i\vec{x}_i^\top\theta -\theta'^\top\vec{x}_i\vec{x}_i^\top\theta' -2y_i\vec{x}_i^\top(\theta-\theta')\right)^2\\
	&=\left[\left(2\theta'^\top\vec{x}_i\vec{x}_i^\top+(\theta-\theta')^\top\vec{x}_i\vec{x}_i^\top-2y_i\vec{x}_i^\top\right)\left(\theta-\theta'\right)\right]^2\\
	&\leq\left[\norm{2\theta'^\top\vec{x}_i\vec{x}_i^\top+(\theta-\theta')^\top\vec{x}_i\vec{x}_i^\top-2y_i\vec{x}_i^\top}\norm{\theta-\theta'}\right]^2\\
	&\leq\left[2\norm{\vec{x}_i\vec{x}_i^\top}\norm{\theta'}+\norm{\vec{x}_i\vec{x}_i^\top}\norm{\theta-\theta'}+2y_i\norm{\vec{x}_i}\right]^2\norm{\theta-\theta'}^2
	\end{align*}
	by triangular inequality and writing $\norm{\vec{x}_i\vec{x}_i^\top}$ the $2$-norm of the matrix $\vec{x}_i\vec{x}_i^\top$, which is equal to $\norm{\vec{x}_i}^2$ and bounded by one by hypothesis. As $\Theta$ is supposed to be enclosed in a ball of radius $\rho$, we further have:
	\begin{align*}
	\left((y_i-\vec{x}_i^\top\theta)^2-(y_i-\vec{x}_i^\top\theta')^2\right)^2&\leq (4\rho+2)^2\norm{\theta-\theta'}^2
	\end{align*}
	
	Given these elements, Theorem~\ref{thm:main} is applicable to the linear regression setting and one obtains the desired result.
\end{proof}	

\section{Implementation}
\label{sec:implementation}
\subsection{The DPP's ideal marginal kernel}

Following the theoretical results, the ideal strategy (although unrealistic) to build the marginal kernel $\ma{K}$ of the ideal DPP sampling scheme would be as follows. 1/~Deal with outliers as explained in Appendix~\ref{remark:outliers} until $\sigma_\text{max}$ is not too large. 2/~Compute all $\sigma_i$. 3/~Set all $\pi_i$ to $m\sigma_i/\mathfrak{S}$ with $m$ sufficiently large as detailed in the theorems. 4/~Find all non-diagonal elements of $\ma{K}$ in order to minimize for all $\theta$ the estimator's variance, as derived in Eq.~\eqref{eq:var}:
\begin{align*}
\text{Var}(\hat{L})&=\sum_i\left(\frac{1}{\pi_i}-1\right)f(x_i,\theta)^2-\sum_{i\neq j}\frac{\ma{K}_{ij}^2}{\pi_i\pi_j}f(\vec{x}_i,\theta)f(\vec{x}_j,\theta).
\end{align*}
while constraining $\ma{K}$ to be a valid marginal kernel, \ie: SDP with $0\preceq\ma{K}\preceq 1$, 5/~sample a DPP with kernel $\ma{K}$. On our way to derive a practical algorithm with a linear complexity in $n$, many obstacles stand before us: there is no known polynomial algorithm to compute all $\sigma_i$ in the general setting, solving exactly the minimization problem of step 4 under eigenvalue constraint remains open, and sampling from this engineered ideal $\ma{K}$ costs  $\mathcal{O}(n^3)$ number of operations (see Algorithm~1 of~\citet{kulesza_determinantal_2012}: it necessitates a full diagonalization of $\ma{K}$). Designing a linear-time algorithm that provably verifies under a controlled error the conditions of our previous theorems is out-of-scope of this paper. In the following, we prefer to first recall the intuitions behind the construction of a good kernel, and then discuss two choices of kernel we advocate: a Gaussian kernel and a Vandermonde-based kernel.

\subsection{A first choice: the Gaussian kernel}
In order for $\ma{K}$ to be a good candidate for coresets, it needs to verify the following two properties:
\begin{itemize}
	\item As indicated by the theorems, the diagonal entries $\ma{K}_{ii}$ should increase as the associated $\sigma_i$ increases.
	\item As indicated by the variance equation of Eq.~\eqref{eq:var}, off-diagonal elements should be as large as possible (in absolute value) given the eigenvalue constraints. In fact, we cannot set all non-diagonal entries of $\ma{K}$ to large values as the matrix's 2-norm would rapidly be larger than 1. We thus need to choose the best pairs $(i,j)$ for which it is worth setting a large value of $\ma{K}_{ij}$.  A first glance at the variance equation indicates that the larger $f(x_i,\theta)f(x_j,\theta)$ is, the larger $\ma{K}_{ij}$ should be, in order to decrease the variance as much as possible. Recall  nevertheless that in the coreset setting, all sampling parameters should be independent of $\theta$. The off-diagonal elements should thus verify the following property: the larger is the correlation between $x_i$ and $x_j$ (the more similar are $f(x_i,\theta)$ and $f(x_j,\theta)$ for all $\theta$), the larger $\ma{K}_{ij}$ should be. 
\end{itemize}

We show in the following in what ways the choice of marginal kernel
$$\ma{K} = \ma{L}(\ma{I}+\ma{L})^{-1}$$
with $\ma{L}$ the Gaussian kernel matrix with parameter $\tau$:
\begin{align*}
\forall(i,j)\quad\ma{L}_{ij} = \exp^{-\frac{\norm{\vec{x}_i-\vec{x}_j}^2}{2\tau^2}},
\end{align*}
is a good candidate to build coresets for $k$-means (the linear regression case is discussed later). % Note that $\ma{L}$ is called the $L$-ensemble associated to $\ma{K}$~\cite{kulesza_determinantal_2012}. 
Let us write $\ma{U} = (\fou_1|\ldots|\fou_n)$ the orthonormal eigenvector basis of $\ma{L}$ and $\ma{\Lambda}=\text{diag}(\lambda_1|\ldots|\lambda_n)$ its diagonal matrix of sorted  eigenvalues, $0\leq\lambda_1\leq\ldots\leq\lambda_n$. 
$\ma{U}$ and $\Lambda$ naturally depend on $\tau$. One shows for instance that, with respect to $\tau$, $\lambda_n$ is a monotonically increasing function between $1$ and $n$.  %via Gershgorin's theorem: $\lambda_n\leq\max_{i} \sum_j \ma{L}_{ij}$, such that $\lambda_n$ decreases as $\tau$ decreases. Moreover, $\lambda_n$ tends to unity as $\tau$ tends to $0$ (as $\ma{L}$ tends to $\ma{I}$), and to $n$ as $\tau$ tends to $+\infty$ (as $\ma{L}$ tends to the constant matrix equal to 1). With respect to $\tau$, $\lambda_n$ is thus a monotonically increasing function between $1$ and $n$. 

Concerning the off-diagonal elements of $\ma{K}$, let us first note that if $x_i$ and $x_j$ are correlated (that is, in the $k$-means setting, if they are close to each other), then
\begin{align*}
\ma{K}_{ij} = \sum_{k} \frac{\lambda_k}{1+\lambda_k} u_k(i)u_k(j)
\end{align*}
should be large in absolute value. In fact, in the limit where $x_i = x_j$, then $\forall k, u_k(i) = u_k(j)$ and $\ma{K}_{ij} = \ma{K}_{ii} = \ma{K}_{jj}$. The determinant of the $2\times 2$ submatrix of $\ma{K}$ indexed by $i$ and $j$ is therefore null: sampling both will never occur. Thus, the closer are $x_i$ and $x_j$, the lower is the chance of sampling both jointly. 
Moreover, if $x_i$ and $x_j$ are far from each other (for instance, in different clusters), then the entries $i$ and $j$ of $\ma{L}'s$ eigenvectors will be very different. For instance, say the data set contains two well separated clusters of similar size. If the Gaussian parameter $\tau$ is set to the size of these clusters, then the kernel matrix $\ma{L}$ will be quasi-block diagonal, with each block corresponding to the entries of each cluster. Also, each eigenvector $\fou_k$ will have energy either in one cluster or the other such that $\ma{K}_{ij}$ is necessarily small if $i$ and $j$ belong to different clusters, and the event of sampling both jointly is probable.

Concerning the probability of inclusion of $i$, we have:
\begin{align*}
\ma{K}_{ii} = \sum_{k} \frac{\lambda_k}{1+\lambda_k} v_i(k)^2,
\end{align*}
where $\vec{v}_i$ is the vector of size $n$ verifying $\forall k, v_i(k) = u_k(i)$. For all $i$, $\norm{\vec{v}_i}^2=1$. The probability of inclusion is thus directly linked to the values of $k$ that contain the energy of $\vec{v}_i$: the more the energy of $\vec{v}_i$ is contained on high values of $k$, the larger is the probability of inclusion. 
Say we are again in a situation where the clusters and the choice of Gaussian parameter $\tau$ are such that $\ma{L}$ is quasi block diagonal. Within each block, the eigenvector associated with the highest eigenvalue corresponds approximately to the constant vector. These eigenvectors being normalized, the associated entry of $v_i(k)$ is thus approximately equal to $1/\sqrt{\#C_i}$ where $\#C_i$ is the size of the cluster containing data $x_i$. Typically, if the cluster is small, that is, if $\#C_i$ tends to 1, the associated entry $v_i(k)$ tends to 1 as well, such that all the energy of $v_i$ is drawn towards high values of $k$, thus increasing the probability of inclusion of $i$. In other words, the more isolated, the higher the chance of being sampled. This corresponds to the intuition one may obtain for the sensitivity $\sigma_i$. It has indeed been shown that the sensitivity may be interpreted as a measure of outlierness~\citep{lucic_linear-time_2016}. 

In the linear regression case, a similar argumentation is possible, up to the fact that point $i$ can be an outlier from the point of view of $\vec{x}_i$ and/or $y_i$, such that the kernel should take both into account: we suggest the Gaussian kernel in dimension $d+1$:
\begin{align*}
\forall(i,j)\quad\ma{L}_{ij} = \exp^{-\frac{\norm{\vec{z}_i-\vec{z}_j}^2}{2\tau^2}},
\end{align*}
with $\vec{z}_i=[\vec{x}_i^\top, y_i]^\top\in\mathbb{R}^{d+1}$.

In both contexts, we thus advocate to sample DPPs via a Gaussian kernel $L$-ensemble. We now move on to detailing an efficient sampling implementation.

\subsubsection{Efficient implementation}

Sampling exactly a DPP from the Gaussian $L$-ensemble verifying
\begin{align*}
\forall(i,j)\qquad \ma{L}_{ij} = \exp^{-\frac{\norm{\vec{x}_i-\vec{x}_j}^2}{2\tau^2}}
\end{align*}
consists in the following steps:
\begin{enumerate}
	\item Compute $\ma{L}$.
	\item Diagonalize $\ma{L}$ in its set of eigenvectors $\{\fou_k\}$ and eigenvalues $\{\lambda_k\}$.
	\item Sample a DPP given $\{\fou_k\}$ and $\{\lambda_k\}$ via Algorithm~1  of~\citet{kulesza_determinantal_2012}.
\end{enumerate}
Step 1 costs $\mathcal{O}(n^2d)$, step 2 costs $\mathcal{O}(n^3)$, step 3 costs  $\mathcal{O}(n\mu^3)$, where we recall that $\mu$ is the expected number of samples of the DPP. This naive approach is thus not practical. We  detail in Appendix~\ref{app:implementation}  how to reduce the overall complexity to $\mathcal{O}(n\mu^2)$, by 1/~taking advantage of Random Fourier Features (RFF)~\citep{rahimi_random_2008} to estimate a low dimensional representation $\Psi\in\mathbb{R}^{2r\times n}$ of the $L$-ensemble $\ma{L}\simeq\ma{\Psi}^\top\ma{\Psi}$, where $r$ is the chosen number of features; and 2/~running a DPP sampling algorithm adapted to such a low rank representation. 

In the experimental section, we will concentrate on $m$-DPPs as they are simpler to compare with state of the art methods that all have a fixed known-in-advance number of samples. The overall $m$-DPP sampling algorithm adapted to the $k$-means problem that we will consider is summarized in Algorithm~\ref{alg:summary}: given the data $\mathcal{X}$, the number of desired samples $m$, and the Gaussian parameter $\tau$, it outputs a weighted set of $m$ samples $\mathcal{S}$ that is a good candidate to be a coreset if $m$ is large enough.  The runtime to build $\ma{\Psi}$ is $\mathcal{O}(ndr)$; to compute $\ma{C}$ and diagonalize it is $\mathcal{O}(nr^2)$; to sample a $m$-DPP given this dual eigendecomposition is $\mathcal{O}(nm^2)$. Given that $r$ is set to a few times $m$, the overall runtime of Algorithm~\ref{alg:summary} is $\mathcal{O}(ndm + nm^2)$. 

Given a number of samples $m$ to draw, how should one set the Gaussian parameter $\tau$? The larger is $\tau$, the more repulsive is the $m$-DPP, and the smaller is the numerical rank of $\ma{\Psi}$ (the number of eigenvalues $\nu$ such that $n\nu$ is larger than the machine's precision). Now, numerical instabilities arise while sampling an $m$-DPP if the numerical rank of $\ma{\Psi}$ decreases below $m$: $\tau$ should not be set too large. Also, the smaller is $\tau$, the closer is $\ma{L}$ to the identity matrix, such that the closer is the $m$-DPP to uniform sampling without replacement: $\tau$ should not be set too small. We will see in the following experimental section how the choice of  $\tau$ affects results.

\begin{algorithm}[tb]
	\caption{The Gaussian kernel coreset sampling heuristics}
	\label{alg:summary}
	\begin{algorithmic}
		\Input $\mathcal{X}=\{x_i\}$ a set of $n$ points in $\mathbb{R}^d$, a Gaussian kernel parameter $\tau$, a number of samples $m$\\
		$\bm{\cdot}$ Draw $r\geq\mathcal{O}(m)$ random Fourier vectors associated to the Gaussian kernel with parameter $\tau$\\
		$\bm{\cdot}$ Compute the associated RFF matrix $\ma{\Psi}\in\mathbb{R}^{2r\times n}$ as explained in Appendix~\ref{subsec:RFF}\\
		$\bm{\cdot}$ Compute $\ma{C}=\ma{\Psi}\ma{\Psi}^\top\in\mathbb{R}^{2r\times 2r}$ the dual representation\\
		$\bm{\cdot}$ Compute the eigendecomposition of $\ma{C}$: obtain eigenvectors $\{\vec{v}_k\}$ and eigenvalues $\{\nu_k\}$\\
		$\bm{\cdot}$ Draw a sample $\mathcal{S}$ from a $m$-DPP with $L$-ensemble $\ma{L}=\ma{\Psi}^\top\ma{\Psi}$ as explained in Appendix~\ref{subsec:fast_mDPP}.\\
		$\bm{\cdot}$ Compute the marginal probabilities $\pi_s$ for all $\vec{x}_s\in\mathcal{S}$ as explained in Appendix~\ref{subsec:fast_mDPP}, and set weights $\omega(\vec{x}_s)=1/\pi_s$.
		\Output $\{\mathcal{S},\omega\}$ a weighted sample of size $m$.
	\end{algorithmic}
\end{algorithm}

\subsection{A second choice: a projective DPP based on the Vandermonde matrix}
A second choice of DPP sampling, that derives from our analysis, is the projective DPP with $m$ samples from the $L$-ensemble $\ma{L}=\ma{VV}^\top$ where $\ma{V}$ is the Vandermonde matrix (discussed in Section~\ref{sec:polynomial-dpps}). This choice has several advantages over the Gaussian kernel:
\begin{itemize}
	\item $\ma{V}$ takes $\mathcal{O}(nm)$ operations to compute: the overall $m$-DPP sampling cost is thus naturally $\mathcal{O}(nm^2)$, with no need for any approximation technique.
	\item no particular scale $\tau$ is introduced.
\end{itemize}
This choice however has the drawback that in dimension higher than $1$, not all values of $m$ are allowed (only values of $m$ for which there exists $\phi\in\mathbb{N}$ s.t. $m = {{\phi+d}\choose {\phi}}$), as explained at the end of Section~\ref{sec:polynomial-dpps}. 

\begin{algorithm}[tb]
	\caption{The Vandermonde-based coreset sampling heuristics}
	\label{alg:summary_vdm}
	\begin{algorithmic}
		\Input $\mathcal{X}=\{x_i\}$ a set of $n$ points in $\mathbb{R}^d$, a number of samples $m$\\
		$\bm{\cdot}$ $m$ should verify: $\exists \phi\in\mathbb{N}$ such that $m = {{\phi+d}\choose \phi}$.\\
		$\bm{\cdot}$ Compute the Vandermonde matrix $\ma{V}\in\mathbb{R}^{n\times m}$.\\
		$\bm{\cdot}$ Compute the $(\ma{Q}\in\mathbb{R}^{n\times m},\ma{R}\in\mathbb{R}^{m\times m})$ decomposition of $\ma{V}$: $\ma{V}=\ma{QR}$ with $\ma{Q}^\top\ma{Q}=\ma{I}_m$ and $\ma{R}$ an upper triangular matrix.\\
		$\bm{\cdot}$ Draw a sample $\mathcal{S}$ from a projective DPP with $L$-ensemble $\ma{L}=\ma{QQ}^\top$ as explained in Algorithm~\ref{alg:sampling_DPP_efficient}.\\
		$\bm{\cdot}$ Compute the marginal probabilities $\pi_s$ for all  $\vec{x}_s\in\mathcal{S}$ with $\pi_s = \norm{\ma{Q}(s,:)}^2$ the energy of the $s$-th line of $\ma{Q}$; and set weights $\omega(\vec{x}_s)=1/\pi_s$.
		\Output $\{\mathcal{S},\omega\}$ a weighted sample of size $m$.
	\end{algorithmic}
\end{algorithm}

\subsection{Alternative algorithms for sampling DPPs, and potential improvements}
\label{sec:potential-improvements}

	The algorithm we suggest scales in our experience rather well with $n$, and
	makes it practical to find coresets with $n$ in the millions or more.  Our method
	scales more poorly in $m$, the number of points retained, which in practice
	should be in the hundreds at most. Recall that $m$ should scale roughly as the
	intrisic dimension of the parameter space: it is therefore possible that in certain
	difficult problems no reliable coreset can be found if\,\footnote{One
		might argue that in such cases the coreset methodology is of dubious value
		anyway.} $m<1,000$. With that in mind, we now review other methods for sampling DPPs.
	
	As an alternative to direct sampling of the kind used here, MCMC methods have
	been suggested several times~\citep[\textit{e.g.,}][]{anari_monte_2016}, and the earliest reference we could find is \citet{belabbas_spectral_2009}. The most basic kind starts with a set of points sampled
	uniformly, and uses random swapping moves: at each iteration, a point from the
	current set may be replaced by one not in the set.\footnote{A more advanced algorithm by \citet{gautier_zonotope_2017} mixes faster than the basic algorithm
		outlined here, but the iterations are more involved.} Acceptance probabilities are
	set so that the limiting distribution of the chain is the correct DPP. Each iteration
	has cost $\mathcal{O}(m^2)$, and approximately $\mathcal{O}(n)$ such iterations are required for
	mixing \citep{hermon_modified_2019}.
	The total cost is therefore the same as in our method ($\mathcal{O}(nm^2)$), so not much
	gain is to be expected here. However, there is no need for a low-rank
	approximation of the kernel such as the RFF approximation used here. In a
	nutshell, MCMC techniques sample approximately from the correct kernel instead
	of sampling exactly from an approximate kernel: which is better is as yet unknown but an
	interesting problem in itself. 
	
	There are two immediate strategies for increasing $m$. One is to use a crude
	heuristic for dividing the original data set into $p$  different subsets, and sampling a DPP
	independently from each subset. This is equivalent to using a block-diagonal
	kernel, and along these lines there is a less radical approach, which is to
	force the kernel to be sparse and exploit sparsity in the sampling.
	\citet{poulson_high-performance_2019} shows how to exploit sparsity for sampling DPPs when the
	\emph{marginal} kernel is sparse. Unfortunately, we use L-ensembles here, and one
	would have to adapt the tools given by Poulson to L-ensembles.
	A different strategy to increase $m$ is to sample the DPP several times rather
	than just once. The resulting sample has less diversity but is much cheaper to
	generate. One can take advantage of recent methods that use pre-processing for
	speeding up repeated sampling of the same DPP
	\citep{gillenwater_tree-based_2019, derezinski_exact_2019}. Here the challenge is to find
	the right trade-off between computational cost and repulsion, which is again an
	interesting question for future research.

\section{Experiments}
\label{sec:experiments}
\subsection{Different strategies to compare...}
We will empirically compare results obtained with the five following approaches:
\begin{enumerate}
	\item \texttt{m-DPP} : The strategy summarized in Algorithm~\ref{alg:summary}.
	\item \texttt{PolyProj-DPP} : The strategy summarized in Algorithm~\ref{alg:summary_vdm}.
	\item \texttt{matched iid} : An iid sampling strategy with replacement, matched to either \texttt{m-DPP} or \texttt{PolyProj-DPP} (depending on the context). More precisely, $m$ samples are drawn iid with replacement, the probability of selecting $x_i$ at each draw being set to $p_i=\pi_i/m$, where $\pi_i$ is the marginal probability of drawing $x_i$ in \texttt{m-DPP} (or \texttt{PolyProj-DPP}). 
	
	\item \texttt{uniform iid} : Uniform iid sampling with replacement.
	
	%	\item A weighted sample obtained with Algorithm~\ref{alg:summary} but replacing the $L$-ensemble $\ma{\Psi}^\mathcal{A}djoint\ma{\Psi}$ by the projective $L$-ensemble $\ma{P}=\ma{W} \ma{W}^\mathcal{A}djoint$ where  $\ma{W}\in\mathbb{R}^{n\times m}$  concatenates all the   eigenvectors $\fou_k=\frac{1}{\sqrt{\nu_k}} \ma{\Psi}^\mathcal{A}djoint \vec{v}_k$ reconstructed from the $m$ first eigenvectors $\{\vec{v}_1,\ldots,\vec{v}_m\}$ of $\ma{C}$ (associated to the $m$ largest eigenvalues $\nu$). 
	
	\item \texttt{sensitivity iid} : The current state of the art iid sampling based on a bi-criteria approximation to upper bound the sensitivity (Algorithm~2 of~\citealt{bachem_practical_2017}), or, if available (for instance in the case of $1$-means and linear regression), an analytical formula of the sensitivity. 
\end{enumerate}
For the three iid methods (methods 3, 4 and 5), we will use the importance sampling estimator adapted to iid sampling of Eq.~\eqref{eq:imp_sampling_est_iid}. For methods 1 and 2, we will use the importance sampling estimator adapted to correlated  sampling of Eq.~\eqref{eq:imp_sampling_est_corr}.

Empirically, when the ambient dimension $d$ is small, performance of all methods is enhanced if the weights in $\hat{\ma{L}}$ are set via Voronoi cells rather than set to inverse probabilities: given the sample $\mathcal{S}$ of size $m$, compute its Voronoi tessellation in $m$ cells, and associate to each sample $\vec{x}_s$ a weight $\omega(\vec{x}_s)$ equal to the number of datapoints in its associated Voronoi cell. We will call the associated cost estimators $\hat{L}$ the Voronoi estimators. %We compare, for each method, the importance sampling estimator vs. its Voronoi counterpart. 

For completeness, we compare all these methods with another negatively correlated sampling method called $D^2$-sampling (commonly used for $k$-means++ seeding, see~\citealt{arthur_k-means++:_2007}):
\begin{enumerate}
	\item[6)] \texttt{D$^2$} : sample the first element of $\mathcal{S}$ uniformly at random. Each subsequent element of $\mathcal{S}$ is drawn according to a probability proportional to the squared distance to the closest of the already sampled elements. The marginal probabilities are not known in this algorithm, so we will only be able to build the associated Voronoi cost estimator. 
\end{enumerate}

To measure the performance of each method, we will empirically estimate the probability that, given the method's sampled weighted subset, it verifies the coreset property of Eq.~\eqref{eq:coresets} for a given randomly chosen $\theta$ (setting $\epsilon$ to $0.1$). On the artificial data models we investigate, we estimate this probability via $50$ randomly chosen $\theta$ on $1000$ realizations of the data. On the real-world data sets, we estimate this probability via $5000$ randomly chosen $\theta$. We will in general plot this probability versus the number of samples: the closer it is to $1$, the better the sampling method for coresets. 

In Sections~\ref{subsec:expes_sp_feat_SBM} and~\ref{subsec:expes_MNIST}, we will not only compare the coreset property of the samples obtained by each method, we will also compare the result of Lloyd's classical $k$-means heuristics~\citep{lloyd_least_1982} performed on the entire data versus the result obtained on the weighted samples of each method.  To be precise, once the $k$-means heuristics on the weighted subset outputs $k$ centroids, we classify all nodes (sampled or not) according to their closest distance to the centroids: this gives us a partition that we then compare using the Adjusted Rand (AR) similarity index~\citep{hubert_comparing_1985} to the ground truth associated to the data set. The AR index is a number between $-1$ and $1$: the closer it is to $1$, the closer are the partitions, the better the sampling method.

\subsection{...on different data sets}
\subsubsection{To start with: two well controlled cases}
\label{sec:well_controlled_case}
We start with two perfectly controlled cases (for which we derived the sensitivities analytically -- see the first and third lemmas of Appendix~\ref{app:proof_lemmas})): 
\begin{itemize}
	\item the $1$-means case, for which we show that, supposing without loss of generality that the data is centered ($\sum_j x_j =0$), the sensitivity verifies the following analytic form:
	\begin{align*}
	\sigma_i = \frac{1}{n}\left(1+\frac{\norm{x_i}^2}{v}\right),
	\end{align*}
	where $v=\frac{1}{n}\sum_{x\in\mathcal{X}} \norm{x}^2$.
	\item the linear regression case, for which we show that:
	\begin{align*}
	\forall i\qquad	
	\sigma_i =x_i^\top \ma{H}^{-1} x_i + \frac{\left(y_i - y^*_i \right)^2}{\norm{y-y^*}^2}
	\end{align*}
	where $\ma{H}=\ma{X}^\top \ma{X}$ and $y^*$ reads $y^*=\ma{X}\theta^* =\ma{XH}^{-1}\ma{X}^\top y$. 
\end{itemize} 
We are thus able to compare our method versus the ideal iid sampling scheme for which we set $p_i$, the probability of drawing $x_i$, exactly to its ideal value given in Theorem~\ref{thm:iid}: $p_i=\sigma_i/\mathfrak{S}$ ($=\sigma_i/2$ for $1$-means, $=\sigma_i/(d+1)$ for linear regression).\\

\begin{figure}
	\centering
	a)\includegraphics[width=0.3\columnwidth]{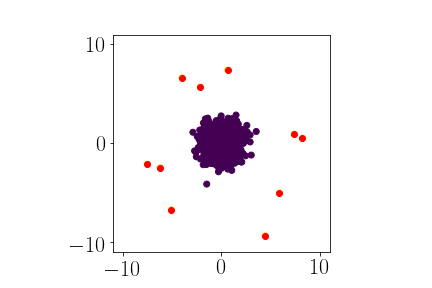}\hfill
	b)\includegraphics[width=0.3\columnwidth]{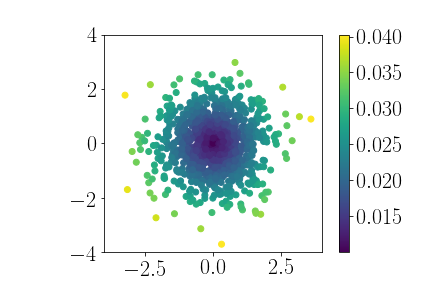}\hfill
	c)\includegraphics[width=0.3\columnwidth]{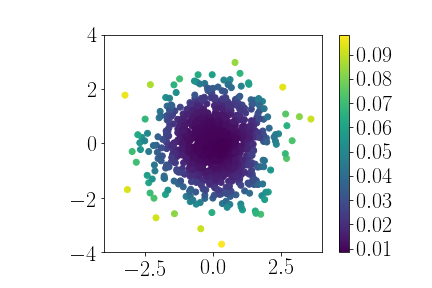}
	
	\caption{a) A realization of an artificial data set of $n=1000$ data points; blue points are drawn from an isotropic Gaussian, a proportion $q=0.01$ of the points are drawn as outliers (displayed in red). b)~In a case without outliers, and for $m=20$, we represent the  inverse importance sampling weights of \texttt{sensitivity iid}, \ie, $m\sigma_i/\mathfrak{S}$, using the exact analytical formulation of the sensitivity in the $1$-means case (see Lemma~\ref{lemma:sensi_1means}). c)~On the same data  realization, and also setting $m=20$, we represent the inverse importance sampling weights of \texttt{m-DPP}: the inclusion probability $\pi_i$. }
	\label{fig:2d_Gaussian}
\end{figure}

\begin{figure}
	\centering
	\begin{tabular}{c|ccc|cc}
		& \multicolumn{3}{c|}{Voronoi weights} & \multicolumn{2}{c}{Importance sampling weights}\\\hline &&&&&\\
		\multirow{1}{*}[6em]{$d=2$}	&&\includegraphics[width=0.4\columnwidth]{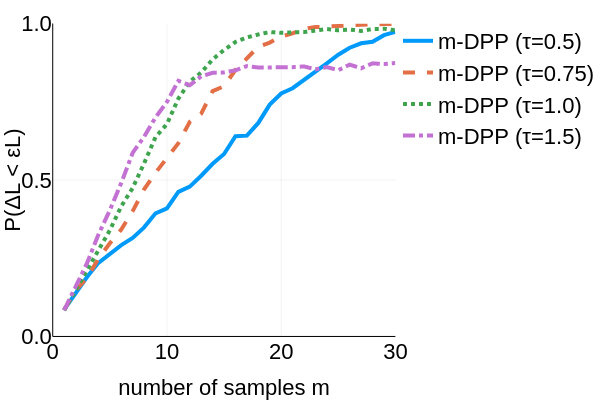} &&&
		\includegraphics[width=0.4\columnwidth]{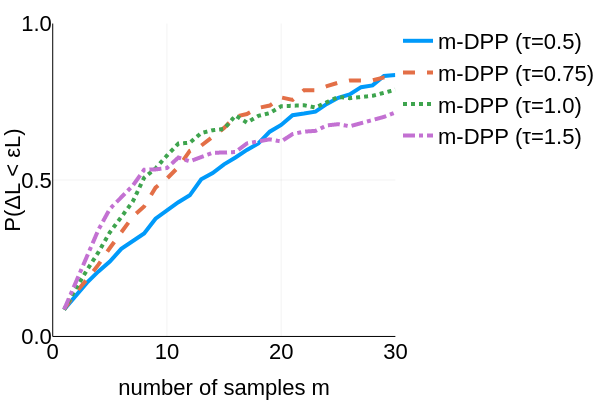}\\\hline &&&&&\\
		\multirow{1}{*}[6em]{$d=20$}&& \includegraphics[width=0.4\columnwidth]{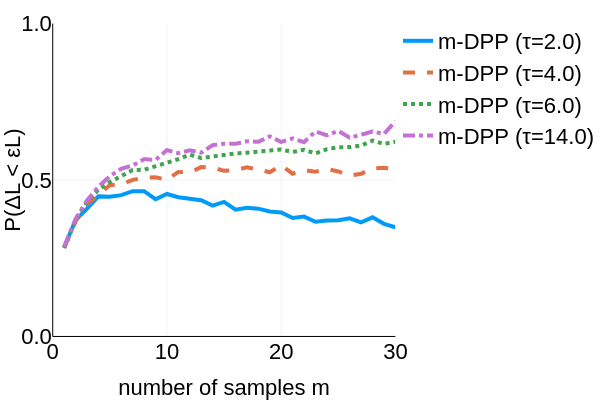}&&&
		\includegraphics[width=0.4\columnwidth]{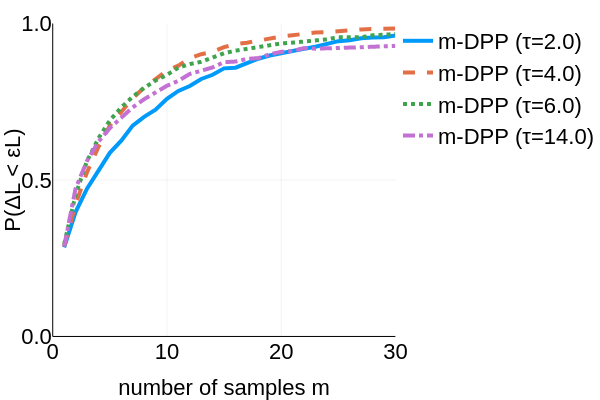}\\\hline &&&&&\\
		\multirow{1}{*}[6em]{$d=100$}&& \includegraphics[width=0.4\columnwidth]{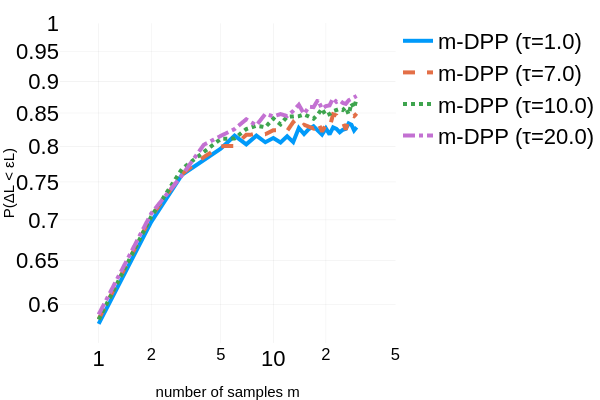}&&&
		\includegraphics[width=0.4\columnwidth]{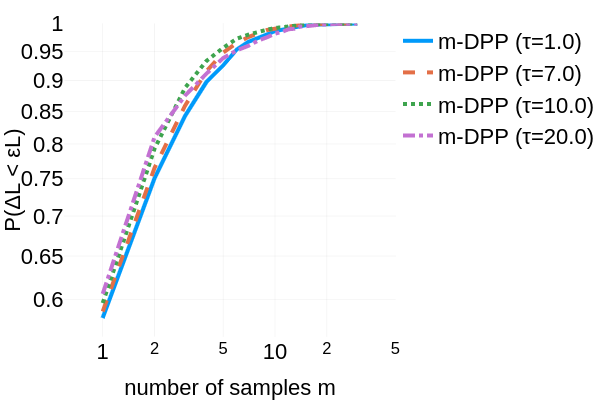}
	\end{tabular}
	\caption{Performance of \texttt{m-DPP} on the $1$-means problem, versus the dimension $d$, the parameter $\tau$ of the Gaussian kernel and the choice of weights (Voronoi or importance sampling weights) in the cost estimator. The two bottom figures are plotted in log-log scale, for readability.}
	\label{fig:perf_1_means}
\end{figure}

\noindent\textbf{Experiments with $1$-means.} 
We will work on a simple isotropic Gaussian data set of $n=1000$ points in dimension $d=2, 20$ or $100$. A percentage $q$ of the $n$ points are drawn as outliers (uniformly in the ambient space and far from the Gaussian mean). An instance of such a data set in $d=2$ dimensions, and with $q=0.01$ is shown in Figure~\ref{fig:2d_Gaussian}a.

We start by showing in Figure~\ref{fig:perf_1_means} the results of \texttt{m-DPP} versus the number of dimensions and the  choice of parameter $\tau$ for the Gaussian kernel. All shown results are with $q=0$ (no outlier) and with a number of random Fourier features $r=200$. Several  comments are in order. Firstly, compared to the importance sampling estimator, the Voronoi estimator produces good results in low dimensions, and fails as the dimension increases. Secondly, the performance of all methods increase and uniformize as the dimension increases. This is due to the fact that in large dimensions, interpoint distances tend to uniformize such that any pair of points tend to be representative of all interpoint distances, thus simplifying the problem of finding good coresets. This may also explain why the choice of $\tau$ is less crucial in higher dimension. In low dimensions, however, the choice of $\tau$ has a strong impact on performance. The best choice for $\tau$ depends in fact on the number of samples $m$ one requires: as $m$ increases, $\tau$ should be set smaller. This is in fact natural: if one desires a very short summary of the data set (small $m$), the repulsion of the DPP has to be strong in order to sample a diverse subset. Whereas if the length of the summary is less constrained, $\tau$ should be decreased to allow for a less coarse-grained description. This observation leads to the natural question of the optimal $\tau$ given the data and $m$. We currently lack of a satisfying answer to this question, both theoretically and empirically. A usual heuristics in kernel methods is to set $\tau$ to the average (or median) interdistance of the points in the data set. In the experiments of Figure~\ref{fig:perf_1_means}, the average interdistance corresponds to $\tau \simeq 1.7, 6.2$ and $14.0$ for $d=2, 20$ and $100$ respectively, which give in fact a good order of magnitude for the choice of $\tau$. In the following, to simplify the discussion, we will sometimes set $\tau$ to be the average interdistance, that we will denote by $\bar{\tau}$. \\

\begin{figure}
	\centering
	\begin{tabular}{c|ccc|cc}
		& \multicolumn{3}{c|}{Voronoi weights} & \multicolumn{2}{c}{Importance sampling weights}\\\hline &&&&&\\
		\multirow{1}{*}[4em]{$d=2$}	&&\includegraphics[width=0.4\columnwidth]{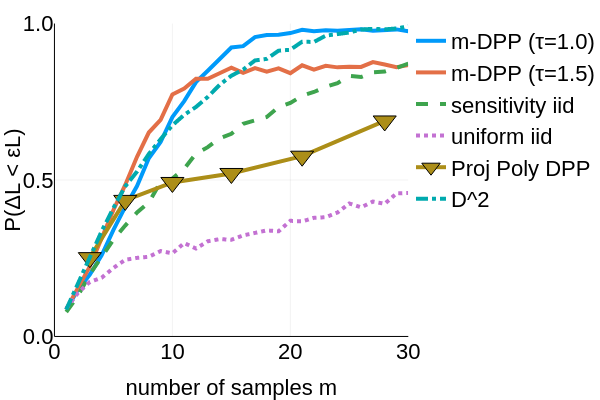} &&&
		\includegraphics[width=0.4\columnwidth]{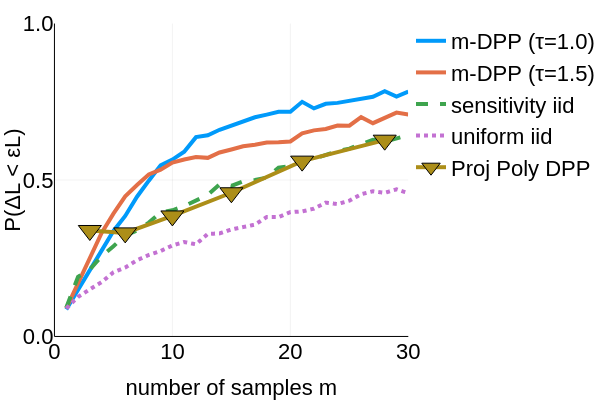}\\\hline &&&&&\\
		\multirow{1}{*}[4em]{$d=20$}&& \includegraphics[width=0.4\columnwidth]{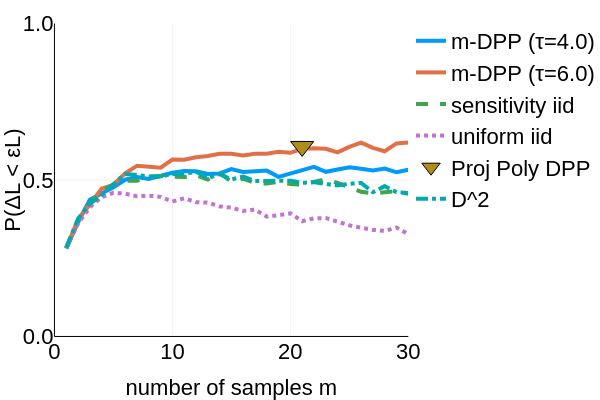}&&&
		\includegraphics[width=0.4\columnwidth]{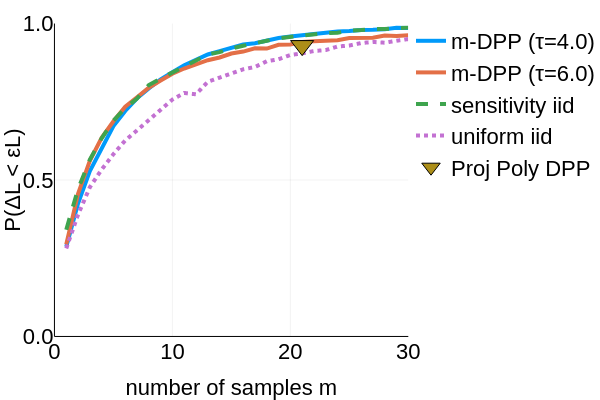}\\\hline &&&&&\\
		\multirow{1}{*}[4em]{$d=100$}&& \includegraphics[width=0.4\columnwidth]{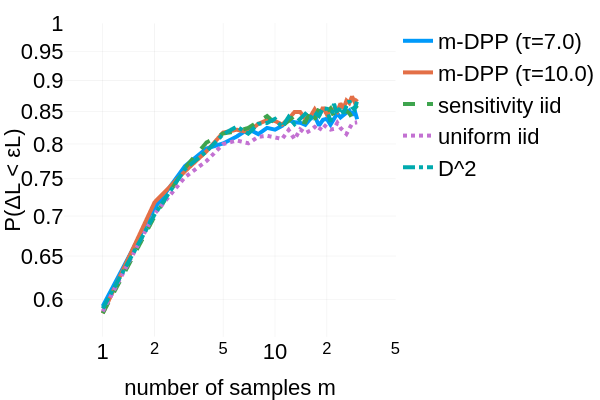}&&&
		\includegraphics[width=0.4\columnwidth]{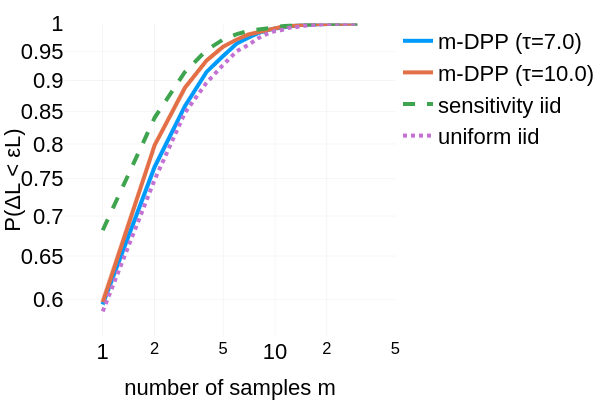}
	\end{tabular}
	\caption{Performance comparison of different sampling methods on the $1$-means problem, versus the dimension $d$ and the choice of weights (Voronoi or importance sampling weights) in the cost estimator. The two bottom figures are plotted in log-log scale, for readability.}
	\label{fig:compare_perf_1_means}
\end{figure}

We compare next the performance of several methods in Figure~\ref{fig:compare_perf_1_means}. One observes that the superior performance of the Voronoi estimator over the importance sampling estimator in low dimension $d$ is verified for all methods. Moreover, as the dimension increases, all methods converge to the performance of the uniform iid sampling method. 
Finally, \texttt{m-DPP} associated with Voronoi weights is competitive with \texttt{D$^2$} in low $d$; and, regardless of how one chooses the weights, \texttt{$m$-DPP} has a clear edge over the sensitivity-based iid random sampling (the lower the dimension, the clearer the edge). Finally, \texttt{PolyProj-DPP} matches the performance of \texttt{sensitivity iid} for importance sampling weights. For Voronoi weights, the results for $d=2$ and $d=20$ are contradictory and we cannot conclude. \\

In order to clarify further discussion, we will now focus on the importance sampling estimated cost. One should keep in mind that in low dimensions, Voronoi-based estimated costs usually perform well, but fail (sometimes drastically) as the dimension increases. \\

A natural question arises at this point: is the observed edge of \texttt{m-DPP} over \texttt{sensitivity iid} due to a better probability of inclusion of the point process? Or is it truly due to the negative correlations induced by the determinantal nature of our method? In fact, we compare in Figure~\ref{fig:2d_Gaussian} the probability of inclusion for \texttt{sensitivity iid} versus \texttt{m-DPP}: they have a similar general behavior but are nevertheless quantitatively different. In Figure~\ref{fig:compare_perf_DPP_vs_iid} (left), we compare \texttt{m-DPP} and \texttt{PolyProj-DPP} versus \texttt{matched iid} and \texttt{sensitivity iid}: the observed edge is clearly due to the negative correlations induced by the determinantal nature of our method. As expected from Corollary~\ref{corollary:mDPP_kmeans}, the best inclusion probability is based on the sensitivity. Nevertheless, the figure shows that even if it is not set to its ideal value (as in \texttt{m-DPP} and \texttt{PolyProj-DPP}), one can still improve the performance by inducing negative correlations. \\

\begin{figure}
	\centering
	\includegraphics[width=0.4\columnwidth]{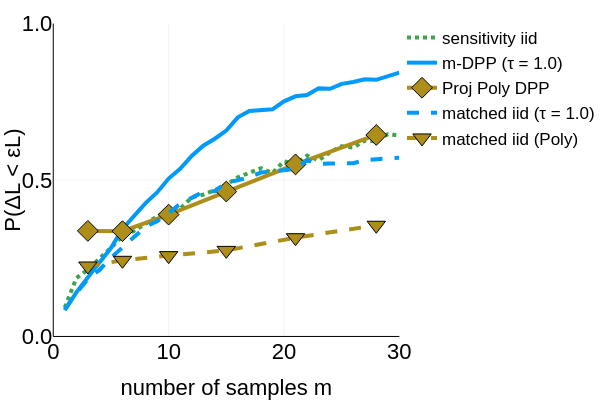}\hspace{1cm}
	\includegraphics[width=0.4\columnwidth]{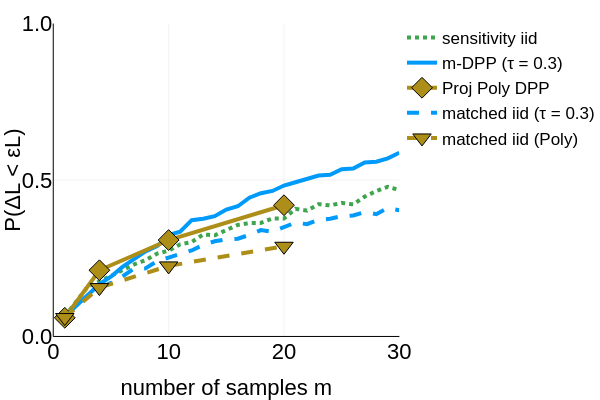}
	\caption{Comparison between \texttt{m-DPP} and $\texttt{PolyProj-DPP}$ versus \texttt{matched iid} and \texttt{sensitivity iid} on the 1-means problem (left) and the linear regression problem (right), for $d=2$.}
	\label{fig:compare_perf_DPP_vs_iid} 
\end{figure}

\begin{figure}
	\centering
	\hspace{1cm}	\includegraphics[width=0.40\columnwidth]{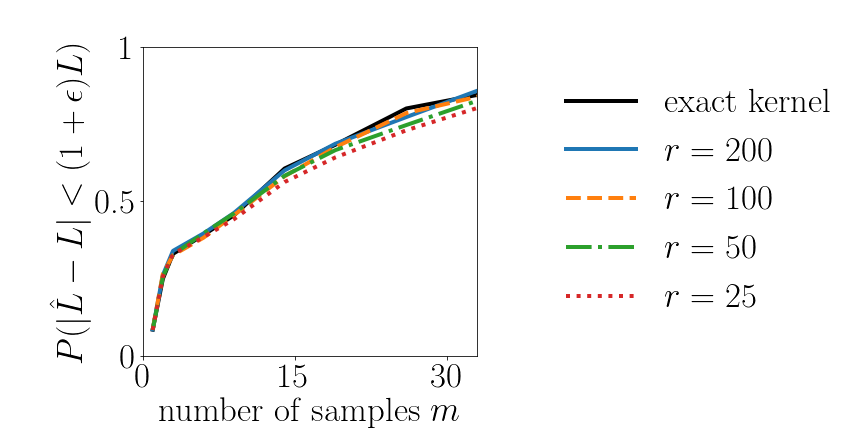} \hfill
	\includegraphics[width=0.4\columnwidth]{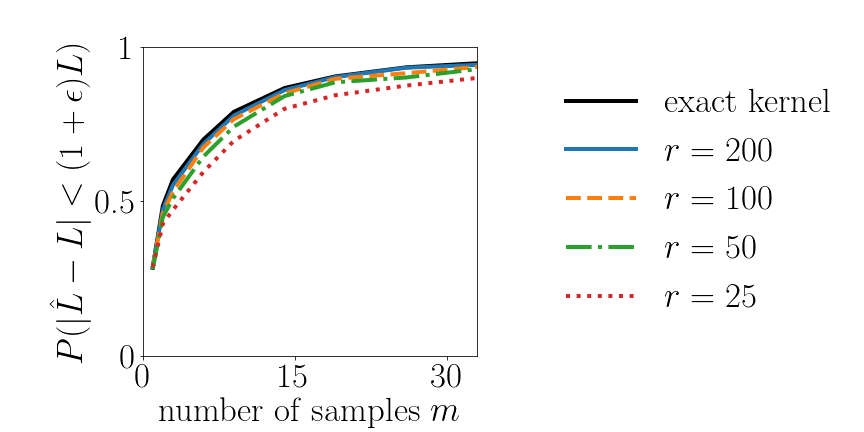} \hspace{1cm}
	\caption{Performance comparison of \texttt{m-DPP} on the $1$-means problem versus the number of RFF $r$, for $d=2$ (left) and $d=20$ (right). }
	\label{fig:compare_perf_DPP_vs_r}
\end{figure}

\begin{figure}
	\centering
	\includegraphics[width=0.40\columnwidth]{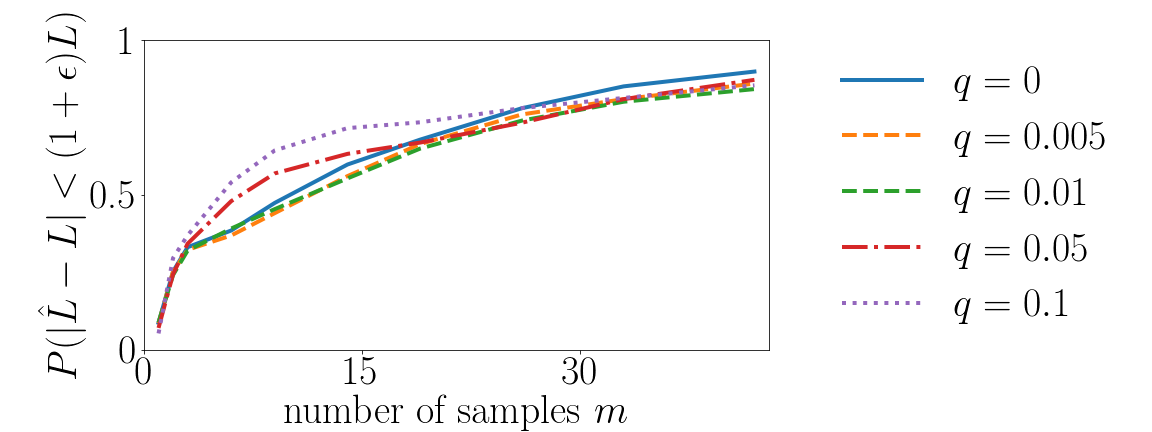} \hfill
	\includegraphics[width=0.29\columnwidth]{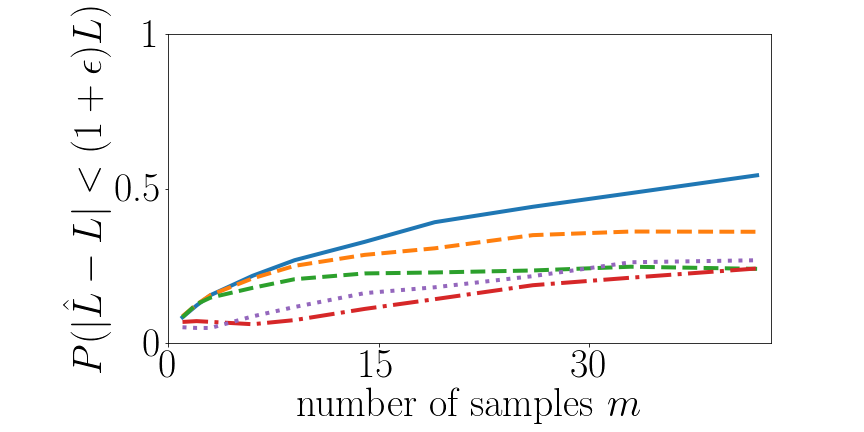} \hfill
	\includegraphics[width=0.29\columnwidth]{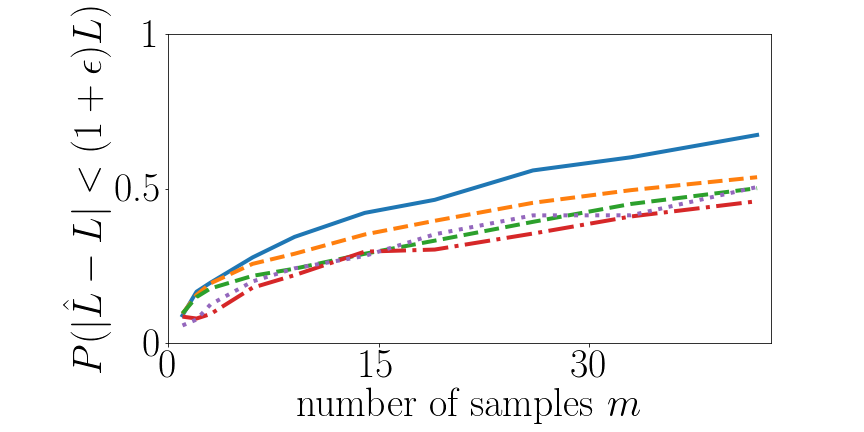} 
	\caption{Performance comparison of \texttt{m-DPP} (left), \texttt{uniform iid} (middle) and \texttt{sensitivity iid} (right) on the $1$-means problem versus the percentage of outliers $q$.}
	\label{fig:compare_perf_DPP_vs_q}
\end{figure}

For completeness, we still need to discuss the impact of two variables: the number of random Fourier features $r$ used in \texttt{m-DPP}, and the percentage of outliers $q$ in the data. In the following, we set $\tau$ to $\bar{\tau}$, the average interdistance.  Figure~\ref{fig:compare_perf_DPP_vs_r} shows the impact of the choice of $r$ on performances: as expected, as $r$ increases, performance increases, and as $d$ increases, performances become more sensitive to the choice of $r$. The impact of the choice of $r$ is nevertheless very limited: setting $r$ to a multiple of $m$ has been a safe choice in all our experiments. Finally, Figure~\ref{fig:compare_perf_DPP_vs_q} shows the impact of the percentage of outliers $q$ on performances. Empirically, we see here that outliers have a smaller impact on DPP sampling than on uniform or sensitivity-based iid sampling.\\

\noindent\textbf{Experiments with linear regression.} The data $\mathcal{X}=(\vec{x}_1,\ldots,\vec{x}_n)$ are generated by sampling $n$ points in the hypercube $[0,1]^d$. Each entry of the vector $y$ is also sampled uniformly from $[0,1]$. The outlier percentage $q$ is set to zero. We show the equivalent of Figs~\ref{fig:perf_1_means} and~\ref{fig:compare_perf_1_means} in, respectively, Figs~\ref{fig:perf_lr} and~\ref{fig:compare_perf_lr}. Results for $d=20$ and $d=100$ are very similar so the case $d=100$ is not shown. 

We observe similar results: \texttt{m-DPP} matches \texttt{D$^2$} in the Voronoi estimator, and outperforms \texttt{sensitivity iid} in all cases; \texttt{PolyProj-DPP} at least matches \texttt{sensitivity iid} in all cases. 

Finally, in Figure~\ref{fig:compare_perf_DPP_vs_iid} (right), we compare \texttt{m-DPP} and \texttt{PolyProj-DPP} versus \texttt{matched iid} and \texttt{sensitivity iid} for the linear regression problem: once again, the observed edge is clearly due to the negative correlations induced by the determinantal nature of our method.\\

\begin{figure}
	\centering
	\begin{tabular}{c|ccc|cc}
		& \multicolumn{3}{c|}{Voronoi weights} & \multicolumn{2}{c}{Importance sampling weights}\\\hline &&&&&\\
		\multirow{1}{*}[6em]{$d=2$}	&&\includegraphics[width=0.4\columnwidth]{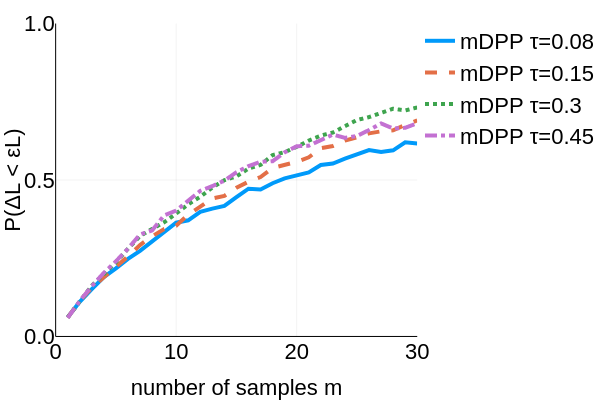} &&&
		\includegraphics[width=0.4\columnwidth]{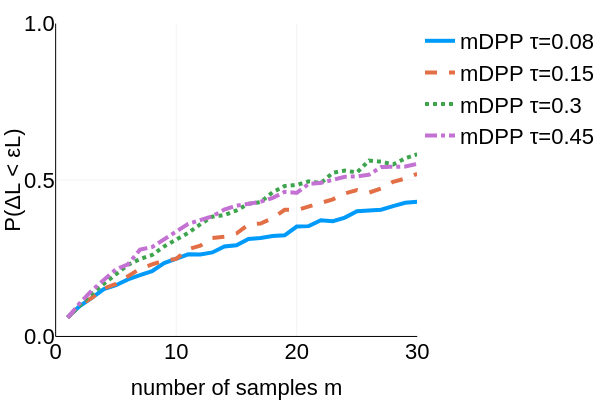}\\\hline &&&&&\\
		\multirow{1}{*}[6em]{$d=20$}&& \includegraphics[width=0.4\columnwidth]{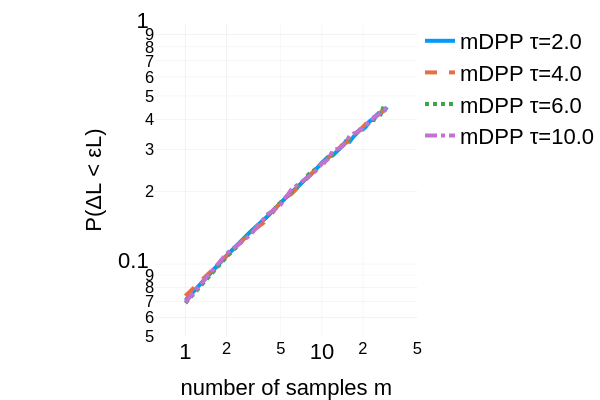}&&&
		\includegraphics[width=0.4\columnwidth]{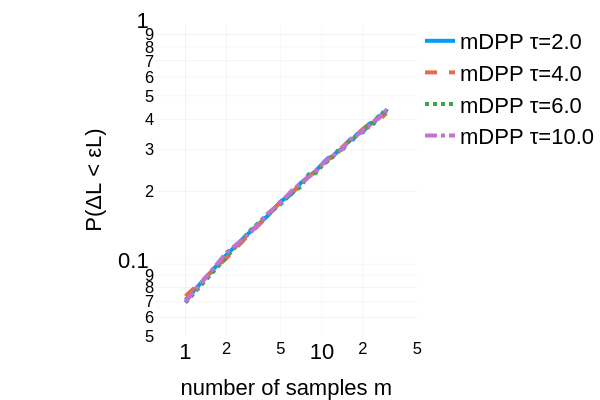}\\\hline &&&&&\\
		%\multirow{1}{*}[6em]{$d=100$}&& \includegraphics[width=0.4\columnwidth]{fig/lr_DPP_vs_S_d=100_V}&&&
		%\includegraphics[width=0.4\columnwidth]{fig/lr_DPP_vs_S_d=100}
	\end{tabular}
	\caption{Performance of \texttt{m-DPP} on the linear regression problem, versus the dimension $d$, the parameter $\tau$ of the Gaussian kernel and the choice of weights (Voronoi or importance sampling weights) in the cost estimator. The two bottom figures (for $d=20$) are in log-log scale. Performances are so similar that even in this scale they remain undecipherable. }
	\label{fig:perf_lr}
\end{figure}

\begin{figure}
	\centering
	\begin{tabular}{c|ccc|cc}
		& \multicolumn{3}{c|}{Voronoi weights} & \multicolumn{2}{c}{Importance sampling weights}\\\hline &&&&&\\
		\multirow{1}{*}[4em]{$d=2$}	&&\includegraphics[width=0.4\columnwidth]{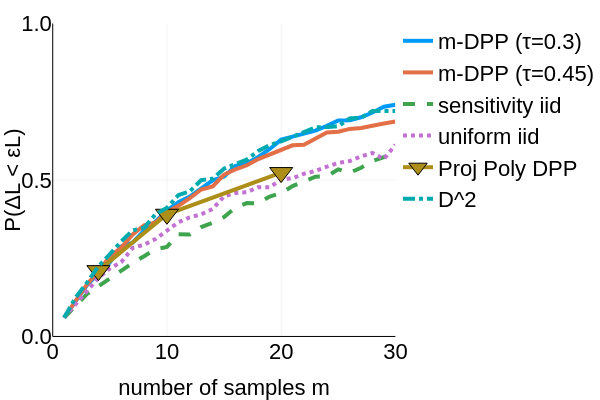} &&&
		\includegraphics[width=0.4\columnwidth]{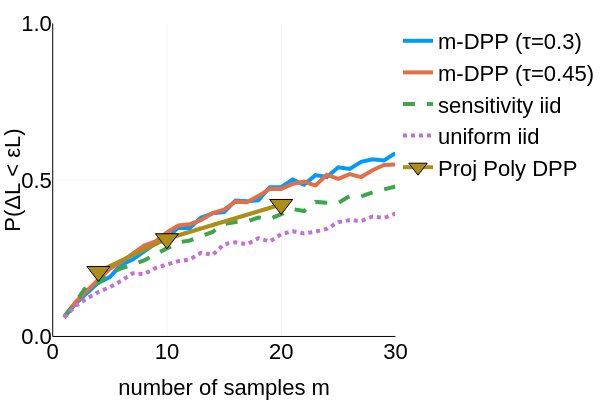}\\\hline &&&&&\\
		\multirow{1}{*}[4em]{$d=20$}&& \includegraphics[width=0.4\columnwidth]{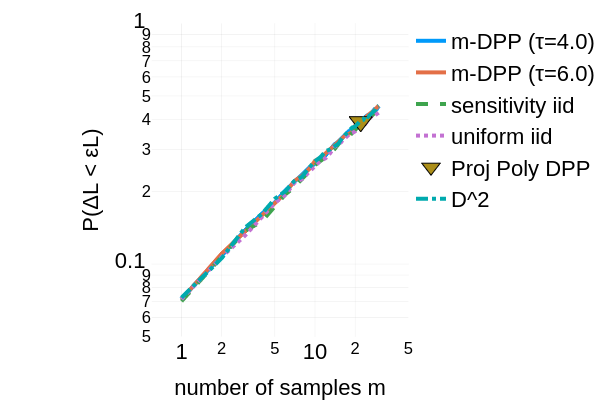}&&&
		\includegraphics[width=0.4\columnwidth]{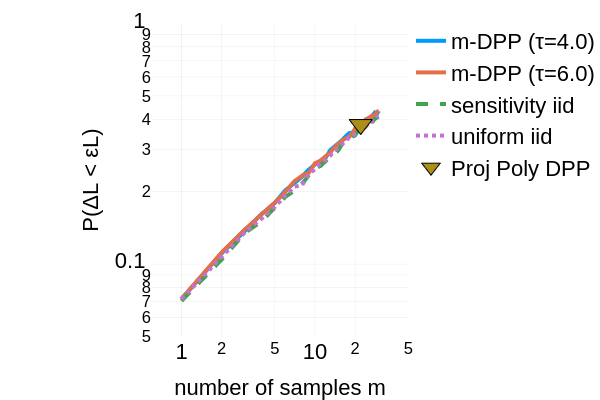}\\\hline &&&&&\\
		%\multirow{1}{*}[4em]{$d=100$}&& \includegraphics[width=0.4\columnwidth]{fig/lr_compare_all_methods_d=100_V}&&&
		%\includegraphics[width=0.4\columnwidth]{fig/lr_compare_all_methods_d=100}
	\end{tabular}
	\caption{Performance comparison of different sampling methods on the linear regression problem, versus the dimension $d$ and the choice of weights (Voronoi or importance sampling weights) in the cost estimator. The two bottom figures (for $d=20$) are in log-log scale. Performances are so similar that even in this scale they remain undecipherable.}
	\label{fig:compare_perf_lr}
\end{figure}

\noindent\textbf{We conclude} these first well-controlled experimental results by summarizing the observed behaviors:
\begin{itemize}
	\item \texttt{m-DPP} outperforms the current state of the art \texttt{sensitivity iid}, even in the $1$-means and the linear regression cases, where sensitivities do not need to be estimated but may be computed exactly. 
	\item \texttt{PolyProj-DPP} matches and in some cases outperforms \texttt{sensitivity iid}, at least for the importance sampling estimator.
	\item As the dimension increases, the edge over iid sampling decreases. In fact, all performances tend to \texttt{uniform iid}. 
	\item The best choice of parameter $\tau$ of the Gaussian kernel in \texttt{m-DPP} is still an open problem. Empirically, a good order of magnitude is the average interdistance of the datapoints. Ideally, nevertheless, $\tau$ should increase as $m$, the number of wanted samples, decreases. 
	\item Regarding the number of RFFs $r$, setting $r$ to $\mathcal{O}(m)$ is sufficient.
	\item Regarding the impact of outliers. Our theorems are not well suited to outliers (due to the proof techniques used); nevertheless, in practice, we see that outliers are not an issue in our method: they even have a smaller impact on our method's performances than on other methods. 
	\item Replacing weights by Voronoi weights yields in general better results, but only in low dimension. As the dimension increases, the Voronoi cost estimator fails (sometimes drastically). 
\end{itemize}

\subsubsection{Experiments on non-Gaussian data: the case of spectral features}
\label{subsec:expes_sp_feat_SBM}
\noindent\textbf{Spectral features.} Given a graph of $n$ nodes where $\ma{W}\in\mathbb{R}^{n\times n}$ is the adjacency matrix (\ie, $\ma{W}_{ij}=1$  if nodes $i$ and $j$ are connected, and $0$ ortherwise), a standard problem consists in partitioning the  nodes in $k$ communities, \ie, sets of nodes more connected to themselves than to other nodes of the graph~\citep{fortunato_community_2010}. A classical algorithm to solve efficiently this problem is the so-called spectral clustering algorithm~\citep{ng_spectral_2002}:
\begin{itemize}
	\item Define the normalized Laplacian matrix $\ma{L}=\ma{I} - \ma{D}^{-\frac{1}{2}}\ma{WD}^{-\frac{1}{2}}\in\mathbb{R}^{n \times n}$ where $\ma{I}$ is here the identity matrix in dimension $n$, and $\ma{D}\in\mathbb{R}^{n\times n}$ is a diagonal matrix with $\ma{D}_{ii}=d_i=\sum_j \ma{W}_{ij}$ the degree of node $i$.
	\item Compute via Arnoldi iterations or a similar algorithm the $k$ first eigenvectors of $\ma{L}$: $(\fou_1, \ldots, \fou_k)$.
	\item Associate to each node $i$ a (spectral) feature vector $\vec{x}_i\in\mathbb{R}^k$: $\forall l=1, \ldots, k\quad x_i(l) = u_l(i)$.  
	\item Normalize all feature vectors: $\vec{x}_i\leftarrow \vec{x}_i/\norm{\vec{x}_i}_2$. 
	\item Run $k$-means on all such normalized spectral features.
\end{itemize}
An extensive literature exists on spectral clustering and it has shown to be a very successful unsupervised classification algorithm in many situations~\citep{von_luxburg_tutorial_2007, tremblay_approximating_2020}. \\

\begin{figure}
	\centering
	\hspace{2cm}
	\includegraphics[width=0.3\columnwidth]{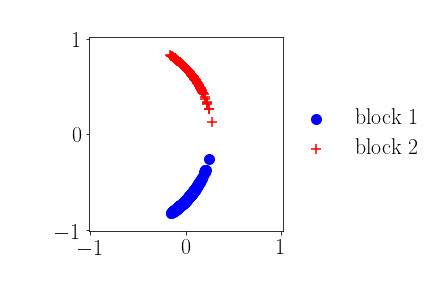}\hfill
	\includegraphics[width=0.3\columnwidth]{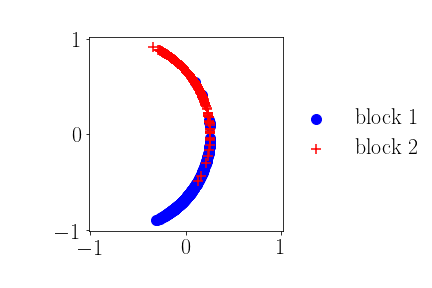}\hspace{2cm}
	\caption{Examples of SBM spectral features $\vec{x}_i$, here with $k=2$. Each colour corresponds to one block of the SBM. On the left, for an ``easy'' classification task ($\zeta=\zeta_c/4$), and on the right, for a harder setting ($\zeta=\zeta_c/2$). }
	\label{fig:2d_spectral_features}
\end{figure}

\begin{figure}
	\centering
	\begin{tabular}{c|ccc|cc}
		& \multicolumn{3}{c|}{Coreset property} & \multicolumn{2}{c}{$k$-means performance}\\\hline &&&&&\\
		\multirow{1}{*}[4em]{$\zeta=\frac{\zeta_c}{4}$}&& \includegraphics[width=0.4\columnwidth]{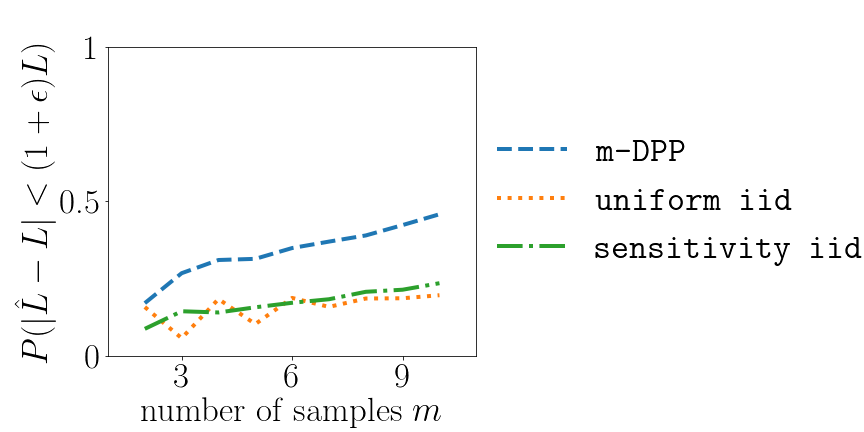}&&&
		\includegraphics[width=0.4\columnwidth]{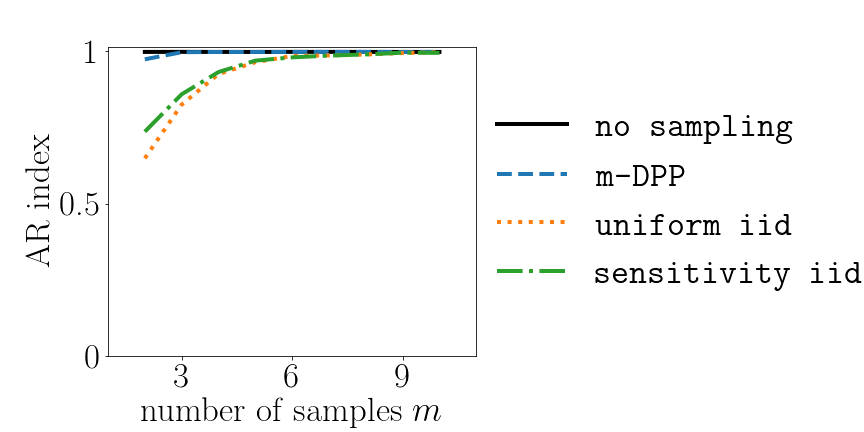}\\\hline &&&&&\\
		\multirow{1}{*}[4em]{$\zeta=\frac{\zeta_c}{2}$}&& \includegraphics[width=0.4\columnwidth]{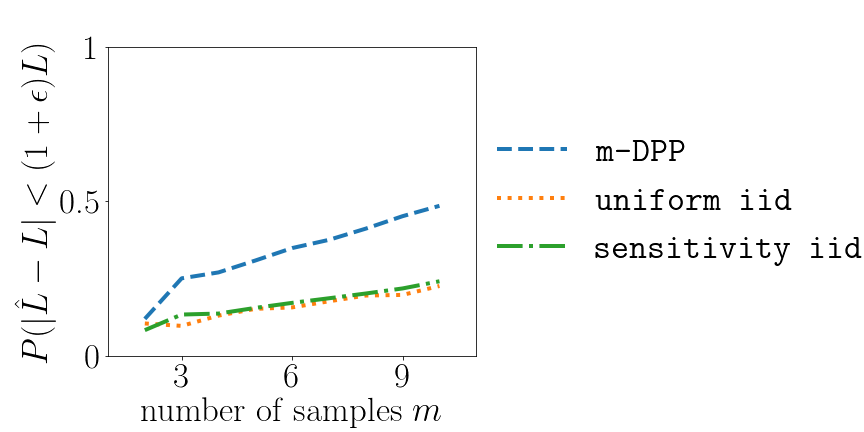}&&&
		\includegraphics[width=0.4\columnwidth]{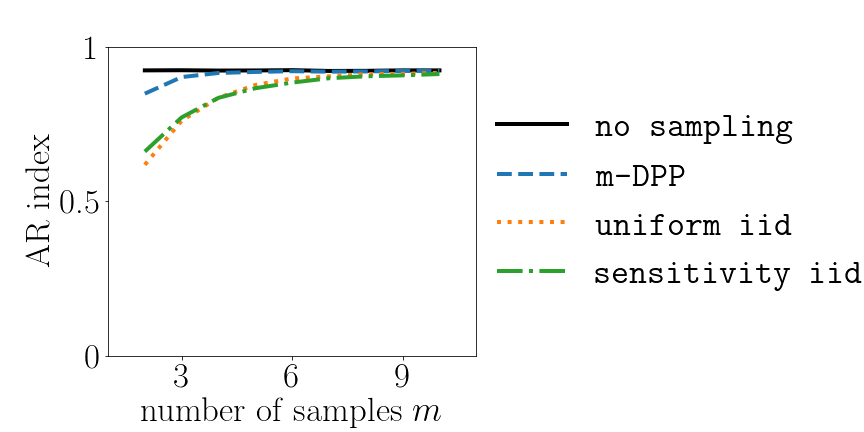}
	\end{tabular}
	\caption{Performance comparison of different methods on the $k$-means problem for spectral features of balanced SBM graphs (here, $k=2$). Left: testing the coreset property. Right: the Adjusted Rand index between the partition recovered by $k$-means on the weighted subsets and the ground truth partition of the SBM. $\zeta$ quantifies the difficulty of the classification task (see text): the lower it is, the easier the classification task. }
	\label{fig:compare_kmeans_perf_SBM_k=2}
\end{figure}

\begin{figure}
	\centering
	\begin{tabular}{c|ccc|cc}
		& \multicolumn{3}{c|}{Coreset property} & \multicolumn{2}{c}{$k$-means performance}\\\hline &&&&&\\
		\multirow{1}{*}[4em]{$\zeta=\frac{\zeta_c}{4}$}&& \includegraphics[width=0.4\columnwidth]{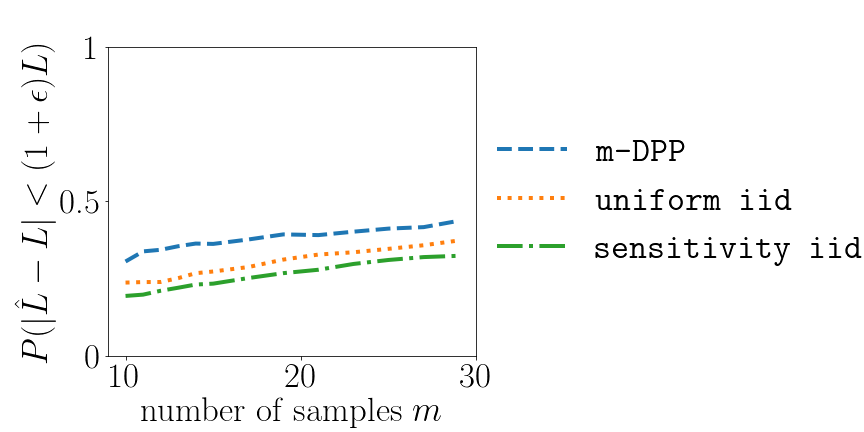}&&&
		\includegraphics[width=0.4\columnwidth]{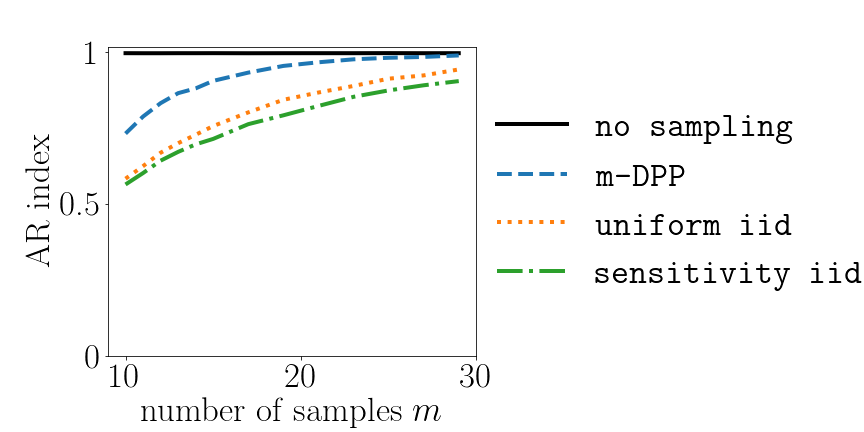}\\\hline &&&&&\\
		\multirow{1}{*}[4em]{$\zeta=\frac{\zeta_c}{2}$}&& \includegraphics[width=0.4\columnwidth]{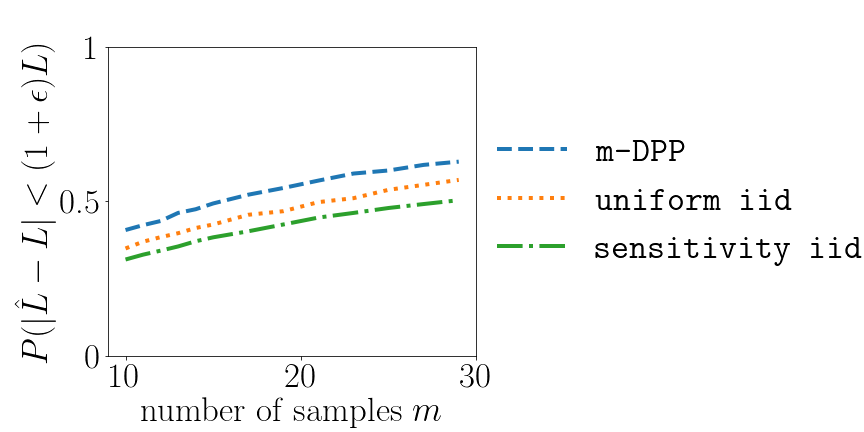}&&&
		\includegraphics[width=0.4\columnwidth]{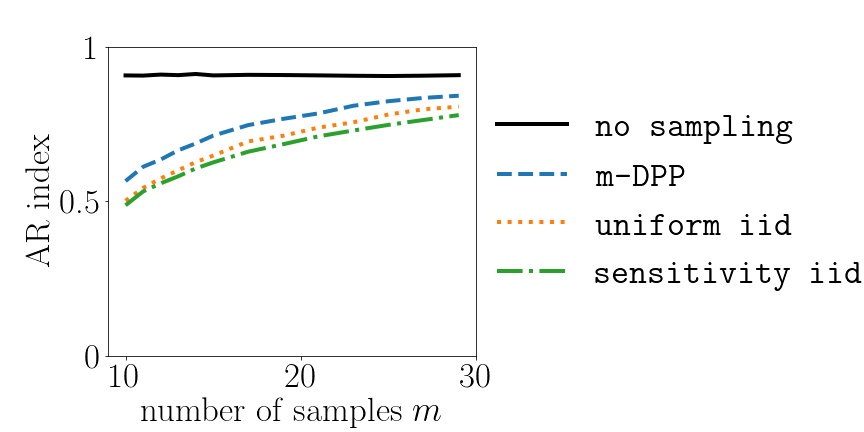}
	\end{tabular}
	\caption{Same as Figure~\ref{fig:compare_kmeans_perf_SBM_k=2} but with $k=10$.}
	\label{fig:compare_kmeans_perf_SBM_k=10}
\end{figure}

\begin{figure}
	\centering
	\begin{tabular}{c|ccc|cc}
		& \multicolumn{3}{c|}{Coreset property} & \multicolumn{2}{c}{$k$-means performance}\\\hline &&&&&\\
		\multirow{1}{*}[4em]{$600/400$}&& \includegraphics[width=0.4\columnwidth]{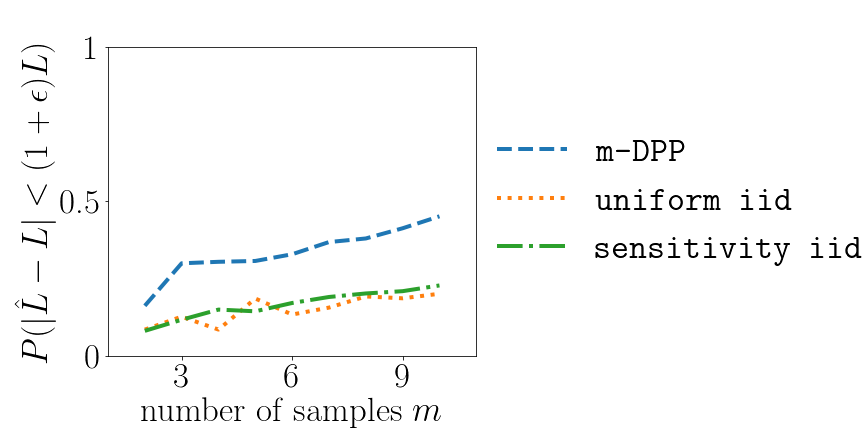}&&&
		\includegraphics[width=0.4\columnwidth]{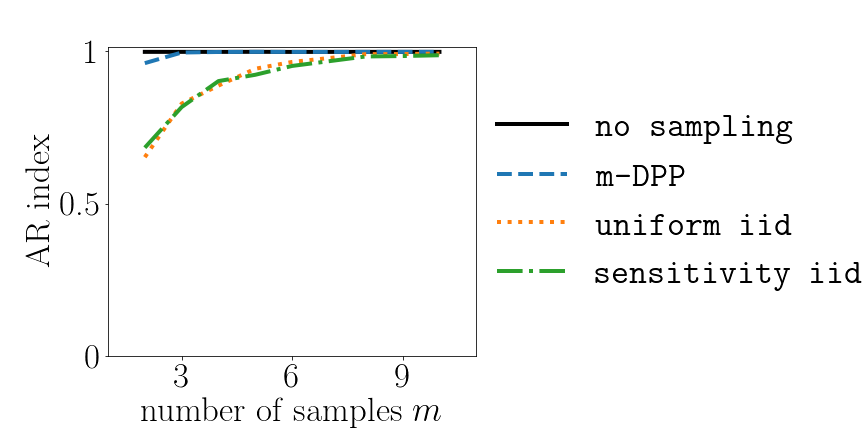}\\\hline &&&&&\\
		\multirow{1}{*}[4em]{$700/300$}&& \includegraphics[width=0.4\columnwidth]{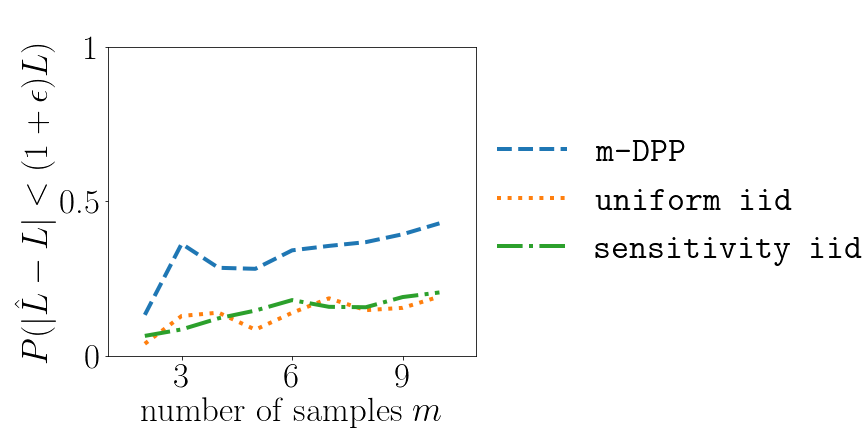}&&&
		\includegraphics[width=0.4\columnwidth]{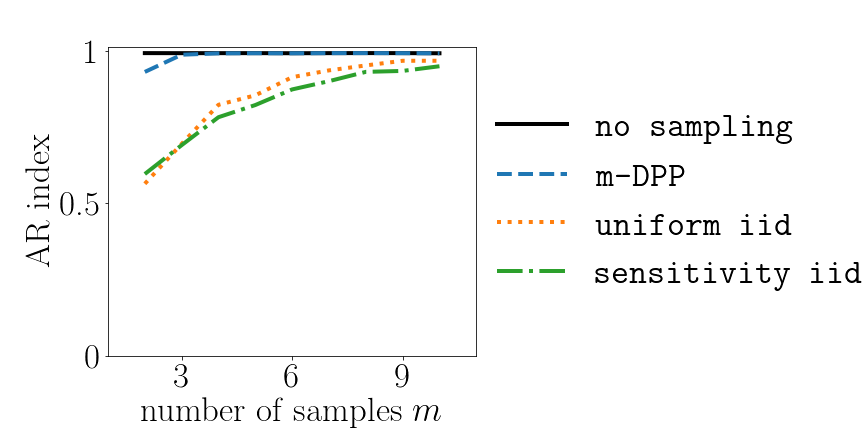}\\\hline &&&&&\\
		\multirow{1}{*}[4em]{$800/200$}&& \includegraphics[width=0.4\columnwidth]{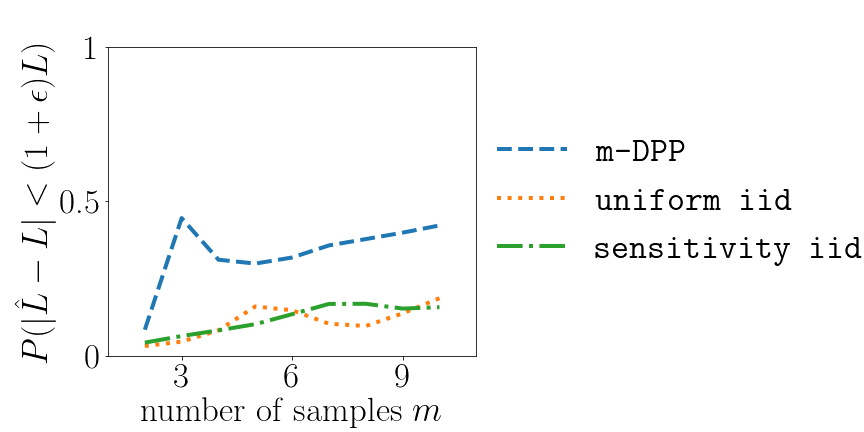}&&&
		\includegraphics[width=0.4\columnwidth]{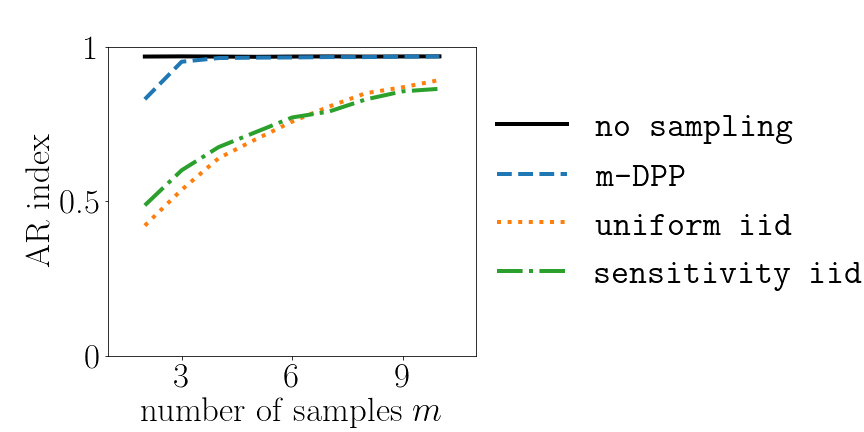}\\\hline &&&&&\\
		\multirow{1}{*}[4em]{$900/100$}&& \includegraphics[width=0.4\columnwidth]{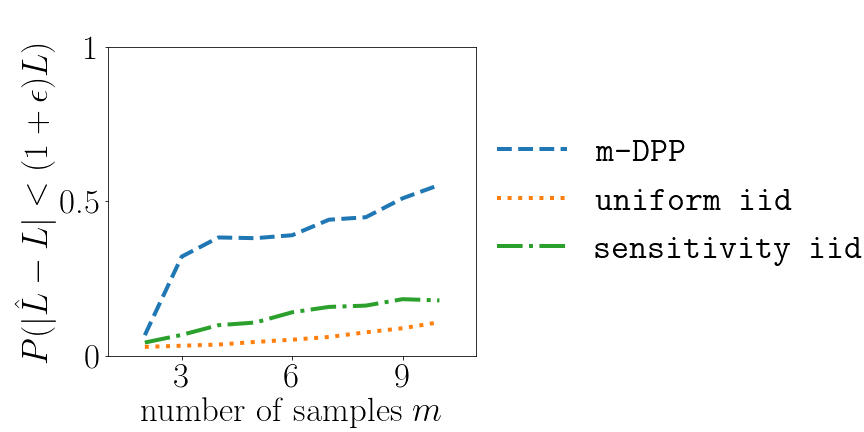}&&&
		\includegraphics[width=0.4\columnwidth]{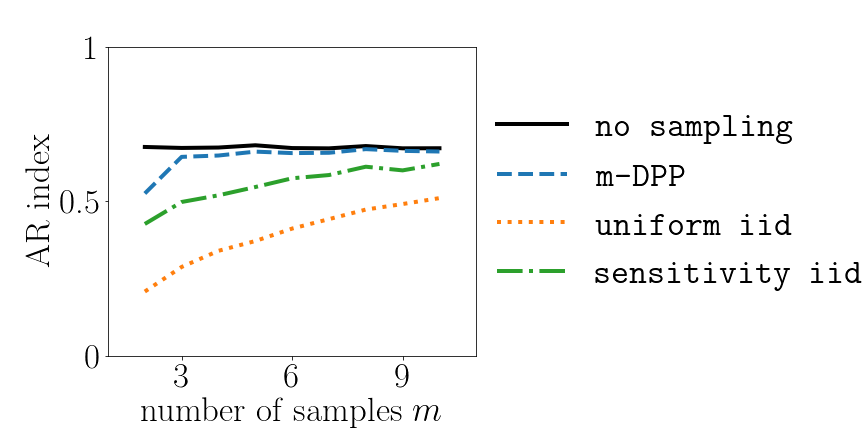}\\\hline &&&&&\\
		\multirow{1}{*}[4em]{$950/50$}&& \includegraphics[width=0.4\columnwidth]{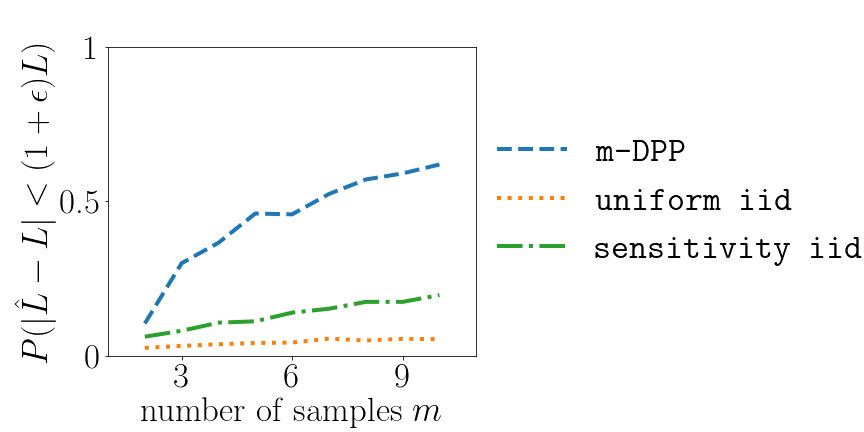}&&&
		\includegraphics[width=0.4\columnwidth]{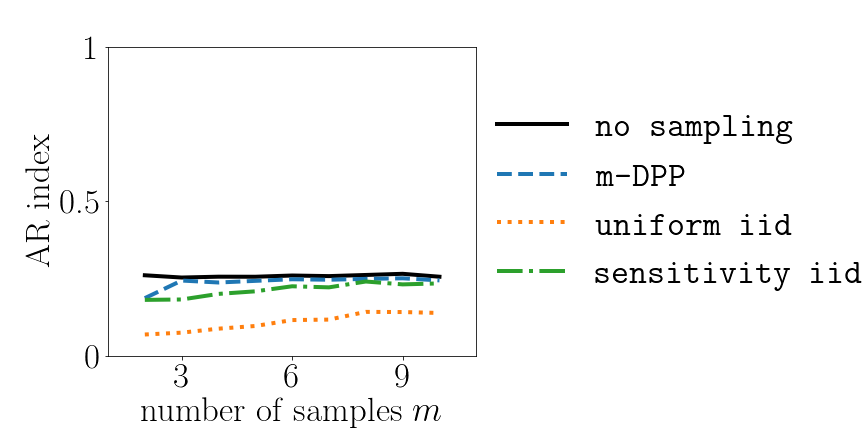}
	\end{tabular}
	\caption{Same as Figure~\ref{fig:compare_kmeans_perf_SBM_k=2} but with a fixed $\zeta=\zeta_c/4$ and a varying level of balance within the sizes of the $k=2$ communities. $n_1/n_2$ means one community with $n_1$ nodes and the other with $n_2$ nodes.}
	\label{fig:compare_kmeans_perf_SBM_unbalanced}
\end{figure}

\noindent\textbf{The Stochastic Block Model (SBM).} 
We consider random community-structured graphs drawn from the SBM, a classical class of structured random graphs~\citep[see for instance][]{abbe_community_2015}. We first look at graphs 
with $k$ communities of same size $n/k$. In the SBM, the probability of connection between any two nodes $i$ and $j$ is $q_1$ if they are in the same community, and $q_2$ otherwise. One can show that the average degree reads  $c=q_1\left(\frac{n}{k}-1\right)+q_2\left(n-\frac{n}{k}\right)$. 
Thus, instead of providing the probabilities $(q_1,q_2)$, one may characterize a SBM by considering $(\zeta=\frac{q_2}{q_1},c)$. 
The larger $\zeta$, the fuzzier the community structure, the harder the classification task. In fact,~\citet{decelle_asymptotic_2011} show that above the critical value $\zeta_c=(c-\sqrt{c})/(c+\sqrt{c}(k-1))$, community structure becomes  undetectable in the large $n$ limit. In the following, we set $n=1000$ and $c=16$; $k$ and $\zeta$ will vary. Note that spectral features $\vec{x}_i$ are not Gaussian and, in fact, do not fall into any classical data model (see Figure~\ref{fig:2d_spectral_features} to visualize instances of SBM spectral features with $k=2$). They are thus interesting candidates to test $k$-means algorithms. \\

\noindent\textbf{Results.} For different values of $\zeta$ and $k$, we generate $1000$ such SBM graphs from which we sample subsets according to different methods. We test both the coreset property (as before) and the $k$-means performance on the weighted subset compared to the $k$-means performed on all data. We plot in Figure~\ref{fig:compare_kmeans_perf_SBM_k=2} (resp. Figure~\ref{fig:compare_kmeans_perf_SBM_k=10}) the results obtained for $k=2$ (resp. $k=10$). Note that in this case, we have no explicit formula for the sensitivity such that for \texttt{sensitivity iid}, we use the bi-criteria approximation scheme provided in Algorithm2 of~\cite{bachem_practical_2017}. Here again, we see how our method outperforms iid sampling schemes, even in difficult classification contexts (for instance when $\zeta=\zeta_c/2$: even with all the data, $k$-means' performance saturates at an AR index of $0.9$). Moreover, as the dimension increases (here $d=k$), performances of all methods tend to uniformize. Surprisingly, \texttt{uniform iid} performs as well ($k=2$) and even outperforms ($k=10$) \texttt{sensitivity iid}. We believe this is due to approximation errors of the bi-criteria scheme used to find upper bounds of the sensitivity. Also, in this balanced case (communities have the same number of nodes), uniform sampling is in fact a good option. We will now see how this changes in the unbalanced case. \\

\noindent\textbf{The unbalanced case.} In the unbalanced case, $\zeta_c$ is no longer a recovery threshold, but we may still use $\zeta$ as a marker of difficulty of the recovery task. We set $\zeta$ to $\zeta_c/4$ and perform the same experiments as previously with $k=2$ blocks of unbalanced size. Results are shown in Figure~\ref{fig:compare_kmeans_perf_SBM_unbalanced}. For a fixed $\zeta$, the more unbalanced, the more difficult the recovery task. Also, the more unbalanced, the better is \texttt{sensitivity iid} compared to \texttt{uniform iid}. Nevertheless, \texttt{m-DPP} shows an edge over all iid methods in all tested configurations.

\subsubsection{Experiments on two real world data sets}
\label{subsec:expes_MNIST}

\noindent\textbf{The MNIST data set.} We perform a first experiment on the MNIST data set~\citep{lecun_mnist_1998} that consists in $7\cdot10^4$ images of handwritten digits (from $0$ to $9$) for which the ground truth is known. The classical associated machine learning goal is to classify them in $10$ classes (one for each digit). To do so, we pre-process the data in the following unsupervised way. We consider all images and extract SIFT descriptors~\citep{vedaldi_vlfeat:_2010} for each image. We then use FLANN~\citep{muja_scalable_2014} to compute a $\kappa$-nearest neighbor graph (with $\kappa$=10) based on these descriptors. We finally run the spectral clustering algorithm with $k=10$ to find the $10$ classes corresponding to each digit, as explained in Section~\ref{subsec:expes_sp_feat_SBM}. The $k$-means step is thus the last step of the overall processing. %We compare the coreset property of methods' subsamples in the top of Figure~\ref{fig:real_data_MNIST}. 
We compare results obtained with different sampling methods versus the results obtained without sampling in Figure~\ref{fig:real_data_MNIST} (bottom). For \texttt{m-DPP}, several values of $\tau$ were tried, and we show here the result obtained for $\tau=1.5$. Also, a number $r=200$ of Fourier features were used. We see that, in the Voronoi weight setting, \texttt{m-DPP} is competitive with \texttt{D$^2$}. Moreover, \texttt{uniform iid} outperforms \texttt{sensitivity iid} certainly due to approximation errors of the bi-criteria procedure and to the fact that the data is balanced (there are more or less $7\cdot10^3$ instances of each digit in the data set), thus favoring uniform sampling. Finally, \texttt{m-DPP} outperforms once again the iid random sampling techniques. Note that the methods' classification performance is remarkable. Without sampling, the overall classification performance in terms of AR index with the ground truth is $0.95$. With only $\sim20$ samples, \texttt{m-DPP} reaches a performance of $\sim0.9$!  \\

\begin{figure}
	\centering
	\includegraphics[width=0.4\columnwidth]{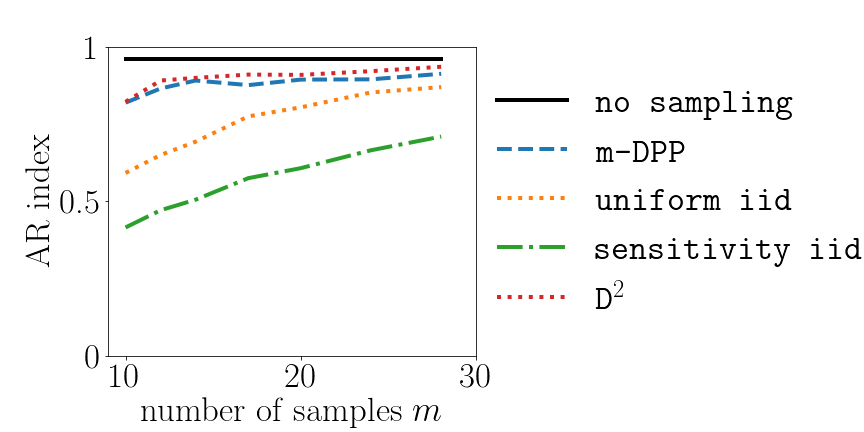}\hfill
	\includegraphics[width=0.4\columnwidth]{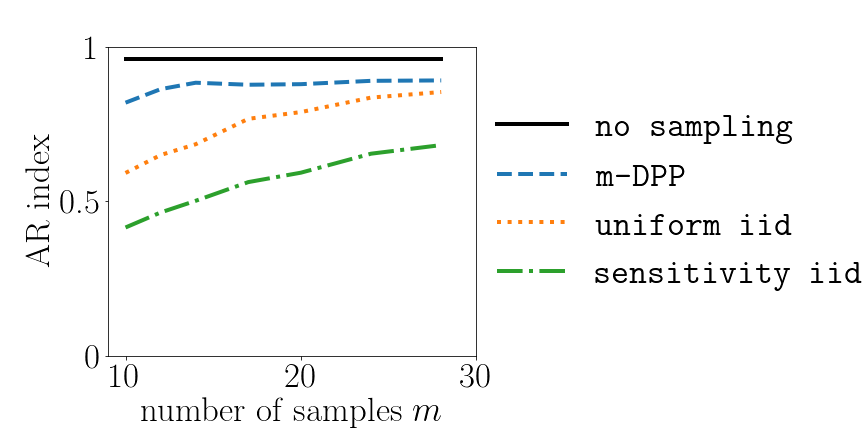}\hspace{2cm}
	\caption{Classification performance on the MNIST data set obtained with different sampling methods versus the result obtained without sampling,  using Voronoi weights (left),  or importance sampling weights (right). }
	\label{fig:real_data_MNIST}
\end{figure}

\begin{figure}
	\centering
	\hspace{2cm}
	\includegraphics[width=0.4\columnwidth]{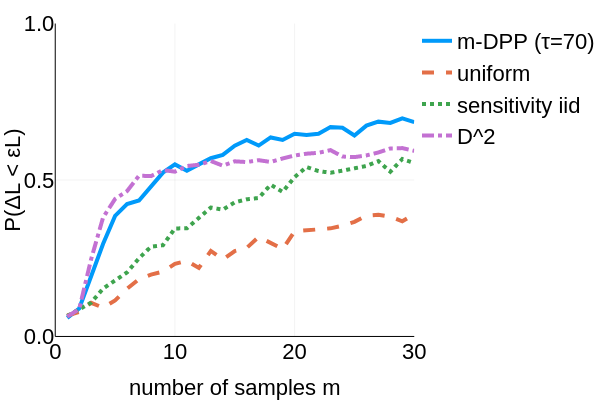}\hfill
	\includegraphics[width=0.4\columnwidth]{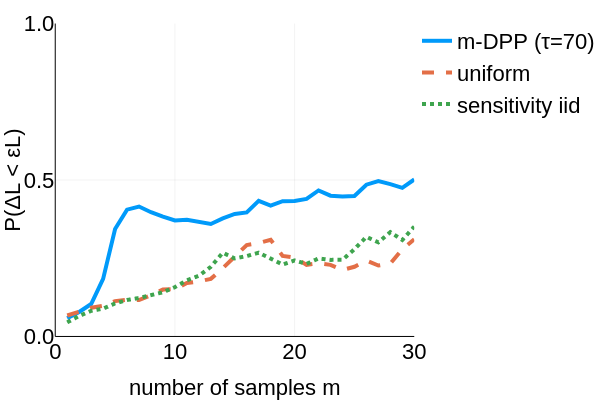}\hspace{2cm}
	\caption{Performance of different sampling methods on the US Census data set. Left: performance using Voronoi weights. Right: performance using importance sampling weights. }
	\label{fig:real_data_USCENSUS}
\end{figure}

\noindent\textbf{The US Census data set.} We also perform experiments on the $1990$ US Census data set\footnote{downloaded from \url{https://archive.ics.uci.edu/ml/datasets/US+Census+Data+(1990)}.}, that consists in $n=2458285$ surveyed persons, and $d=68$ categorical attributes such as age, income, etc. The data was pre-processed by a series of operation detailed on its download webpage. As there is no ground truth in this data set to compare to, we arbitrarily decide $k=15$ classes, and show solely the coreset property of the samples obtained via different methods. For \texttt{m-DPP}, $\tau$ was set to $70$ (the mean interdistance estimated on $1000$ randomly chosen pairs of datapoints), and a number $r=30$ of Fourier features was chosen. Experiments were done with $\tau$ ranging from $\tau=30$ to $\tau=140$ with no qualitative change in performance (not shown). Figure~\ref{fig:real_data_USCENSUS} shows the results of the experiments. We see that \texttt{m-DPP} outperforms all other methods, in both Voronoi and importance sampling settings. In this example, note that \texttt{sensitivity iid} outperforms \texttt{uniform  iid} probably due to the fact that the $15$ potential classes  are unbalanced.

\subsection{Computation time}
Comparing computation times is always a tricky endeavour: they heavily depend on the quality of implementation, choice of language, choice of experiments, choice of parameters, hardware, etc. Giving the full picture is out-of-scope of this paper. We hope here to give some insight in the computational complexity of the proposed methods. We suggest to first recap theoretical times before looking at times observed in a few experiments. \\

\begin{figure}
	\centering
	\includegraphics[width=0.3\columnwidth]{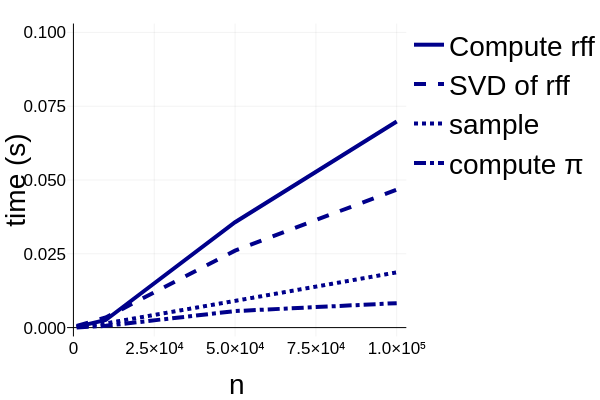}\hfill
	\includegraphics[width=0.3\columnwidth]{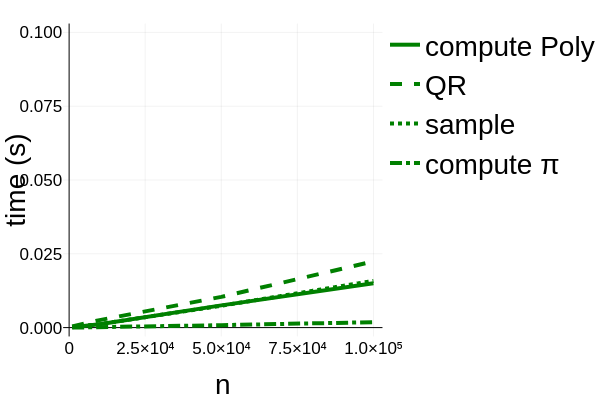}\hfill
	\includegraphics[width=0.3\columnwidth]{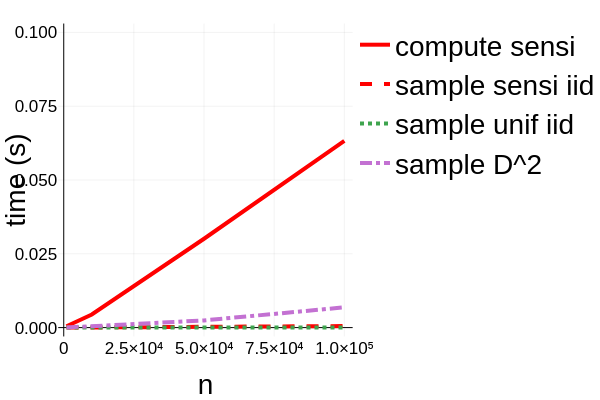}\\
	\includegraphics[width=0.3\columnwidth]{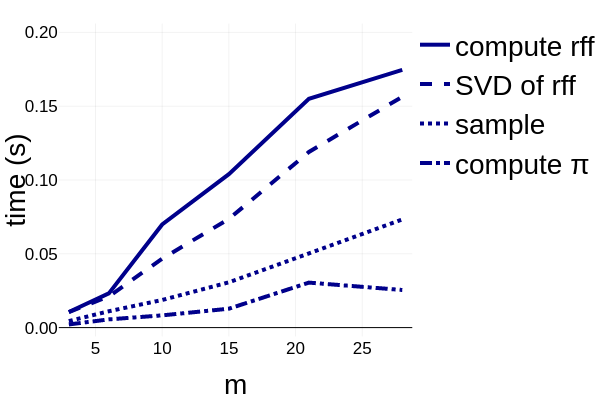}\hfill
	\includegraphics[width=0.3\columnwidth]{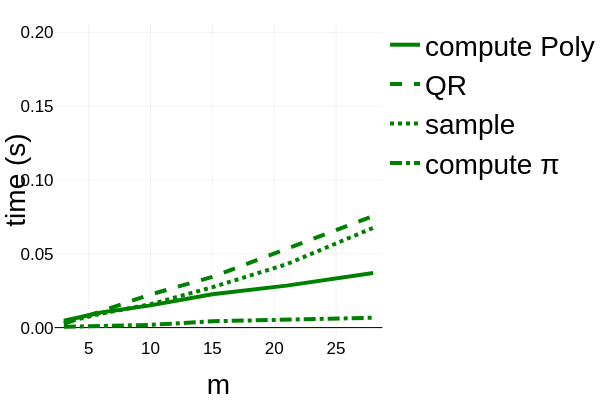}\hfill
	\includegraphics[width=0.3\columnwidth]{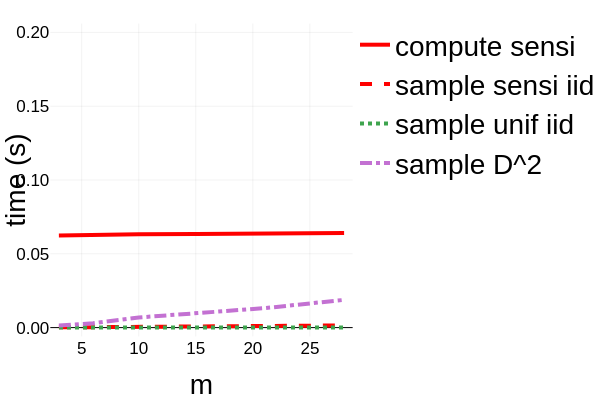}\\
	\includegraphics[width=0.3\columnwidth]{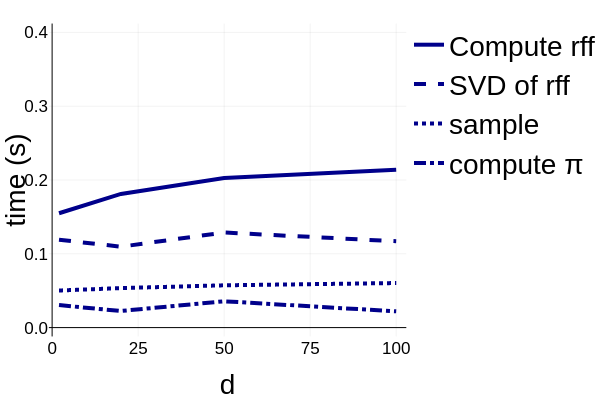}\hfill
	\includegraphics[width=0.3\columnwidth]{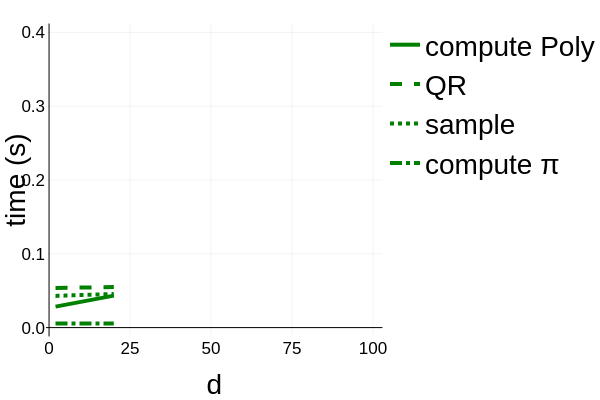}\hfill
	\includegraphics[width=0.3\columnwidth]{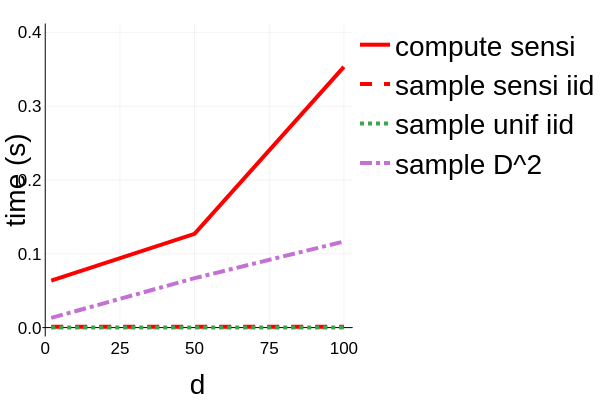}
	\caption{Computation times for all methods. Top line: $d=2$, $m=10$, versus $n$. Middle line: $n=10^5$, $d=2$, versus $m$. Bottom line: $n=10^5$, $m=21$, versus $d$. Left column: the four fundamental operations associated to \texttt{m-DPP}: i/~compute the RFFs with $r$ set to $m$ here, ii/~compute the SVD of the RFFs, iii/~sample from the $m$-DPP, iv/~compute the marginal probabilities $\pi$ useful for the importance sampling estimator. Middle column: the four equivalent operations for \texttt{PolyProj-DPP}: i/~compute the Vandermonde matrix ("compute Poly"), ii/~compute the $QR$ decomposition of that matrix, iii/~sample from the projective DPP, iv/~compute the marginal probabilities $\pi$. Right column: the two fundamental operations for \texttt{sensitivity iid}: i/~the computation of the bi-criteria approximation ("compute sensi"), ii/~the weighted iid sampling with replacement itself, along with \texttt{uniform iid} and \texttt{D}$^2$.}
	\label{fig:comp_time}
\end{figure}

\noindent\textbf{In theory.} The theoretical time for \texttt{m-DPP} is $\mathcal{O}(nrd)$ for the RFFs, $\mathcal{O}(nr^2)$ for the SVD of the RFFs, $\mathcal{O}(nm^2)$ for the sampling, and $\mathcal{O}(nm+r)$ for the computation of the inclusion probabilities $\pi$. As $r$ is set to $\mathcal{O}(m)$, this sums up to $\mathcal{O}(nm^2+nmd)$. The theoretical time for \texttt{PolyProj-DPP} is $\mathcal{O}(nm)$ to compute the Vandermonde matrix, $\mathcal{O}(nm^2)$ for the $QR$ decomposition, $\mathcal{O}(nm^2)$ for the sampling, and $\mathcal{O}(nm)$ to compute the inclusion probabilities $\pi$ (which are simply the sum of squares of each line of $\ma{Q}$); which sums up to $\mathcal{O}(nm^2)$. The theoretical time for Algorithm 2 of~\cite{bachem_practical_2017} used here for the bi-criteria approximation of the sensitivities in the $k$-means context, is $\mathcal{O}(Indk)$ where $I$ is the number of times Algorithm 1 of~\cite{bachem_practical_2017} is run to find the best initialization (we set it to $I=10$ in our experiments). Once upper bounds of the sensitivities have been computed, the iid sampling time is negligeable. The theoretical time of \texttt{D$^2$} is $\mathcal{O}(ndm)$. \\

\noindent\textbf{In practice.} Experiments were made on a laptop with 8 cores and 16 GB of memory, with the Julia toolbox available on the authors's website.\footnote{The DPP4Coresets Julia toolbox is also available at \url{https://gricad-gitlab.univ-grenoble-alpes.fr/tremblan/dpp4coresets.jl}\,.}  Figure~\ref{fig:comp_time} shows some computation times versus $m$, $d$ and $n$ for the $1$-means problem. We observe a linear progression in $n$ of the determinantal methods, as well as a superlinear progression in $m$ for some parts of the computations. Comparing with \texttt{sensitivity iid}, one observes: the lower $m$ and the larger $d$, the more our methods are comparable in terms of computation time.

Note that this comparison is in fact in favor of \texttt{sensitivity iid}: the computation time of the bi-criteria approximation increases in fact with the number of expected classes $k$. The figures represented here are for $k=1$.  
We reproduce in the following table the computation times for the US-Census data set, with $k=15$, $m=30$, $r=30$:\\

\hspace{-1.5cm}\begin{tabular}{|c|c|c|c|c|c|c|}
	\texttt{m-DPP}: rff&\texttt{m-DPP}: svd& \texttt{m-DPP}: sample& \texttt{m-DPP}: $\pi$& \texttt{sensi iid}: bi-criteria& \texttt{sensi iid}: sample& \texttt{D$^2$}\\\hline
	$2.6$s& $2.5$s& $2.7$s& $1.3$s& $18.7$s& $0.04$s& $3.1$s
\end{tabular}\\

\noindent The total time for $m$-DPP sampling is thus half the time necessary for the bi-criteria approximation; for a substantial gain in performance, as seen in Figure~\ref{fig:real_data_USCENSUS}.\\

\noindent\textbf{To conclude,} the determinantal methods proposed in this paper are in many situations, especially as $m$ increases, and $d$ decreases, slightly heavier in computation time than iid sampling. Due to the observed gain in coreset performance, however, we believe that the additional cost is worth the effort.%, and may find interesting applications in data summarization.}

\section{Conclusion}
\label{sec:conclusion}
In this work, we introduced a new random sampling method based on DPPs to build coresets. Different from sensitivity-based iid random sampling, our method introduces negative correlations between samples due to its determinantal nature. Also, different from $D^2$ sampling, also known to be repulsive, the proposed method is tractable in the sense that marginal probabilities are known and importance sampling schemes can be used. Our theoretical results may be summarized in three points. Firstly, Thms~\ref{thm:main} and~\ref{thm:main_DPP} provide coreset guarantees in function of the point process' probabilities of inclusion. These guarantees are not stronger than the iid case and are in fact similar: they both show that the ideal marginal probabilities are proportional to the sensitivity. Nevertheless, these results do not take into account higher order inclusion probabilities coding for the repulsion within the sampled subsets and are in fact verified  for any choice of such high-order marginals (provided $\ma{K}$ stays SDP with eigenvalues between $0$ and $1$). This leads to the second point: given that these higher-order inclusion probabilities offer extra degrees of freedom and due to simple variance arguments (theorems of Section~\ref{subsec:variance}), we show that DPP-based random sampling necessarily yields better performance than its independent counterparts. On the theoretical side, additional work is required to specify precisely the minimum number of required samples guaranteeing the coreset property. We expect that further research on concentration properties of strongly Rayleigh measures, involving not only first order marginals, but higher-order ones, should enable to move forward in this direction. The third and final point is the rebalancing property of polynomial DPPs: without any prior density-like estimation, polynomial DPPs such as the ones described in Section~\ref{sec:polynomial-dpps} provide samples that are asymptotically independent of the underlying distribution of the data. Even though this result is only asymptotic, it is yet another argument in favor of DPP sampling for coresets.

From an application point-of-view, the coreset theorems were applied to the ubiquitous $k$-means and linear regression problems. Given a data set, the ideal $L$-ensemble $\ma{L}$ adapted to these problems is untractable and we thus propose two  heuristics, one via random Fourier features of the Gaussian kernel, and one based on the Vandermonde matrix, in order to efficiently sample a DPP that has the desirable properties to sample coresets (if not provably, at least quantitatively). To sample a subset of size $m$, our heuristics run resp. in $\mathcal{O}(nm^2+nmd)$ and $\mathcal{O}(nm^2)$. This is more expensive than the sensitivity-based iid strategy, especially as the number of samples $m$ increases; but empirically provides better coresets on different artificial and real-world data sets. 

Finally, this work calls for several extensions. First of all, two likely difficult theoretical questions: how to improve the concentration inequalities for DPPs? (such improvements would directly benefit the coreset theorem's bounds). How to find the optimal DPP kernel given a data set? (which asks in fact difficult questions in frame theory). Also, these DPP sampling schemes should be extended to the streaming and/or distributed settings.

% Acknowledgements should go at the end, before appendices and references
\acks{This work was partly funded by the ANR GenGP (ANR-16-CE23-0008), the LabEx PERSYVAL-Lab (ANR-11-LABX-0025-01), the CNRS PEPS I3A (Project RW4SPEC), the Grenoble Data Institute (ANR-15-IDEX-02) and the LIA CNRS/Melbourne Univ Geodesic.}

% Manual newpage inserted to improve layout of sample file - not
% needed in general before appendices/bibliography.

\newpage

\appendix

\section{Proof of Theorem~\ref{thm:main}}
\label{app:proof_thm_main}

\begin{proof}
	The theorem consists in proving that Eq.~\eqref{eq:coresets} is true under the assumptions of Theorem~\ref{thm:main}. We follow a classical proof scheme from compressed sensing~\citep{baraniuk_simple_2008}, in four steps:
	\begin{enumerate}
		\item we first use concentration arguments for a given $\theta\in\Theta$. 
		\item we then build an $\epsilon$-net paving the space of parameters $\Theta$. 
		\item via the union bound, we obtain the result for all $\theta$ in the $\epsilon$-net.
		\item via the Lipschitz property of $f$, we obtain the desired result for all $\theta\in\Theta$.
	\end{enumerate}
	
	\textbf{Step 1} (Concentration around $\theta\in\Theta$) DPPs are instances of sampling schemes that are strongly Rayleigh. Since strongly Rayleigh distributions are closed under truncation, any $m$-DPP is also strongly Rayleigh. We can thus apply the concentration results of~\citet{pemantle_concentration_2014}. For a given $\theta\in\Theta$, we have : $\forall\epsilon\in(0,1),  \forall\delta\in(0,1)$:
	\begin{align*}
	\mathbb{P}\left(\left|\frac{\hat{L}}{L}-1\right|\geq\epsilon\right) = \mathbb{P}\left(\left|\hat{L}-L\right|\geq\epsilon L\right)\leq\delta,
	\end{align*}
	provided that:
	\begin{align}
	\label{eq:to_replace_if_mDPP_1}
	m\geq \frac{8}{\epsilon^2}  C^2 \log{\frac{2}{\delta}},
	\end{align}
	with $C=\displaystyle \max_{i} \frac{f(x_i, \theta)}{L\bar{\pi}_i}$, where  $\bar{\pi}_i$ is a shorthand for $\pi_i/m$, and $\pi_i$ is the marginal probability of sampling element $i$.
	
	Using the same concentration results, we also have:
	\begin{align}
	\label{eq:conc_size}
	\forall(\epsilon,\delta)\in(0,1)^2,~~\mathbb{P}\left(\left|\frac{\sum_{i}\frac{\epsilon_i}{\pi_i}}{n}-1\right|\geq\epsilon\right) \leq\delta,
	\end{align}
	provided that:
	\begin{align}
	\label{eq:to_replace_if_mDPP_2}
	m\geq \frac{8}{\epsilon^2} \frac{1}{n^2\bar{\pi}_{\text{min}}^2} \log{\frac{2}{\delta}},
	\end{align}
	where $\bar{\pi}_{\text{min}} = \min_i \bar{\pi}_i$. 
	
	\textbf{Step 2} ($\epsilon'$-net of $\Theta$) Consider $\Gamma_{\epsilon'} = (\theta^*_1,\ldots,\theta^*_\eta)$ the smallest subset of $\Theta$ such that balls of radius $\epsilon'$ centered around the elements in $\Gamma_{\epsilon'}$ cover $\Theta$. $\Gamma_{\epsilon'}$ is called an $\epsilon'$-net of $\Theta$ and $\eta=|\Gamma_{\epsilon'}|$ its covering number. The covering property entails that:
	\begin{align*}
	\forall\theta\in\Theta\quad\exists\theta^*\in\Gamma_{\epsilon'}~~~\text{ s.t. }~~~d_\Theta(\theta,\theta^*)\leq\epsilon'.
	\end{align*}
	
	\textbf{Step 3.} (Union bound) 
	Write $\delta' = \delta / 2\eta$. 
	From step 1, we know that, $\forall\theta^*\in\Gamma_{\epsilon'}$:
	\begin{align*}
	\mathbb{P}\left(\left|\frac{\hat{L}}{L}-1\right|\geq\epsilon\right) \leq\delta'
	\end{align*}
	provided that:
	\begin{align*}
	m\geq \frac{8}{\epsilon^2}  C^2 \log{\frac{2}{\delta'}}.
	\end{align*}
	From the union bound, we have:
	\begin{align*}
	\mathbb{P}\left(\forall\theta^*\in\Gamma_{\epsilon'},\quad \left|\frac{\hat{L}}{L}-1\right|\leq \epsilon\right)\geq 1-\sum_{\theta^*\in\Gamma}\delta' = 1-\frac{\delta}{2},
	\end{align*}
	provided that:
	\begin{align}
	\label{eq:pre_condition}
	m\geq \frac{8}{\epsilon^2} \left(\max_{\theta^*\in\Gamma_{\epsilon'}} C\right)^2 \log{\frac{4\eta}{\delta}}
	\end{align}
	Given that $\bar{\pi}_i$ will \textit{in fine} be independent of $\theta$ (as we want the coreset property to be true for all $\theta\in\Theta$),
	\begin{align}
	\max_{\theta^*\in\Gamma_{\epsilon'}} C &= \max_{\theta^*\in\Gamma_{\epsilon'}}\max_{i} \frac{f(x_i, \theta)}{L\bar{\pi}_i}\\
	&=\max_{i} \frac{1}{\bar{\pi}_i}\max_{\theta^*\in\Gamma_{\epsilon'}}\frac{f(x_i, \theta)}{L}\\
	&\leq \max_{i} \frac{1}{\bar{\pi}_i} \max_{\theta\in\Theta}\frac{f(x_i, \theta)}{L} = \max_{i} \frac{\sigma_i}{\bar{\pi}_i}\label{eq:intro_sigma},
	\end{align}
	where we see how the sensitivity $\sigma_i$ naturally arises in the proof. Eq.~\eqref{eq:intro_sigma} entails that 
	Eq.~\eqref{eq:pre_condition} is verified if $m\geq m_1$ with
	\begin{align*}
	m_1^ = \frac{8}{\epsilon^2} \left(\max_{i} \frac{\sigma_i}{\bar{\pi}_i}\right)^2 \log{\frac{4\eta}{\delta}}.
	\end{align*}

	Write $\delta'' = \delta / 2$. From Eq.~\eqref{eq:conc_size}, we have:
	\begin{align*}
	\mathbb{P}\left(\left|\frac{\sum_{i}\frac{\epsilon_i}{\pi_i}}{n}-1\right|\geq\epsilon\right) \leq\delta'',
	\end{align*}
	provided that $m\geq m_2$ with
	\begin{align*}
	m_2=\frac{8}{\epsilon^2n^2\bar{\pi}_{\text{min}}^2} \log{\frac{4}{\delta}}.
	\end{align*}
	We have (with the union bound again):
	\begin{align*}
	\mathbb{P}&\left( \left|\frac{\sum_{i}\frac{\epsilon_i}{\pi_i}}{n}-1\right|\leq\epsilon \quad \text{AND} \quad \forall\theta^*\in\Gamma_{\epsilon'},\quad \left|\frac{\hat{L}}{L}-1\right|\leq \epsilon\right)\\
	&\qquad\qquad\qquad\geq 1-\delta/2-\delta'' = 1-\delta,
	\end{align*}
	provided that:
	\begin{align*}
	m\geq \max (m_1, m_2).
	\end{align*}
	
	\textbf{Step 4} (Continuity argument) 
	Suppose that $m\geq \max (m_1^*, m_2^*)$ with $m^*_1$, $m^*_2$ as defined in the theorem. The result of step 3 with $\epsilon \leftarrow \epsilon/2$ states that, with probability at least $1-\delta$, one has:
	\begin{align}
	\label{eq:concentration_eps_net}
	\left|\frac{\sum_{i}\frac{\epsilon_i}{\pi_i}}{n}-1\right|\leq\frac{\epsilon}{2}\quad\text{AND}\quad\forall\theta^*\in\Gamma_{\epsilon'}, ~~\left|\frac{\hat{L}}{L}-1\right|\leq \frac{\epsilon}{2}.
	\end{align}
	
	We now look for the maximum value of $\epsilon'$ such that Eq.~\eqref{eq:concentration_eps_net} implies the following desired result:
	\begin{align}
	\label{eq:desired_result}
	\forall\theta\in\Theta, \qquad\left|\frac{\hat{L}}{L}-1\right|\leq \epsilon.
	\end{align}
	
	Consider $\theta\in\Theta$. By the covering property of $\Gamma_{\epsilon'}$, we have:
	\begin{align*}
	\exists\theta^*\in\Gamma_{\epsilon'} ~\text{ s.t. } ~d_\Theta(\theta,\theta^*)\leq\epsilon'.
	\end{align*}
	Moreover, as $f$ is $\gamma$-Lipschitz, $\forall x_i\in\mathcal{X}$:
	\begin{align}
	\label{eq:Lipschitz}
	|f(x_i,\theta)-f(x_i,\theta^*)|\leq \gamma ~d_\Theta(\theta,\theta^*)\leq \gamma \epsilon'.
	\end{align}
	Thus, using Eqs.~\eqref{eq:Lipschitz} and then~\eqref{eq:concentration_eps_net}:
	\begin{align*}
	\hat{L}(\mathcal{X},\theta)&\leq \hat{L}(\mathcal{X},\theta^*) + \gamma\epsilon'\sum_i \frac{\epsilon_i}{\pi_i}\\
	&\leq (1+\frac{\epsilon}{2})(L(\mathcal{X},\theta^*) + n\gamma\epsilon').
	\end{align*}
	Also, using Eq.~\eqref{eq:Lipschitz} again:
	\begin{align*}
	L(\mathcal{X},\theta^*)\leq L(\mathcal{X},\theta) + n\gamma\epsilon'.
	\end{align*}
	Thus:
	\begin{align}
	\label{eq:higher_bound}
	\hat{L}(\mathcal{X},\theta)&\leq (1+\frac{\epsilon}{2})L(\mathcal{X},\theta) + 2n\gamma\epsilon'(1+\frac{\epsilon}{2}).
	\end{align}
	Similarly, for the lower bound, one obtains:
	\begin{align}
	\label{eq:lower_bound}
	(1-\frac{\epsilon}{2})(L(\mathcal{X},\theta) - 2n\gamma\epsilon') \leq \hat{L}(\mathcal{X},\theta)
	\end{align}
	In order for Eqs~\eqref{eq:higher_bound} and~\eqref{eq:lower_bound} to imply Eq.\eqref{eq:desired_result}, we need:
	\begin{align*}
	2n\gamma\epsilon'(1+\frac{\epsilon}{2}) \leq \frac{\epsilon}{2}L(\mathcal{X},\theta),
	\end{align*}
	\ie:
	\begin{align*}
	\epsilon'\leq \frac{\epsilon L(\mathcal{X},\theta)}{4n\gamma(1+\frac{\epsilon}{2})}\leq \frac{\epsilon L(\mathcal{X},\theta)}{6n\gamma}.
	\end{align*}
	In order for this condition to be true for all $\theta$, we choose:
	\begin{align}
	\label{eq:epsilonprime}
	\epsilon' = \frac{\epsilon\min_{\theta\in\Theta} L(\mathcal{X},\theta)}{6n\gamma} = \frac{\epsilon L^{\text{opt}}}{6n\gamma}=  \frac{\epsilon \langle f\rangle_\text{opt}}{6\gamma}.
	\end{align}
	
	\textbf{Concluding the proof.}
	Consider $\mathcal{S}$ a sample from a DPP with $L$-ensemble $\ma{L}$, with marginal probabilities $\pi_i$ and normalized marginal probabilities $\bar{\pi}_i=\pi_i/m$.  Consider $\epsilon\in(0,1)$ and $\delta\in(0,1)$. Define $\epsilon'$ as in Eq.~\eqref{eq:epsilonprime} and $\Gamma$ the set of centers of the $\eta$ balls of radius $\epsilon'$ covering the parameter space. We showed that if $m\geq \max (m_1^*, m_2^*)$, then $\mathcal{S}$ is an $\epsilon$-coreset with probability at least $1-\delta$. 
\end{proof}

\section{Coreset results for DPPs}
\label{app:DPPs}

\begin{theorem}[DPP for coresets] 
	\label{thm:main_DPP}Consider $\mathcal{S}$ a sample from a DPP with $L$-ensemble $\ma{L}$, and an average number of sample $\mu=\sum_i\frac{\lambda_i}{1+\lambda_i}$.  Let $\epsilon\in(0,1)$, $\delta\in(0,1)$. Denote by $\eta$ the minimum number of balls of radius $\epsilon \langle f\rangle_\text{opt} / 6 \gamma$ necessary to cover $\Theta$. 
	With probability higher than $1-\delta$, $\mathcal{S}$ is a $\epsilon$-coreset provided that 
	\begin{align*}
	\mu\geq \mu^*=\max(\mu_1^*, \mu_2^*)
	\end{align*}
	with:
	\begin{align*}
	\mu^*_1 &= \frac{32}{\epsilon^2} \left(\epsilon \max_{i} \frac{\sigma_i}{\bar{\pi}_i} + 4 \left(\max_{i} \frac{\sigma_i}{\bar{\pi}_i}\right)^2 \right) \log{\frac{10\eta}{\delta}},\\
	\mu^*_2 &= \frac{32}{\epsilon^2} \left(\frac{\epsilon}{n\bar{\pi}_\text{min}} + \frac{4}{n^2\bar{\pi}_\text{min}^2} \right) \log{\frac{10}{\delta}},
	\end{align*}
	and $\forall i,~\bar{\pi}_i = \pi_i / \mu$.
\end{theorem}

\begin{proof}
	According to~\citet{pemantle_concentration_2014}, replace Eq.~\eqref{eq:to_replace_if_mDPP_1} by:
	\begin{align*}
	\mu \geq \frac{16}{\epsilon^2} \left(\epsilon C + 2C^2\right)  \log{\frac{5}{\delta}},
	\end{align*}
	with $C=\displaystyle \max_{i} \frac{f(x_i, \theta)}{L\bar{\pi}_i}$, where  $\bar{\pi}_i$ is a shorthand for $\pi_i/\mu$; 
	and Eq.~\eqref{eq:to_replace_if_mDPP_2} by:
	\begin{align*}
	\mu \geq \frac{16}{\epsilon^2n\bar{\pi}_{\text{min}}}\left(\epsilon + \frac{2}{n\bar{\pi}_{\text{min}}}\right) \log{\frac{5}{\delta}},
	\end{align*}	
	and change accordingly the rest of the proof. 
\end{proof}
For the same reasons as the $m$-DPP case, we have:
\begin{corollary}
	If $n\sigma_\text{min}\geq1$, then $\mu_1^* \geq \mu_2^*$ and the coreset property of Theorem~\ref{thm:main_DPP} is verified if:
	\begin{align*}
	\mu\geq \mu^*=\frac{32}{\epsilon^2} \left(\epsilon \max_{i} \frac{\sigma_i}{\bar{\pi}_i} + 4 \left(\max_{i} \frac{\sigma_i}{\bar{\pi}_i}\right)^2 \right) \log{\frac{10\eta}{\delta}}.
	\end{align*}
\end{corollary}

\begin{corollary}
	If there exists $\alpha> 0$ and $\beta \geq 1$ such that: 
	\begin{align*}
	\forall i ~~~~~~\qquad &\alpha\sigma_i\leq\pi_i\leq\alpha\beta\sigma_i,\\
	\text{and}~\qquad 
	&\frac{\alpha}{\beta}\geq\frac{32}{\epsilon^2} (\epsilon+4\mathfrak{S}) \log{\frac{10n}{\delta}},
	\end{align*}
	then $\mathcal{S}$ is a $\epsilon$-coreset with probability at least $1-\delta$. In this case, the expected number of samples verifies:
	\begin{align*}
	\mu\geq\frac{32}{\epsilon^2} \beta\mathfrak{S}(\epsilon+4\mathfrak{S}) \log{\frac{10n}{\delta}}.
	\end{align*}
\end{corollary}

\section{Proof of Theorem~\ref{thm:balanced_thm}}
\label{app:proof_asymptotic_thm}

We split the proof into two parts. We first show the convergence of the discrete intensity function to its continuous limit (as $n$ goes to infinity). We then deal with the outer limit (as the degree $\phi$ goes to infinity) to prove the theorem. 

\subsection{Discrete-to-continuous limit}
\label{sec:rebalancing-main-result}

The discrete DPP defined in Section~\ref{sec:polynomial-dpps} has a natural continuous counterpart: namely,
instead of sampling $\mathcal{S}$ from $\mathcal{X}$, we directly sample $\mathcal{S}$ from $\Omega$. The
corresponding orthogonal polynomials are now orthogonal w.r.t. the measure
$\mu$, and the inclusion probabilities turn into intensity functions (in the
continous limit, any given point in $\Omega$ has probability 0 of being
selected, which is why we need to integrate over an $\epsilon$ ball). The
counterpart of the discrete marginal kernel $\ma{K}=\ma{QQ^\top}$ is now a positive-definite kernel
\begin{equation*}
k_\mu(\vec{x},\vec{y}) = \sum_{i=1}^m q_{\mu,i}(\vec{x}) q_{\mu,i}(\vec{y}) 
\end{equation*}
where $q_{\mu,i}$ is the i'th orthogonal polynomial under $\mu$.
As in the discrete case, the intensity function for the continuous DPP simply
equals the diagonal values of the kernel, i.e:
\begin{equation*}
\iota(\vec{y}) = k_\mu(\vec{y},\vec{y})
\end{equation*}
We need to introduce the Christoffel functions of a measure. 
The Christoffel function of $\mu$ is defined as:
\begin{equation}
\label{eq:christoffel-function}
\lambda_{\mu,\phi}(\vec{y}) = \min_{f \in \Pi^d_\phi} \frac{\int f(\vec{x})^2 d \mu}{f(\vec{y})^2}
\end{equation}
where $\Pi^d_\phi$ is the set of polynomials in $\mathbb{R}^d$ of degree less than or equal
to $\phi$. Re-expressing $f$ in the orthonormal basis for $\mu$, and solving for the argmin in \eqref{eq:christoffel-function}, we find:
\begin{equation*}
\lambda_{\mu,\phi}(\vec{y}) = \frac{1}{k_\mu(\vec{y},\vec{y})}
\end{equation*}
so that the intensity function of the (continuous) polynomial DPP is just one
over the Christoffel function. The argument is also valid for the discrete case,
replacing $\mu$ with the empirical distribution $\mu_n = (1/n) \sum
\delta_{\vec{x}_i}$. We may rewrite the empirical Christoffel function as:
\begin{equation*}
\lambda_{\mu_n,\phi}(\vec{y}) = \min_{f \in \Pi^d_\phi,f(\vec{y})=1} \frac{1}{n} \sum_{i=1}^n {f(\vec{x}_i)^2}
\end{equation*}
Convergence of the empirical Christoffel function to its continuous limit is
proved formally by \citet[Theorem 3.11]{lasserre_empirical_2017}, and
is easy to see from the formula above ($\lambda_{\mu_n,\phi}(\vec{y})$ is just the
minimum of a quadratic empirical functional). In the large $n$ limit, we have:

\begin{equation*}
\lim_{n \rightarrow \infty} \lambda_{\mu_n,\phi}(\vec{y}) = \lambda_{\mu,\phi}(\vec{y})
\end{equation*}
a.s., uniformly in $\vec{y} \in D$. $\lambda_{\mu,\phi}(\vec{y})$ for $\vec{y} \in D$ is
bounded below in convex domains~\citep{prymak_upper_2017}, so that
convergence of the inclusion probabilities to the intensity function follows:
\begin{equation*}
\lim_{n \rightarrow \infty} \mathbb{P}(\vec{x}_i \in \mathcal{S}| \mathcal{X}) = \iota_{\phi}(\vec{x}_i)
\end{equation*}
For the unconditional intensity measure, we need to average the left-hand side
over $\mathcal{X}$:
\begin{align*}
\lim_{n\rightarrow \infty}I_{n,\phi}(\mathcal{A}) &=  \lim_{n \rightarrow \infty}  \mathbb{E}_{\mathcal{X},\mathcal{S}} \left\{  \sum_{s_i \in \mathcal{S}} \mathbb{I} \left( s_i \in \mathcal{A} \right)\right\} \\
&=  \lim_{n \rightarrow \infty}  \mathbb{E}_{\mathcal{X}} \left\{  \sum_{x_i \in \mathcal{X}} \mathbb{P}(\vec{x}_i \in \mathcal{S}| \mathcal{X}) \mathbb{I} \left( x_i \in \mathcal{A} \right)\right\} \\
\end{align*}
The quantities in the expectation are bounded ($\leq 1$), and by dominated
convergence we may interchange the limit and the expectation, so that:
\begin{equation}
\label{eq:uncond-intens-function}
\lim_{n\rightarrow \infty}I_{n,\phi}(\vec{y})  =   \int_\mathcal{A} \iota_\phi(\vec{y}) d\mu(\vec{y})
\end{equation}
uniformly in $\vec{y} \in D$.

\subsection{Large-m asymptotics}
\label{sec:large-m-asymptotics}

For the next step, we use the fact that asymptotics of Christoffel functions are
well-studied. We let $\phi \rightarrow \infty$, in which case Theorem 1.5 in~\cite{kroo_christoffel_2013} gives us:
\begin{equation}
\label{eq:large-r-asymp}
\lim_{\phi\rightarrow \infty}\frac{\iota_\phi(\vec{y})}{k_{\nu,\phi}(\vec{y},\vec{y})} = \frac{1}{\mu'(\vec{y})} 
\end{equation}
Here $k_{\nu,\phi}$ is the projection kernel for orthogonal polynomials of degree
$\l$ under the Lebesgue measure. Note that $\iota_\phi(\vec{y})$ and
$k_{\nu,\phi}(\vec{y},\vec{y})$ both integrate to $m$ over $\Omega$, and
$\frac{1}{k_{\nu,\phi}(\vec{y},\vec{y})}$ is the Christoffel function for $\nu$, which
tends to a well-defined limit independent of $\mu$. Injecting
\eqref{eq:large-r-asymp} into \eqref{eq:uncond-intens-function}   proves our result.

\section{Proof of three Lemmas}
\label{app:proof_lemmas}

\begin{lemma}
	\label{lemma:sensi_1means}
	In the $1$-means problem (the $k$-means problem with $k=1$), and supposing without loss of generality that the data is centered (\ie: $\sum_j x_j = 0$), we have:
	\begin{align*}
	\sigma_i = \frac{1}{n}\left(1+\frac{\norm{x_i}^2}{v}\right),
	\end{align*}
	where $v=\frac{1}{n}\sum_{x\in\mathcal{X}} \norm{x}^2$. Thus:
	\begin{align*}
	\mathfrak{S}=\sum_i \sigma_i=2.
	\end{align*}
\end{lemma}

\begin{proof}
	By definition: 
	$$\frac{1}{\sigma_i}=\min_{c} \frac{\sum_x \norm{x-c}^2}{\norm{x_i-c}^2}.$$
	Consider $\mathcal{S}(x_i,R)$ the sphere centered on $x_i$ and radius $R\geq0$. We have that:
	\begin{align*}
	\min_{c} ~~ = ~~~~\min_{R\geq 0 } ~~~\min_{c\in\mathcal{S}}
	\end{align*}
	We thus have:
	\begin{align*}
	\frac{1}{\sigma_i}=\min_{R\geq 0 } \frac{1}{R^2}  \min_{c\in\mathcal{S}}  \sum_x \norm{x-c}^2.
	\end{align*}
	Writing $x-c= x - x_i - (c-x_i)$, we may write 
	\begin{align*}
	\sum_{x}\norm{x - c}^2 = nR^2 + &\sum_{x}\norm{x - x_i}^2  
	- 2R \norm{\sum_{x} x - x_i} \cos{\theta},
	\end{align*}
	with $\theta$ the angle formed by $\sum_{x} x - x_i$ and $c - x_i$. As the minimum is sought for $c$ on the sphere, the angle $\theta$ may take any value, such that the minimum is always attained with $\theta$ s.t. $\cos{\theta}=1$. We finally obtain:
	\begin{align*}
	\frac{1}{\sigma_i}=n+\min_{R\geq 0 } \frac{1}{R^2}  \left(\sum_{x}\norm{x - x_i}^2  - 2R \norm{\sum_{x} x - x_i} \right).
	\end{align*} 
	Studying analytically the function $f(R) = \frac{a-2bR}{R^2}$, its minimum is attained for $R^* = \frac{a}{b}$ and $f(R^*) = -\frac{b^2}{a}$, such that:
	\begin{align*}
	\frac{1}{\sigma_i} 
	=  n - \frac{||\sum_{x} x-x_i||^2}{\sum_{x}\norm{x-x_i}^2}.
	\end{align*}
	Supposing without loss of generality that the data is centered, \ie: $\sum_x x = 0$ and denoting $v=\frac{1}{n}\sum_x \norm{x}^2$, we have:
	\begin{align*}
	\frac{1}{\sigma_i} 
	=  n - \frac{n^2\norm{x_i}^2}{nv + n\norm{x_i}^2}.
	\end{align*}
	Inverting this equation yields:
	\begin{align*}
	\sigma_i 
	&= \frac{v+\norm{x_i}^2}{nv+n\norm{x_i}^2 - n \norm{x_i}^2}
	= \frac{1}{n}\left(1+\frac{\norm{x_i}^2}{v}\right)
	\end{align*}
\end{proof}

\begin{lemma}
	In the $k$-means problem, $n\sigma_{\text{min}}\geq 1$.
\end{lemma}

\begin{proof}
	Consider $\theta^\text{opt}=(c_1^\text{opt}, \ldots, c_k^\text{opt})$ the optimal solution of $k$-means and $\{\mathcal{V}_1, \mathcal{V}_2, \ldots, \mathcal{V}_k\}$ their associated Voronoi sets. Consider $x_i\in\mathcal{X}$ and suppose, without loss of generality that $x_i\in\mathcal{V}_1$. Also, for any $x\in\mathcal{X}$, we denote by $c(x) = \argmin_{c\in\theta} \norm{x-c}^2$. We have:
	\begin{align*}
	\frac{1}{\sigma_i} &= \min_{c_1, \ldots, c_k} \frac{\sum_{x\in\mathcal{X}} \norm{x - c(x)}^2}{\norm{x_i - c(x_i)}^2}\\
	&=\min_{c_1, \ldots, c_k} \frac{\sum_{x\in\mathcal{V}_1} \norm{x - c(x)}^2}{\norm{x_i - c(x_i)}^2} + \sum_{j=2}^k \frac{\sum_{x\in\mathcal{V}_j} \norm{x - c(x)}^2}{\norm{x_i - c(x_i)}^2}
	\end{align*}
	Given that, by definition of $c(x)$, $\forall j, ~\norm{x-c(x)}^2 \leq \norm{x-c_j}^2$, we have:
	\begin{align*}
	\frac{1}{\sigma_i} 
	\leq \min_{c_1, \ldots, c_k} \frac{\sum_{x\in\mathcal{V}_1} \norm{x - c_1}^2}{\norm{x_i - c(x_i)}^2} + \sum_{j=2}^k \frac{\sum_{x\in\mathcal{V}_j} \norm{x - c_j}^2}{\norm{x_i - c(x_i)}^2}
	\end{align*}
	To further bound this quantity, let us constrain the domain over which the minimum is sought.  Consider $\mathcal{B}(x_i,R)$ the ball centered on $x_i$ and radius $R\geq0$. Consider $\mathcal{S}(x_i,R)$ its surface (\ie, the associated sphere). We have that:
	\begin{align*}
	\min_{c_1, \ldots, c_k} ~~\leq ~~~~\min_{R\geq 0 } ~~~\min_{c_1\in\mathcal{S}, (c_2,\ldots, c_k)\notin\mathcal{B}}
	\end{align*}
	Given this restricted search space, we have: $c(x_i) = c_1$ and $\norm{x_i - c_1}^2 = R^2$, and thus:
	\begin{align*}
	\frac{1}{\sigma_i} 
	~~\leq~~ &\min_{R\geq 0 } \frac{1}{R^2}  \min_{c_1\in\mathcal{S}} \left(\sum_{x\in\mathcal{V}_1} \norm{x - c_1}^2  
	+ \min_{(c_2,\ldots, c_k)\notin\mathcal{B}} \sum_{j=2}^k \sum_{x\in\mathcal{V}_j} \norm{x - c_j}^2\right)
	\end{align*}
	Now, one may show, for all $j=2, \ldots, k$, that:
	\begin{align*}
	\sum_{x\in\mathcal{V}_j} \norm{x - c_j}^2 = \sum_{x\in\mathcal{V}_j} \norm{x - c_j^\text{opt}}^2 + \#\mathcal{V}_j \norm{c_j-c_j^\text{opt}}^2,
	\end{align*}
	due to the fact that $c_j^\text{opt} = \frac{1}{\#\mathcal{V}_j} \sum_{x\in\mathcal{V}_j}  x$. 
	Given that the minimum of $\norm{c_j-c_j^\text{opt}}^2$ is necessarily smaller than $R^2$:
	\begin{align*}
	\min_{c_j\notin\mathcal{B}}	\sum_{x\in\mathcal{V}_j} \norm{x - c_j}^2 \leq \sum_{x\in\mathcal{V}_j} \norm{x - c_j^\text{opt}}^2 + \#\mathcal{V}_j R^2, 
	\end{align*}
	such that:
	\begin{align*}
	\frac{1}{\sigma_i} 
	\leq &  \min_{R\geq 0 } \frac{1}{R^2} 
	\min_{c_1\in\mathcal{S}} \left(\sum_{x\in\mathcal{V}_1} \norm{x - c_1}^2  
	+ \alpha + (n-\#\mathcal{V}_1)R^2\right)\\
	&=n-\#\mathcal{V}_1 + \min_{R\geq 0 } \frac{1}{R^2} \min_{c_1\in\mathcal{S}}  \left(\sum_{x\in\mathcal{V}_1} \norm{x - c_1}^2  
	+ \alpha\right)
	\end{align*}
	with $\alpha = L^{\text{opt}\backslash\mathcal{V}}$ the optimal $(k-1)$-means cost on $\mathcal{X}\backslash\mathcal{V}$.  
	Writing $x-c_1= x - x_i - (c_1-x_i)$, we may decompose $
	\sum_{x\in\mathcal{V}_1}\norm{x - c_1}^2$ in $R^2\#\mathcal{V}_1 + \sum_{x\in\mathcal{V}_1}\norm{x - x_i}^2  - 2R \norm{\sum_{x\in\mathcal{V}_1} x - x_i} \cos{\theta}$, 
	with $\theta$ the angle formed by $\sum_{x\in\mathcal{V}_1} x - x_i$ and $c_1 - x_i$. As the minimum is sought for $c_1$ on the sphere, the angle $\theta$ may take any value, such that the minimum is always attained with $\theta$ s.t. $\cos{\theta}=1$. We finally obtain, denoting $\forall x\in\mathcal{V}_1,~y = x-x_i$:
	\begin{align*}
	\frac{1}{\sigma_i} 
	\leq &  n + \min_{R\geq 0 } \frac{1}{R^2} \left(\sum_{x\in\mathcal{V}_1}\norm{y}^2  - 2R \norm{\sum_{x\in\mathcal{V}_1} y} + \alpha\right). 
	\end{align*}
	Studying analytically the function $f(R) = \frac{a-2bR+\alpha}{R^2}$, its minimum is attained for $R^* = \frac{a+\alpha}{b}$ and $f(R^*) = -\frac{b^2}{a+\alpha}$, such that:
	\begin{align*}
	\frac{1}{\sigma_i} 
	\leq &  n - \frac{||\sum_{x\in\mathcal{V}_1} y||^2}{\sum_{x\in\mathcal{V}_1}\norm{y}^2 + \alpha} \leq n.
	\end{align*}
	This is true for all $i$, and in particular for $\sigma_{\text{min}}$. 
\end{proof}

\begin{lemma}
	\label{lemma:sensi_lr}
	With the notations of Section~\ref{sec:DPP_for_LR}, the sensitivities in the linear regression problem verify:
	\begin{align*}
	\forall i\qquad	
	\sigma_i =x_i^\top \ma{H}^{-1} x_i + \frac{\left(y_i - y^*_i \right)^2}{\norm{y-y^*}^2}
	\end{align*}
	where $\ma{H}=\ma{X}^\top \ma{X}$ and $y^*$ reads $y^*=\ma{X}\theta^* =\ma{XH}^{-1}\ma{X}^\top y$. 
	Also:
	\begin{align*}
	\mathfrak{S}&=\sum_i \sigma_i=d+1.
	\end{align*}
	As a remark, note that the sensitivity is different from the usual definition of leverage score in the context of linear regression, which simply reads $l_i=x_i^\top\ma{H}^{-1}x_i$~\citep[see, \textit{e.g.},][]{hoaglin_hat_1978, chatterjee_sensitivity_1988, chatterjee_influential_1986, chen_statistical_2016}.
\end{lemma}
\begin{proof}
	We have:
	\begin{align*}
	\frac{1}{\sigma_i} = \min_{\theta\in\Theta} \frac{\sum_j (y_j-x_j^\top\theta)^2}{(y_i-x_i^\top\theta)^2}
	\end{align*}
	Let us write $\theta=u+v$ where $u$ is colinear to $x_i$ and $v$ is orthogonal to $x_i$. We obtain:
	\begin{align*}
	\frac{1}{\sigma_i} &= \min_{u,v} \frac{\sum_j (y_j-x_j^\top(u+v))^2}{(y_i-x_i^\top u)^2}\\
	&= \min_u \frac{1}{(y_i-x_i^\top u)^2}\left(\norm{y}^2+\min_v \left[(u+v)^\top \ma{H} (u+v) - 2(u+v)^\top \ma{X}^\top y\right] \right)\\
	&=\min_u \frac{1}{(y_i-x_i^\top u)^2}\left(\norm{y}^2+ u^\top \ma{H}u - 2u^\top \ma{X}^\top y + \min_v \left[v^\top \ma{H} v + 2u^\top \ma{H}v- 2v^\top \ma{X}^\top y\right] \right)
	\end{align*}
	where $\ma{H}=\sum_j x_j x_j^\top=\ma{X}^\top \ma{X}$. Let us first concentrate on solving:
	\begin{align*}
	\min_v v^\top \ma{H} v + 2u^\top \ma{H}v- 2v^\top \ma{X}^\top y = \min_v v^\top \ma{H} v + 2z^\top v
	\end{align*}
	with $z=\ma{H}^\top u - \ma{X}^\top y$. The minimum is to be found for $v$ orthogonal to $x_i$. We write the Lagrangian:
	\begin{align*}
	L(v,\lambda)=v^\top \ma{H} v + 2z^\top v-\lambda x_i^\top v.
	\end{align*}
	We solve it wrt $v$:
	\begin{align*}
	2\ma{H} v + 2z -\lambda x_i =0
	\end{align*}	
	\ie:
	\begin{align*}
	v = \ma{H}^{-1} \left(\frac{\lambda}{2} x_i -z \right).
	\end{align*}
	We know that $v$ should be orthogonal to $x_i$ such that:
	\begin{align*}
	0 = x_i^\top \ma{H}^{-1} \left(\frac{\lambda}{2} x_i -z \right)
	\end{align*}
	\ie:
	\begin{align*}
	\frac{\lambda}{2}= \frac{x_i^\top \ma{H}^{-1}z}{x_i^\top \ma{H}^{-1}x_i}.
	\end{align*}
	We finally have:
	\begin{align*}
	v^* = \ma{H}^{-1} \left(\frac{x_i^\top \ma{H}^{-1}z}{x_i^\top \ma{H}^{-1}x_i} x_i -z \right)
	\end{align*}
	and thus:
	\begin{align*}
	\min_v v^\top \ma{H} v + 2z^\top v = v^{*\top} \ma{H} v^* + 2z^\top v^* = \frac{(x_i^\top \ma{H}^{-1} z)^2}{x_i^\top \ma{H}^{-1} x_i}-z^\top \ma{H}^{-1} z.
	\end{align*}
	\ie:
	\begin{align*}
	\frac{1}{\sigma_i} &=\min_u \frac{1}{(y_i-x_i^\top u)^2}\left(\norm{y}^2+\frac{(x_i^\top u-x_i^\top \ma{H}^{-1}\ma{X}^\top y)^2}{x_i^\top \ma{H}^{-1} x_i}-y^\top \ma{X}\ma{H}^{-1}\ma{X}^\top y\right)\\
	&=\frac{1}{x_i^\top \ma{H}^{-1} x_i}\min_u \frac{\left(\norm{y}^2-y^\top \ma{X}\ma{H}^{-1}\ma{X}^\top y\right) x_i^\top \ma{H}^{-1} x_i + \left(x_i^\top u-x_i^\top \ma{H}^{-1}\ma{X}^\top y\right)^2}{\left(y_i-x_i^\top u\right)^2}
	\end{align*}
	Let us write $u=\alpha \frac{x_i}{\norm{x_i}^2}$. We have:
	\begin{align*}
	\frac{1}{\sigma_i} &= \frac{1}{x_i^\top \ma{H}^{-1} x_i}\min_\alpha \frac{a+(\alpha-b)^2}{(\alpha -c)^2}
	\end{align*}
	with: $a= \left(\norm{y}^2-y^\top \ma{X}\ma{H}^{-1}\ma{X}^\top y\right) x_i^\top \ma{H}^{-1} x_i$, $b=x_i^\top \ma{H}^{-1}\ma{X}^\top y$ and $c=y_i$.
	The minimum of $f(\alpha)= \frac{a+(\alpha-b)^2}{(\alpha -c)^2}$ is attained for $\alpha^*=\frac{a}{b-c}+b$ which entails:
	\begin{align*}
	f(\alpha^*)= \frac{a}{a+(b-c)^2}.
	\end{align*}
	And thus:
	\begin{align*}
	\frac{1}{\sigma_i} = \frac{\norm{y}^2-y^\top \ma{X}\ma{H}^{-1}\ma{X}^\top y}{\left(\norm{y}^2-y^\top \ma{X}\ma{H}^{-1}\ma{X}^\top y\right) x_i^\top \ma{H}^{-1} x_i+\left(x_i^\top \ma{H}^{-1}\ma{X}^\top y - y_i \right)^2}
	\end{align*}
	\ie:
	\begin{align*}
	\sigma_i = x_i^\top \ma{H}^{-1} x_i + \frac{\left(x_i^\top \ma{H}^{-1}\ma{X}^\top y - y_i \right)^2}{\norm{y}^2-y^\top \ma{X}\ma{H}^{-1}\ma{X}^\top y}.
	\end{align*}
	Writing $\theta^*=\ma{H}^{-1}\ma{X}^\top y$ the least-square solution to the problem, this is re-written:
	\begin{align*}
	\sigma_i = x_i^\top \ma{H}^{-1} x_i + \frac{\left(x_i^\top \theta^* - y_i \right)^2}{\norm{y}^2-\theta^{*\top}\ma{H}\theta^*}.
	\end{align*}
	Finally, denoting $y^*=\ma{X}\theta^*$:
	\begin{align*}
	\sigma_i &= x_i^\top \ma{H}^{-1} x_i + \frac{\left(y_i - y^*_i \right)^2}{\norm{y}^2-\norm{y^*}^2}\\
	&=x_i^\top \ma{H}^{-1} x_i + \frac{\left(y_i - y^*_i \right)^2}{\norm{y-y^*}^2}
	\end{align*}
	Thus:
	\begin{align*}
	\mathfrak{S}&=\sum_i \sigma_i=\text{Tr}(\ma{X}^\top \ma{H}^{-1}\ma{X})+\sum_i \frac{\left(y_i - y^*_i \right)^2}{\norm{y-y^*}^2}\\
	&=d+1.
	\end{align*}
\end{proof}

\section{The issue of outliers}
\label{remark:outliers}
Corollary~\ref{cor:opt_DPP} is applicable to cases where $\sigma_\text{max}$ is not too large. In fact, in order for $\alpha\sigma_i$ to be smaller than $\pi_i$, and thus smaller than $1$ as $\pi_i$ is a probability, $\alpha$ should always be set inferior to $\frac{1}{\sigma_\text{max}}$. Now, if $\sigma_\text{max}$ is so large that $\frac{1}{\sigma_\text{max}}\leq \frac{32}{\epsilon^2}\mathfrak{S}\log{\frac{4\eta}{\delta}}$, then, even by  setting $\beta$ to its minimum value $1$, there is no admissible $\alpha$ verifying both conditions~\eqref{eq:condition_beta_gamma1} and~\eqref{eq:condition_beta_gamma2}. Large values of $\sigma_i$ means strong outliers.\footnote{Sensitivities have indeed been shown to be good outlierness indicators~\citep{lucic_linear-time_2016}.} A simple workaround in this case is to separate the data in two: $\mathcal{X}_o=\{x_i \text{ s.t. } \sigma_i>\sigma^*\}$ the set of outliers and $\bar{\mathcal{X}}=\{x_i \text{ s.t. } \sigma_i\leq\sigma^*\}$ the others, where $\sigma^*$ is the threshold sensitivity over which a data point is considered as an outlier (it is discussed in the following). The initial cost $L$ may also be separated in two: $L=L_o + \bar{L}$ where 
\begin{align*}
L_o = \sum_{x\in\mathcal{X}_o} f(x,\theta) \qquad \text{ and }\qquad \bar{L} = \sum_{x\in\bar{\mathcal{X}}} f(x,\theta).
\end{align*}
Let us write $\bar{\sigma}_i$ the sensitivity of data point $i$ in $\bar{\mathcal{X}}$ and $\bar{\mathfrak{S}}=\sum_{x\in\bar{\mathcal{X}}}\bar{\sigma}_i$. Let us choose $\sigma^*$ to be the largest value in $[0,1]$ for which $\frac{1}{\bar{\sigma}_\text{max}}\geq \frac{32}{\epsilon^2}\bar{\mathfrak{S}}\log{\frac{4\eta}{\delta}}$ is verified. 
One can thus apply the corollary to $\bar{\mathcal{X}}$ to obtain $\bar{\mathcal{S}}$ such that:
$$\forall\theta\in\Theta\qquad(1-\epsilon)\bar{L}(\bar{\mathcal{X}},\theta)\leq\hat{\bar{L}}(\bar{\mathcal{S}},\theta)\leq(1+\epsilon)\bar{L}(\bar{\mathcal{X}},\theta). $$
Trivially, one may add to $\bar{\mathcal{S}}$ all outliers in $\mathcal{X}_o$ and associate to each of them a weight $1$ in the estimated cost. The resulting set $\mathcal{S}$ is thus necessarily a coreset for all datapoints:
$$\forall\theta\in\Theta\qquad(1-\epsilon)L\leq(1-\epsilon)\bar{L}+L_o\leq\hat{L}=\hat{\bar{L}}+L_o\leq(1+\epsilon)\bar{L}+L_o\leq (1+\epsilon)L.$$
The number of required samples is thus the number required for $\bar{\mathcal{S}}$ to be a coreset for $\bar{\mathcal{X}}$ plus the number of outliers in $\mathcal{X}_o$: $\mathcal{O}(|\mathcal{X}_o|+\frac{\bar{\mathfrak{S}}^2}{\epsilon^2}\log{\frac{\eta}{\delta}})$. The exact value of $\sigma^*$ is application and data dependent. In general, we expect it to be $\mathcal{O}(1)$, such that the number of outliers $|\mathcal{X}_o|$ may be considered as a constant and the number of required samples is of the order  
$\mathcal{O}(\frac{\mathfrak{S}^2}{\epsilon^2}\log{\frac{\eta}{\delta}})$. 

%Moreover, we expect that further research on concentration properties of DPPs will overcome this issue. In fact, in practice, outliers have a high chance of being sampled; and if they are sampled once, they are never sampled again due to the determinantal nature of DPPs: they should only have a small impact on the minimum number of samples required. 

\section{Implementation}
\label{app:implementation}

\subsection{Approximating the kernel via Random Fourier Features}
\label{subsec:RFF}
In order to approximate $\ma{L}$ in time linear in $n$, we rely on random Fourier features (RFF)~\citep{rahimi_random_2008}. We briefly recall the RFF framework in the following.

Let us write $\kappa$ the Gaussian kernel that we use: $\kappa(\vec{t})=\exp(-\vec{t}^2/2\tau^2)$. Its Fourier transform is:
\begin{align*}
\hat{\kappa}(\vec{\omega})=\int_{\mathbb{R}^d}\kappa(\vec{t})\exp^{-i\vec{\omega}^\adjoint\vec{t}}\mbox{d}\vec{t}. 
\end{align*}
It has real values as 
$\kappa$ is symmetrical. One may write:
\begin{equation*}
\kappa(\vec{x},\vec{y})=\kappa(\vec{x}-\vec{y})=\frac{1}{Z}\int_{\mathbb{R}^d}\hat{\kappa}(\vec{\omega})\exp^{i\vec{\omega}^\adjoint(\vec{x}-\vec{y})}\mbox{d}\vec{\omega},
\end{equation*}
where, in order to ensure that $\kappa(\vec{x},\vec{x})=1$:
\begin{align*}
Z=\int_{\mathbb{R}^d}\hat{\kappa}(\vec{\omega})\mbox{d}\vec{\omega}.
\end{align*} 
According to Bochner's theorem, and due to the fact that $\kappa$ is positive-definite,  $\hat{\kappa}/Z$ is a valid
probability density function. 
$\kappa(\vec{x},\vec{y})$ may thus be interpreted as the expected value of $\exp^{i\vec{\omega}^\adjoint(\vec{x}-\vec{y})}$ 
provided that $\vec{\omega}$ is drawn from $\hat{\kappa}/Z$:
\begin{align*}
\kappa(\vec{x},\vec{y})=\mathbb{E}_{\vec{\omega}} \left(\exp^{i\vec{\omega}^\adjoint(\vec{x}-\vec{y})}\right)
\end{align*}
The distribution $\hat{\kappa}/Z$ from which $\omega$ should be drawn from may be shown to be $\mathcal{N}(\omega;0,1/\tau^2)$, 
where $\mathcal{N}(x;\mu,v)$ is the normal law:
\begin{align*}
\mathcal{N}(x;\mu,v)=\frac{1}{\sqrt{2v\pi}}\exp^{-\frac{(x-\mu)^2}{2v}}.
\end{align*}
In practice, we draw $r$ random Fourier vectors from $\hat{\kappa}/Z$: $$\Omega_r=(\vec{\omega}_1,\ldots,\vec{\omega}_r).$$
For each data point $\vec{x}_j$, we define a column feature vector associated 
to $\Omega_r$:
\begin{align*}
\vec{\psi}_j=\frac{1}{\sqrt{r}} [\cos(\vec{\omega}_1^\adjoint\vec{x}_j)|\cdots|\cos(\vec{\omega}_r^\adjoint\vec{x}_j)|
\sin(\vec{\omega}_1^\adjoint\vec{x}_j)|\cdots|\sin(\vec{\omega}_r^\adjoint\vec{x}_j)]^\adjoint\in\mathbb{R}^{2r},
\end{align*}
and call $\ma{\Psi}=\left(\vec{\psi}_1|\cdots|\vec{\psi}_n\right)\in\mathbb{R}^{2r\times n}$ the RFF matrix. Other embeddings are possible in the RFF framework, but this one was shown to be the most appropriate to the Gaussian kernel~\citep{sutherland_error_2015}. 
As $r$ increases, $\kappa(\vec{x}_i,\vec{x}_j)$ concentrates around its expected value:
$\vec{\psi}_i^\adjoint\vec{\psi}_j\simeq \kappa(\vec{x}_i,\vec{x}_j)$.  
The Gaussian kernel matrix is thus approximated via:
\begin{align*}
\ma{L}\simeq \ma{\Psi}^\adjoint\ma{\Psi}.
\end{align*}
Computing the RFF matrix requires $\mathcal{O}(nrd)$ operations.

\begin{remark}
	How many random features $r$ should we choose? Firstly, note that the 
	entry-wise concentration of $\ma{\Psi}^\adjoint\ma{\Psi}$ around its expected value $\ma{L}$ is controlled by a multiplicative error $\epsilon$ provided that  $r\geq\mathcal{O}(d/\epsilon^2)$~\citep{rahimi_random_2008}. Thus, $r$ should at least be of the order of the dimension $d$ if one wants a proper approximation of the Gaussian kernel. However, this is not our goal here. In fact, what is needed is to obtain in average $\mu$ samples from a DPP with $L$-ensemble $\ma{\Psi}^\adjoint\ma{\Psi}$. The maximum number of samples of such a DPP is the rank of $\ma{\Psi}$, such that $r$ should necessarily be chosen larger than $\mu$. In the following, we thus set $r$ to simply be a few times $\mu$.
\end{remark}

\subsection{Fast sampling of DPPs}

In order to sample a DPP from a $L$-ensemble given its eigenvectors $\{\fou_k\}$ and eigenvalues $\lambda_k$, one may follow Algorithm~1 of~\cite{kulesza_determinantal_2012}, originally from~\citet{hough_determinantal_2006}. This algorithm runs in $\mathcal{O}(n\mu^3)$ in average. 
The limiting step of the overall sampling algorithm is the $\mathcal{O}(n^3)$ cost of the diagonalisation of $\ma{L}$. Thankfully, the RFFs not only provide us with an approximation of $\ma{L}$ in linear time, it also provides us with a dual representation, \ie, a representation of $\ma{L}$ in the form 
\begin{align*}
\ma{L}=\ma{\Psi}^\adjoint\ma{\Psi}.
\end{align*}
Thus, we may circumvent the prohibitive diagonalization cost of $\ma{L}$ and only diagonalize its dual form:
\begin{align*}
\ma{C} = \ma{\Psi}\ma{\Psi}^\adjoint  \in\mathbb{R}^{2r\times 2r},
\end{align*}
costing only $\mathcal{O}(nr^2)=\mathcal{O}(n\mu^2)$ (time to compute $\ma{C}$ from $\ma{\Psi}$ and to compute the low-dimensional diagonalization). $\ma{C}$'s eigendecomposition yields:
\begin{align*}
\ma{C} = \ma{V}\ma{D}\ma{V}^\adjoint,
\end{align*}
with $\ma{V}=(\vec{v}_1|\ldots|\vec{v}_{2r})$ the orthonormal basis of eigenvectors and $\ma{D}$ the diagonal matrix of eigenvalues such that $0\leq \nu_1\leq\ldots\leq \nu_{2r}$. 

Note that all eigenvectors associated to non-zero eigenvalues of $\ma{L}$ can be recovered from $\ma{C}$'s eigendecomposition~\citep[see, \emph{e.g.},][Proposition 3.1]{kulesza_determinantal_2012}.  More precisely, if $\vec{v}_k$ is an eigenvector of $\ma{C}$ associated to eigenvalue $\nu_k$, then:
\begin{align*}
\vec{u}_k = \frac{1}{\sqrt{\nu_k}} \ma{\Psi}^\adjoint \vec{v}_k
\end{align*}
is a normalized eigenvector of $\ma{L}$ associated to the same eigenvalue. 

In the case of such a dual representation, two standard approaches are used in the literature: 1)~either follow Algorithm~1 of~\cite{kulesza_determinantal_2012} with the reconstructed eigenvectors $\ma{U}=\ma{\Psi}^\adjoint\ma{V}\ma{D}^{-1/2}$ as inputs, running in $\mathcal{O}(n\mu^3)$; 2)~or follow Algorithm~3 of~\cite{kulesza_determinantal_2012} with the dual eigendecomposition $\{\vec{v}_k\}$ and $\{\nu_k\}$ as inputs, running in $\mathcal{O}(nr\mu^2+r^2\mu^3)$. 
Both approaches are nevertheless suboptimal and we show in~\cite{tremblay_optimized_2018} that the first (resp. second) one has an equivalent formulation running in $\mathcal{O}(n\mu^2)$ (resp. $\mathcal{O}(n\mu r)$). In this paper, we work with the following sampling strategy, given the dual eigendecomposition $\{\vec{v}_k\}$ and $\{\nu_k\}$:
\begin{enumerate}
	\item[i/] Sample eigenvectors. Draw $n$ Bernoulli variables with parameters $\nu_k/(1+\nu_k)$: for $k=1,\ldots,2r$, add $k$ to the set of sampled indices $\mathcal{J}$ with probability $\nu_k/(1+\nu_k)$. We generically denote by $J$ the number of elements in $\mathcal{J}$. Note that the expected value of $J$ is $\mu$.
	\item[ii/] Run Algorithm~\ref{alg:sampling_DPP_efficient} to sample a $J$-DPP with projective $L$-ensemble $\ma{P}=\ma{W} \ma{W}^\adjoint$ where  $\ma{W}\in\mathbb{R}^{n\times J}$  concatenates all the reconstructed  eigenvectors $\fou_k=\frac{1}{\sqrt{\nu_k}} \ma{\Psi}^\adjoint \vec{v}_k$ such that $k\in\mathcal{J}$.
\end{enumerate}

\begin{algorithm}
	\caption{Efficient $J$-DPP sampling algorithm with projective $L$-ensemble $\ma{P}=\ma{W} \ma{W}^\adjoint$}
	\label{alg:sampling_DPP_efficient}
	\begin{algorithmic}
		\Input $\ma{W}\in\mathbb{R}^{n\times J}$ such that $\ma{W}^\adjoint \ma{W}=\ma{I}_J$\\
		Write $\forall i, ~~\vec{y}_i=\ma{W}^\adjoint\vec{\delta_i}\in\mathbb{R}^{J}$.\\
		$\mathcal{S} \leftarrow \emptyset$\\
		Define $\vec{p}\in\mathbb{R}^n$ : $\forall i, \quad p(i) = \norm{\vec{y}_i}^2$\\
		\textbf{for} $n=1,\ldots,J$ \textbf{do}:\\
		\hspace{0.5cm}$\bm{\cdot}$ Draw $s_n$ with proba $\mathbb{P}(s)=p(s)/\sum_{i}p(i)$\\
		\hspace{0.5cm}$\bm{\cdot}$ $\mathcal{S} \leftarrow \mathcal{S}\cup\{s_n\}$\\
		\hspace{0.5cm}$\bm{\cdot}$ Compute 
		$\vec{f}_n = \vec{y}_{s_n} - \sum_{l=1}^{n-1} \vec{f}_l(\vec{f}_l^\adjoint\vec{y}_{s_n})\in\mathbb{R}^{J}$\\
		\hspace{0.5cm}$\bm{\cdot}$ Normalize $\vec{f}_n \leftarrow \vec{f}_n / \sqrt{\vec{f}_n^\adjoint\vec{y}_{s_n}}$\\
		\hspace{0.5cm}$\bm{\cdot}$ Update $\vec{p}$ : $\forall i\quad p(i) \leftarrow p(i) - (\vec{f}_n^\adjoint\vec{y}_i)^2$\\
		\textbf{end for}
		\Output $\mathcal{S}$ of size $J$.
	\end{algorithmic}
\end{algorithm}
The runtime of this strategy given the dual eigendecomposition is $\mathcal{O}(n\mu^2)$. Also, for a proof that this strategy does sample from a DPP with $L$-ensemble $\ma{L}=\ma{\Psi}^\adjoint\ma{\Psi}$ we refer the reader to our technical report~\citep{tremblay_optimized_2018}. These algorithms are implemented in the Julia toolbox we have developed for this work (DPP.jl available at \url{https://gricad-gitlab.univ-grenoble-alpes.fr/barthesi/dpp.jl}). Alternatively, one could use G. Gautier's well-documented Python toolbox DPPy (available at \url{http://github.com/guilgautier/DPPy}).

\subsection{Fast sampling of $m$-DPPs}
\label{subsec:fast_mDPP}
In the experiments, we will only provide results for $m$-DPP sampling. In fact, results are easier to compare with classical i.i.d. coreset methods when the number of samples is fixed and not random. Given the eigendecomposition of the dual representation $\ma{C}$, one samples a $m$-DPP via the following two steps (we refer once again to~\cite{tremblay_optimized_2018} for a proof):
\begin{enumerate}
	\item[i/] Sample $m$ eigenvectors. Draw $2r$ Bernoulli variables with parameters $\nu_k/(1+\nu_k)$ under the constraint that exactly $m$ variables should be equal to one. Call $\mathcal{J}$ the set of indices thus drawn: $|\mathcal{J}|=J=m$. 
	\item[ii/] Run Algorithm~\ref{alg:sampling_DPP_efficient} to sample a $J$-DPP with projective $L$-ensemble $\ma{P}=\ma{W} \ma{W}^\adjoint$ where  $\ma{W}\in\mathbb{R}^{n\times J}$  concatenates all the reconstructed  eigenvectors $\fou_k=\frac{1}{\sqrt{\nu_k}} \ma{\Psi}^\adjoint \vec{v}_k$ such that $k\in\mathcal{J}$.
\end{enumerate}
The only difference with a usual DPP is in the first step, where the $n$ Bernoulli variables are not drawn independently anymore, but under constraint that exactly $m$ of them should be equal to one. To do so, one may follow Algorithm~8 of~\cite{kulesza_determinantal_2012} which runs in $\mathcal{O}(nm)$. Step ii/ runs in $\mathcal{O}(nm^2)$, such that the overall cost of sampling a $m$-DPP given the dual eigendecomposition is also $\mathcal{O}(nm^2)$. Algorithm~8 of~\cite{kulesza_determinantal_2012} makes use of  elementary polynomials. Given the eigenvalues of $\ma{C}$, $\{\nu_i\}$, the $p$-th order associated elementary polynomial reads:
\begin{align*}
e_p(\nu_1,\ldots,\nu_{2r})=\sum_{\mathcal{J}\subseteq\{1,2,\ldots,2r\} ~\text{s.t.}~ |\mathcal{J}|=p} ~\quad\prod_{j\in\mathcal{J}}\nu_j\qquad\in\mathbb{R}.
\end{align*}
As $r$ increases, these polynomials become less and less stable to compute and Algorithm~8 of~\citet{kulesza_determinantal_2012} fails in many practical situations due to numerical precision errors as $m$ becomes too large. In order to avoid these errors, we follow the saddle-point approximation method detailed by~\citet{barthelme_asymptotic_2019}. This method has the additional advantage of providing very accurate approximations of the probabilities of inclusion of the $m$-DPP (that are exactly written as a ratio of elementary polynomials and thus also vulnerable to numerical instability). We in fact need these marginals for the importance sampling estimator.

	\vskip 0.2in
\bibliography{biblio.bib}

\end{document}